\documentclass[11pt]{article}
\usepackage{custom}
\title{A Confidence Interval for the $\ell_2$ Expected Calibration Error}
\author{Yan Sun$^{1}$, Pratik Chaudhari$^{2}$, Ian J. Barnett$^{2}$, and Edgar Dobriban$^{2}$\footnote{Correspondence to YS (\texttt{yan.sun@njit.edu}) and ED (\texttt{dobriban@wharton.upenn.edu}).} \\
\small $^{1}$ New Jersey Institute of Technology, \, $^{2}$University of Pennsylvania
}

\begin{document}

\maketitle
\begin{abstract}
Recent advances in machine learning have significantly improved prediction accuracy in various applications. However, ensuring the calibration of probabilistic predictions remains a significant challenge. Despite efforts to enhance model calibration, the rigorous statistical evaluation of model calibration remains less explored. 
In this work, we develop confidence intervals {for} the $\ell_2$ Expected Calibration Error (ECE). We consider top-1-to-$k$ calibration, which includes both the popular notion of confidence calibration as well as full calibration.
For a debiased estimator of the ECE, we show asymptotic normality, but with different convergence rates and asymptotic variances for calibrated and miscalibrated models. 
We develop methods to construct asymptotically valid confidence intervals for the ECE, accounting for this behavior as well as non-negativity. 
Our theoretical findings are supported through extensive experiments, showing that our methods produce valid confidence intervals with shorter lengths compared to those obtained by resampling-based methods.
\end{abstract}

\renewcommand{\baselinestretch}{0.5}\normalsize
\tableofcontents
\bigskip
\renewcommand{\baselinestretch}{1.0}\normalsize

\section{Introduction}

The recent development of sophisticated machine learning methods, 
such as deep neural networks, 
has dramatically improved prediction accuracy on a range of problems. 
As a result, machine learning methods are increasingly used in safety-critical applications, such as self-driving cars  \citep{bojarski2016end} and medical diagnosis  \citep{esteva2017dermatologist}. In these applications, machine learning methods may provide probabilistic forecasts for classification tasks, and the probability outputs of a model can be important for decision-making. 
In these cases, machine learning methods should not only be accurate, but also have \emph{calibrated} probability predictions.

Calibration provides a formal guarantee to ensure the predicted probabilities are meaningful and reliable. Take, for example, a binary classification predictor 
$f$, which inputs a feature 
$X$ and outputs 
$f(X)$, the probability that the corresponding label $Y$ is one (``1"). The model $f$ is considered calibrated if, for any given probability $p \in [0, 1]$, $\EE{Y | f(X) = p} = p$. This could mean, for instance, 
% that 10\% of patients predicted to have a 10\% chance of heart disease actually have this condition.
{if it is predicted that a particular type of patient has 10\% chance of heart disease, then on average, 10\% of such patients will actually have heart disease.}

This form of calibration is recognized as an important aspect of probabilistic forecasts in many applications \citep{murphy1984probability, bjerregaabd1978measurement, guo2017calibration, minderer2021revisiting, gneiting2007probabilistic,gneiting2014probabilistic}. Unfortunately, many popular modern deep learning models are reported to be poorly calibrated \citep{guo2017calibration}, prompting a significant body of research focused on enhancing the calibration of machine learning methods \citep{guo2017calibration, kull2019beyond, kumar2018trainable, mukhoti2020calibrating}. Despite the considerable effort dedicated to improving and quantifying model calibration, the rigorous statistical evaluation of these model calibration metrics remains an under-explored area. 

The most widely used calibration measure is the $\ell_p$ Expected Calibration Error (ECE), for $p\ge 1$, defined as
\begin{equation*}
\label{ece}
\mathrm{ECE}_{p}:= \EE{\|\EE{Y-f(X)|f(X)}\|^p}^{\frac{1}{p}}.
\end{equation*}
To estimate $\mathrm{ECE}_{p}$ in practical settings, the support of 
$f(X)$ is typically divided into a predefined number of bins (e.g., 15 bins is widely used as in \cite{guo2017calibration}), and the mean of  $\|Y-f(X)\|^p$ within each bin is calculated, followed by an average across all bins. However, the estimator on a finite dataset can be both biased \citep{kumar2019verified} and highly variable \citep{tao2023benchmark}. 
Therefore, constructing a confidence interval for the calibration error would be an important step towards a thorough evaluation of model calibration.

A confidence interval for the calibration can be used in many ways.
Since a high accuracy does not necessarily imply calibration, a confidence interval can be a diagnostic tool that tells us---with confidence---when a model is miscalibrated; and thus its probabilistic predictions themselves are not trustworthy. 
On the flip side, if a model provider aims to argue that their model's probabilistic predictions are trustworthy, they can use calibration as a proxy. Since calibration can be viewed as a yes-no property (calibrated or not), computing the empirical expected calibration error may not be enough, as this number will likely be nonzero on any finite dataset. 
Our confidence interval can be used in this setting. If it touches zero, then the model provider can argue that their model is indistinguishable from being calibrated (based on the given test data).

In this work, we focus on the problem of evaluating top-1-to-$k$ calibration in $K$-class classification, i.e., the calibration of the largest $k$ ($k\le K$) predicted probabilities. 
This metric encompasses both the widely used confidence calibration when $k=1$ \citep{guo2017calibration} and the strongest form, multi-class calibration, when $k=K$.  
% {We developed asymptotic distributions of a debiased estimator for top-1-to-$k$ ECE under suitable smoothness and binning conditions, which lead to asymptotically valid confidence intervals for small $k$.}
% We construct asymptotically valid confidence intervals for the $\ell_2$ ECE. 
% {under suitable smoothness and binning conditions.}
We make the following contributions: 

\begin{itemize}
\item We propose a debiased estimator for $\ell_2$ ECE and establish its asymptotic normality {under suitable smoothness and binning conditions.} We show that the convergence rate and asymptotic variance of the estimator differ for calibrated and mis-calibrated models. Using the limiting distribution of the estimator, we develop a confidence interval for $\ell_2$ ECE that accounts for different limiting distributions and the non-negativity of $\ell_2$ ECE. Our theoretical results show that these confidence intervals provide asymptotically correct coverage rates {for small $k$}.

\item We validate our theoretical results through extensive experiments. Simulations support our theoretical findings on the coverage rate of the confidence intervals. Compared to general methods for constructing confidence sets, such as subsampling \citep{politis1994large} and the HulC \citep{kuchibhotla2024hulc}, our method produces intervals with shorter lengths. We also show that the bootstrap undercovers in finite samples, especially for models with a small calibration error. While resampling methods are commonly used for constructing error bars in calibration studies \citep[e.g.,][]{harrell2015regression, vaicenavicius2019evaluating}, our results indicate they may not be reliable. Our method offers an easy-to-use approach for constructing valid confidence intervals.
We also conduct experiments on the benchmark CIFAR10 and CIFAR100 datasets with popular deep learning architectures (e.g. ResNet \citep{he2016deep}) and an experiment on Alzheimer's disease prediction. Our methods can provide a quantitative evaluation of whether the calibration of two models is significantly different.

\item Our proofs include the following innovations:
\begin{enumerate}
    \item To analyze our debiased estimator, we address the sum of dependent random variables across different bins. Deriving the limit law of the estimator requires a Poissonization argument. We specify the conditions needed to obtain the limit law of the conditional distribution 
    % (Lemma \ref{de_poissonization_uncorelated}), 
    (see Lemma 2.3 in the supplementary material),
    filling gaps in existing analyses \citep[see e.g.,][]{beirlant1998asymptotic, gyorfi2015asymptotic}.
    \item We extend the existing Poissonization argument to handle non-zero correlations between the components of the Poissonized variables 
    % (Lemma \ref{de_poissonization}). 
    (Lemma 2.4 in the supplementary material). 
    This novel technical tool may be useful beyond our paper.
\end{enumerate}
\end{itemize}

% Our experiments can be reproduced with the code available at \url{https://github.com/sylydya/CI-ECE}. 
Our experiments can be reproduced with the code in the supplementary material.

\begin{figure}
    \centering
\begin{tikzpicture}[scale=5]

% Model
\tikzstyle{inputnode} = [circle, minimum width=0.6cm, minimum height=0.6cm, text centered, draw=black, fill=green!50]
\tikzstyle{hiddennode} = [circle, minimum width=0.6cm, minimum height=0.6cm, text centered, draw=black, fill=orange!50]
\tikzstyle{outputnode} = [rectangle, rounded corners, minimum width=0.6cm, minimum height=0.6cm,text centered, draw=black, fill=purple!50]
\tikzstyle{arrow} = [thick,->,>=stealth]

\draw[->, thick,] (1.3, 1.5) -- (1.3, 1.2);
\node[right] at (1.35, 1.35){Evaluate Calibration};
\node[right] at (0,1.8){Features: $X^{(1)},\dots, X^{(n)}$};
\node[right] at (1.7,1.8){$Z^{(1)} = f\bigl(X^{(1)}\bigr),\dots, Z^{(n)} = f\bigl(X^{(n)}\bigr)$};
\node[right] at (1.7,2.0){Labels: $Y^{(1)},\dots, Y^{(n)}$};
\node[right] at (1.2, 2.3){$f$};

\node(X1)[inputnode] at (1,2.0){};
\node(X2)[inputnode,below of=X1]{};
\node(X3)[inputnode,below of=X2]{};

\node(Y1)[hiddennode,right of=X1,yshift=0.5cm] {};
\node(Y2)[hiddennode,right of=X2,yshift=0.5cm]{};
\node(Y3)[hiddennode,right of=X3,yshift=0.5cm]{};
\node(Y4)[hiddennode,right of=X3,yshift=-0.3cm]{};

\node(Z1)[hiddennode,right of=Y1]{};
\node(Z2)[hiddennode,right of=Y2]{};
\node(Z3)[hiddennode,right of=Y3]{};
\node(Z4)[hiddennode,right of=Y4]{};

\node(U1)[outputnode,right of=Z1,yshift=-0.5cm]{};
\node(U2)[outputnode,right of=Z2,yshift=-0.5cm]{};
\node(U3)[outputnode,right of=Z3,yshift=-0.5cm]{};

\draw[->](X1)--(Y1);
\draw[->](X1)--(Y2);
\draw[->](X1)--(Y3);
\draw[->](X1)--(Y4);

\draw[->](X2)--(Y1);
\draw[->](X2)--(Y2);
\draw[->](X2)--(Y3);
\draw[->](X2)--(Y4);

\draw[->](X3)--(Y1);
\draw[->](X3)--(Y2);
\draw[->](X3)--(Y3);
\draw[->](X3)--(Y4);

\draw[->](Y1)--(Z1);
\draw[->](Y1)--(Z2);
\draw[->](Y1)--(Z3);
\draw[->](Y1)--(Z4);

\draw[->](Y2)--(Z1);
\draw[->](Y2)--(Z2);
\draw[->](Y2)--(Z3);
\draw[->](Y2)--(Z4);

\draw[->](Y3)--(Z1);
\draw[->](Y3)--(Z2);
\draw[->](Y3)--(Z3);
\draw[->](Y3)--(Z4);

\draw[->](Y4)--(Z1);
\draw[->](Y4)--(Z2);
\draw[->](Y4)--(Z3);
\draw[->](Y4)--(Z4);

 \draw[->](Z1)--(U1);
 \draw[->](Z1)--(U2);
 \draw[->](Z1)--(U3);
 
 \draw[->](Z2)--(U1);
 \draw[->](Z2)--(U2);
 \draw[->](Z2)--(U3);

 \draw[->](Z3)--(U1);
 \draw[->](Z3)--(U2);
 \draw[->](Z3)--(U3);

 \draw[->](Z4)--(U1);
 \draw[->](Z4)--(U2);
 \draw[->](Z4)--(U3);

% Estimator 
\node[right] at (0.3, 1.0) {Estimate {\bf Top-1-to-$k$ Calibration Error}:};
\node[right] at (0.3,0.85) {$T_{m,n} = \frac1n\sum_{
\substack{1 \le  i \le  \ell_{m,n}\\ |\mathcal{I}_{m,n, i}| \ge  2}} 
\frac{1}{|\mathcal{I}_{m, n, i}| - 1} \sum_{a\neq b \in \mathcal{I}_{m, n, i}} \Bigl(Y^{(a)}_{r_{1:k}} - Z^{(a)}_{(1:k)}\Bigr)^{\top}\Bigl(Y^{(b)}_{r_{1:k}} - Z^{(b)}_{(1:k)}\Bigr)$};

\node[right, rectangle, rounded corners, minimum width=9cm, minimum height=2.5cm,text centered, draw=red] at (1.28,0.33) {};
\node[right] at (1.3, 0.5) {{\bf Asymptotic Normality of ${ T_{m,n}}$:}};
\node[right] at (1.3, 0.35) {$\mathrm{ECE}_{1:k} = 0: \, \, n\sqrt{w}T_{m,n}/\sigma_{0}  \xrightarrow{d} \N(0, 1)$};
\node[right] at (1.3, 0.2) {$\mathrm{ECE}_{1:k} \neq 0: \, \, \sqrt{n}(T_{m,n} - \EE{T_{m,n}}) / \sigma_{1} \xrightarrow{d} \N(0, 1)$};

\draw[->, thick](2, 0.05)--(2, -0.1);

% Support of Z
% Axes
\draw[->] (0,0) -- (1.2,0) node[right] {$z_1$};
\draw[->] (0,0) -- (0,1.2) node[above] {$z_2$};

% Simplex boundaries
\draw (0,1) -- (1,0);

% Shaded region for \Delta(K,k)
\fill[gray!20] (0.4,0) -- (1,0) -- (0.5, 0.5) -- (0.2,0.2) -- (0.4,0) -- cycle;

% Labels
\node[right] at (0.75,0.25) {$\Delta(K,k)$};
\node[below left] at (0,0) {$0$};
\node[below right] at (1,0) {$1$};
\node[above left] at (0,1) {$1$};

\node[right] at (0.43,0.58) {$\mathcal{I}_{m,n, i}: \{j: Z^{(j)}_{(1:k)} \in B_i\}$};
\draw[->] (0.63,0.53) -- (0.55, 0.35);

\draw[dashed] (0.3,0.1) -- (0.3, 0.3);
\draw[dashed] (0.4,0.0) -- (0.4, 0.4);
\draw[dashed] (0.5,0.0) -- (0.5, 0.5);
\draw[dashed] (0.6,0.0) -- (0.6, 0.4);
\draw[dashed] (0.7,0.0) -- (0.7, 0.3);
\draw[dashed] (0.8,0.0) -- (0.8, 0.2);
\draw[dashed] (0.9,0.0) -- (0.9, 0.1);

\draw[dashed] (0.4,0.4) -- (0.6, 0.4);
\draw[dashed] (0.3,0.3) -- (0.7, 0.3);
\draw[dashed] (0.2,0.2) -- (0.8, 0.2);
\draw[dashed] (0.3,0.1) -- (0.9, 0.1);

% Dashed lines for z1 = z2
\draw[dashed] (0,0) -- (0.5, 0.5);
% \node at (1.1,0.8) {$z_1 = z_2$};

% Dashed lines for k/K
\draw[dashed] (0,0.4) node[left] {$\frac{k}{K}$} -- (0.4,0) node[below] {$\frac{k}{K}$};

% Adjusted Confidence Interval

\node[right, rectangle, rounded corners, minimum width=18cm, minimum height=4.5cm,text centered, draw=red] at (-0.4,-0.6) {};
\node[right] at (-0.4, -0.25){{\bf Adjusted Confidence Interval ${ C_{m,n}}$:}};
\node[right] at (-0.4, -0.5){
    $
    C_{m, n, 1} = 
    \begin{cases}
         [T_{m,n}^{+} - z_{\alpha/2} \hat\sigma_1/\sqrt{n} ,\quad T_{m,n}^{+} + z_{\alpha/2} \hat\sigma_1/\sqrt{n}], & \text{if} \quad T_{m,n}^{+} / 2 \le  T_{m,n}^{+} - z_{\alpha/2} \hat\sigma_1/\sqrt{n}, \\ 
        [\max\{0, T_{m,n}^{+} - z_{\alpha}\hat\sigma_1/\sqrt{n}\}, \quad T_{m,n}^{+} + z_{\alpha/2} \hat\sigma_1/\sqrt{n}] \setminus \{0\}, & \text{if} \quad T_{m,n}^{+} - z_{\alpha} \hat\sigma_1/\sqrt{n} < T_{m,n}^{+} / 2, \\
         [T_{m,n}^{+} / 2, \quad T_{m,n}^{+} + z_{\alpha/2} \hat\sigma_1/\sqrt{n}], & 
        \text{otherwise}.
    \end{cases}
    $
};
\node[right] at (-0.4,-0.75){$C_{m,n} = C_{m,n,1}\cup \{0\}, \quad \text{if} \quad T_{m,n}^{+} < {z_{\alpha} \sigma_0}/{(n\sqrt{\textnormal{Vol}(B_1)})}$};

\node[right] at (-0.4,-0.9){
{\bf Asymptotically Exact Coverage}: $P(\mathrm{ECE}_{1:k}^2 \in C_{m,n}) \rightarrow 1-\alpha$ 
};

\draw[->, thick](1.5, -1.05)--(1.5, -1.2);

% Comfidence Interval Examples 

\draw[thick, ->, color=black] (0.2, -1.5) -- (2.7, -1.5);
\node[right] at (1.4, -1.6) {Calibration Error};

\draw[thick, color=green] (0.65, -1.2) -- (1.6, -1.2);
\node[color=green, right] at (1.0, -1.1) {Model 1};
\node[circle, fill=green!50] at (1.125, -1.2){};
\draw[thick, color=green, -] (0.65, -1.15) -- (0.65, -1.25);
\draw[thick, color=green, -] (1.6, -1.15) -- (1.6, -1.25);
\node[color=orange] at ( 0.5,-1.3) {Model 2};
\draw[thick, color=orange] (0.3, -1.4) -- (0.8, -1.4);

\node[circle, fill=orange!50] at (0.4, -1.4){};
\draw[thick, color=orange, -] (0.3, -1.35) -- (0.3, -1.45);
\draw[thick, color=orange, -] (0.8, -1.35) -- (0.8, -1.45);

\draw[thick, color=purple] (1.3, -1.4) -- (2.3, -1.4);
\node[circle, fill=purple!50] at (1.8, -1.4){};
\draw[thick, color=purple, -] (1.3, -1.35) -- (1.3, -1.45);
\draw[thick, color=purple, -] (2.3, -1.35) -- (2.3, -1.45);

\draw[dashed] (0.3, -1.7) -- (0.3, -1.2);
\node at (0.25, -1.55){0};

\node[right, color=purple] at ( 1.6,-1.25) {Model 3 $C_{m,n}$};
\draw [decorate, decoration={brace, amplitude=0.5em}] (1.3, -1.33) -- (2.3, -1.33);
\draw [->] (2.3, -1.55)--(1.8, -1.4);
\node[right, color=purple] at ( 2.2,-1.6) {Model 3 $T_{m,n}$};

\end{tikzpicture}
    \caption{An overview of our method, with {\bf our contributions highlighted in red boxes}. For a probability predictor $f$, we compute the debiased estimator of the top-1-to-$k$ calibration error $T_{m,n}$ \eqref{debiased_estimator} and construct confidence intervals $C_{m,n}$ \eqref{aci}. Our theoretical results guarantee that the confidence interval provides asymptotically valid coverage for $\mathrm{ECE}_{1:k}^2$,and our simulations support these theoretical findings (see Figure \ref{simulation_coverage} for details).}

\end{figure}

\subsection{Related Work}

There is a substantial body of related work on evaluating the calibration of prediction methods, constructing confidence intervals, and deriving limit laws for nonparametric estimators. Here we review the most closely related studies. {Additional discussion can be found in Section 1.2 in the Supplementary Material}.

\paragraph{Calibration.} The origins of calibration for classification can be traced back {at least} to meteorology research in the early 1900s \citep{hallenbeck1920forecasting}, where forecasts were expressed as probabilities and assessed against empirical observations. Since then, calibration has been recognized as a crucial aspect of probabilistic predictions in various applications, including weather forecasting \citep{murphy1998early}, medical diagnosis \citep{van2015calibration,van2016calibration,van2019calibration}, image classification \citep{minderer2021revisiting}, and natural language processing \cite{jiang2021can,huang2024uncertainty}.

\paragraph{Calibration Metrics.} Various metrics have been used to evaluate model calibration. Proper scoring rules \cite{gneiting2007strictly}, such as the Brier score \citep{brier1950verification} and negative log-likelihood \citep{winkler1996scoring}, define loss functions between the predictive and true distributions. At the population level, these metrics are minimized when the predicted distribution matches the true distribution.
Recently, \cite{guo2017calibration} highlighted the mis-calibration issue of modern neural networks and the Expected Calibration Error (ECE) became the most widely used metric for calibration, hence our focus in this work. 

Variants of ECE include the confidence calibration error \citep{guo2017calibration}, class-wise calibration error, and multi-class calibration error \citep{kull2019beyond}. In practice, the ECE is typically estimated using a fixed binning scheme, but it can be biased and sensitive to the choice of bins \citep{kumar2019verified, nixon2019measuring, tao2023benchmark}. To address this, \cite{zhang2020mix} proposed a kernel density estimation of ECE without relying on binning, \cite{kumar2019verified} introduced a debiased ECE estimator for probability predictors with a finite number of outputs, and \cite{lee2023t} proposed a debiased estimator for multi-class calibration error, developing a hypothesis test framework to detect whether a model is significantly mis-calibrated. 
We extend the debiased estimator proposed in \cite{lee2023t} to top-1-to-$k$ calibration error and investigate its limiting distribution to construct a confidence interval.\footnote{{A detailed comparison of our method and the hypothesis testing approach proposed in \cite{lee2023t} is given in Section 1.2 in the supplementary material. The key difference is that we establish the limiting distribution of the ECE, which enables confidence intervals. This was not studied in \cite{lee2023t}.}}

Other than ECE, \cite{kumar2018trainable} introduced a kernel calibration error based on maximum mean discrepancy, \cite{blasiok2023unifying} proposed defining the distance from calibration as the $\ell_1$ distance to the closest calibrated predictor, and \cite{gupta2020calibration} developed calibration metrics based on the Kolmogorov-Smirnov test.

\paragraph{Limit Laws of Nonparametric Estimators.} 
Studying the limit law of the binned ECE estimator requires a Poissonization argument. 
Similar approaches have been used in \cite{beirlant1994asymptotic,beirlant1998asymptotic} to derive the asymptotic distribution of $\ell_2$ errors in histogram function estimation. The key step is a partial inversion argument for obtaining characteristic functions of conditional distributions, which goes back at least to \cite{bartlett1938characteristic}, see also \cite{holst1979two,esty1983normal,holst1979asymptotic}, among others.

\paragraph{Constructing Confidence Intervals.} In this work, we construct valid confidence intervals for the $\ell_2$ ECE by studying the estimator's asymptotic distribution. There are generic techniques for constructing confidence intervals, such as resampling-based methods like the bootstrap \citep{efron1979bootstrap} and subsampling \citep{politis1994large}. The bootstrap requires that the estimator be Hadamard differentiable \citep{dumbgen1993nondifferentiable}, while subsampling requires knowing the rate of convergence of the estimator. More recently, generic approaches have been developed to construct valid confidence intervals with minimal regularity conditions, such as universal inference \citep{wasserman2020universal} and the HulC \citep{kuchibhotla2024hulc}, which use sample-splitting techniques.
{The universal inference framework is based on likelihood ratio statistics and therefore requires specification of an underlying likelihood. It is not applicable to our problem as we do not have a likelihood model for the inferential target ECE.}  
Compared to {the HulC}, for the problem of confidence intervals for ECE, our method is computationally more efficient (no resampling required) and produces confidence intervals with shorter lengths. In particular, our method produces shorter intervals because it relies on the asymptotic distribution instead of sample-splitting, which would effectively reduces the sample size, leading to wider intervals.

% \paragraph{Broader Work on Uncertainty Quantification for Machine Learning.}
% Beyond calibration, an important line of work aims to quantify the uncertainty in the predictions of machine learning models by constructing prediction sets, as opposed to point prediction sets.
% The idea of prediction sets dates back at least to the pioneering works of \cite{Wilks1941}, \cite{Wald1943}, \cite{scheffe1945non}, and \cite{tukey1947non,tukey1948nonparametric}.
% More recently conformal prediction has emerged as a prominent methodology for constructing prediction sets \citep[see, e.g.,][]{saunders1999transduction,vovk1999machine,vovk2005algorithmic,Vovk2013, lei2013distribution,angelopoulos2021gentle,guan2023localized, romano2020classification,liang2023conformal,dobriban2023symmpi}. 
% Predictive inference methods
% \citep[e.g.,][etc]{geisser2017predictive}  
% have been developed under various assumptions
% \citep[see, e.g.,][]{bates2021distribution,park2021pac,park2022pac,sesia2022conformal,qiu2023prediction,li2022pac,kaur2022idecode,si2023pac,lee2024simultaneous}.

\subsection{Notation}
For a positive integer $d$, we denote $[d]:=\{1,\ldots,d\}$.
$ \Cov X$ denotes the covariance matrix of a random vector $X$ and the support of $X$ is denoted by $\supp(X)$.
For $p\ge 1$, a positive integer $d$ 
and a vector $v\in \R^d$ we write $\|v\|_p:=\bigl(\sum_{j=1}^d |v_j|^p\bigr)^{1/p}$ and
$\|v\|_\infty:=\max_{j=1}^d |v_j|$;
further $\|v\| = \|v\|_2$.
Denote by $I(\cdot)$ the indicator of an event.
We write $X_n\xrightarrow{d}X$ if 
a sequence of random variables $(X_n)_{n\ge 1}$
converges in distribution to a random variable $X$,
and $X_n\rightarrow_{p}c$ if 
a sequence of random variables $(X_n)_{n\ge 1}$
converges in probability to a constant $c$.
For $\alpha \in (0,1)$, let $z_\alpha$ be the $(1-\alpha)$th quantile of the standard normal distribution.
For two positive sequences $(a_n)_{n\ge 1}$ and $(b_n)_{n\ge 1}$, we write $a_n = O(b_n)$ if for some $C>0$ that does not depend on $n$, we have $a_n \le Cb_n$, and write 
$a_n = o(b_n)$ or $a_n \ll  b_n$ if $a_n/b_n\to 0$ as $n \to \infty$. 
We also write $a_n \asymp b_n$ if $0 < \lim\inf_n a_n / b_n < \lim\sup a_n / b_n < \infty$.

\section{Formulation}
\label{form}

% For $K\ge 2$,
% consider a $K$-class classification problem with inputs $X\in\mX$ and one-hot encoded labels $Y \in \mathcal{Y}:=\{(Y_1, \dots, Y_K)^\top \in \{0,1\}^{K}: \sum_{j=1}^{K} Y_j = 1\}$. 
% Thus, for any $j\in[K]$,
% $Y_j=1$ if and only if $j$ is the correct class.
% We want to find a confidence interval for the calibration error of a given predictive model.
% Let $\Delta_{K-1}:=\{(w_1, \dots, w_K)^\top \in [0,1]^{K}: \sum_{j=1}^{K} w_j = 1\}$
% be the simplex of probability distributions over $K$ classes.
% We consider a pre-trained model (probabilistic classifier) $f:\mX\to \Delta_{K-1}$, which for every input $X\in \mX$ outputs
% predicted probabilities $Z = f(X)\in \Delta_{K-1}$ for the $K$ classes.

For $K\ge 2$,
consider a $K$-class classification problem with inputs $X\in\mX$ and one-hot encoded labels $Y \in \mathcal{Y}:=\{(Y_1, \dots, Y_K)^\top \in \{0,1\}^{K}: \sum_{j=1}^{K} Y_j = 1\}$. 
Thus, for any $j\in[K]$,
$Y_j=1$ if and only if $j$ is the correct class.
Let $\Delta_{K-1}:=\{(w_1, \dots, w_K)^\top \in [0,1]^{K}: \sum_{j=1}^{K} w_j = 1\}$
be the simplex of probability distributions over $K$ classes.
We consider a pre-trained model (probabilistic classifier) $f:\mX\to \Delta_{K-1}$, which for every input $X\in \mX$ outputs
predicted probabilities $Z = f(X)\in \Delta_{K-1}$ for the $K$ classes. Suppose we are given a calibration data set $\{(X^{(1)}, Y^{(1)}),\dots,(X^{(n)}, Y^{(n)})\}$ that is independent of the data used to construct the model $f$. For all $i\in[n]$, we define the predicted probabilities $Z^{(i)} = f(X^{(i)}) \in \Delta_{K-1}$,

\subsection{Motivation}
{
To better contextualize our method in the existing literature, we begin with the binary classification setting and relate calibration error estimation to the classical goodness-of-fit problem.
For binary classification (\(K=2\)), estimation of the \(\ell_2\) expected calibration error is closely related to goodness-of-fit testing. Recall that calibration requires $\mathrm{ECE}^2=\mathbb{E}\!\left[\left\|\mathbb{E}[Y-Z\mid Z]\right\|^2\right]=0$. Conditioning on a fixed model prediction \(Z\), let
$O_{1,Z}=\sum_{f(X^{(j)})=Z}\mathbf{1}\{Y^{(j)}=(1,0)^\top\},O_{2,Z}=\sum_{f(X^{(j)})=Z}\mathbf{1}\{Y^{(j)}=(0,1)^\top\}$ be the observed counts for the two labels, and let \(N_Z = O_{1,Z}+O_{2,Z}\) be the total number of observations with prediction \(Z\). The corresponding expected counts are $E_{1,Z}=N_Z Z_1, E_{2,Z}=N_Z Z_2$.
Classical goodness-of-fit tests compare observed and expected counts. For example, Pearson's chi-square statistic \citep{pearson1900} takes the form
\[
\frac{(O_{1,Z}-E_{1,Z})^2}{E_{1,Z}}
+
\frac{(O_{2,Z}-E_{2,Z})^2}{E_{2,Z}}.
\]
In our setting, a closely related quantity is the plug-in estimator of $\left\|\mathbb{E}[Y-Z\mid Z]\right\|^2$:
\[
\left\|
\frac{1}{N_Z}\sum_{f(X^{(j)})=Z}\bigl(Y^{(j)}-Z^{(j)}\bigr)
\right\|^2 
=\frac{(O_{1,Z}-E_{1,Z})^2}{N_Z^2} + \frac{(O_{2,Z}-E_{2,Z})^2}{N_Z^2}.
\]
Calibration requires the observed counts \(O_{1,Z},O_{2,Z}\) to be close to the expected counts \(E_{1,Z},E_{2,Z}\) for all \(Z\in\Delta_1\). In practice, however, a classifier may output essentially arbitrary values in \(\Delta_1\), and it is rare for two observations to have exactly the same predicted probability vector. Hence, instead of conditioning on an exact value of \(Z\), we need to partition the simplex \(\Delta_1\) into bins $\{B_1,\dots,B_m\}$ and aggregate within bins. Let $\mathcal{I}_i=\{j: Z^{(j)}\in B_i\}$ denote the set of indices in bin \(i\). Define the observed counts
$
O_{1,i}
=
\sum_{j\in\mathcal{I}_i}
\mathbf{1}\{Y^{(j)}=(1,0)^\top\},\,\,
O_{2,i}
=
\sum_{j\in\mathcal{I}_i}
\mathbf{1}\{Y^{(j)}=(0,1)^\top\},
$
and the corresponding expected counts
$
E_{1,i}
=
\sum_{j\in\mathcal{I}_i} Z^{(j)}_1,\,\,
E_{2,i}
=
\sum_{j\in\mathcal{I}_i} Z^{(j)}_2.
$
The Hosmer--Lemeshow-type goodness-of-fit test \citep{hosmer1997comparison} for logistic regression is based on the statistic
\begin{equation}
\label{Hosmer_Lemeshow_statistic}
T_{\mathrm{HL}}
=
\sum_{i=1}^m
\frac{(O_{1,i}-E_{1,i})^2}{E_{1,i}}
+
\frac{(O_{2,i}-E_{2,i})^2}{E_{2,i}}.
\end{equation}
}

{
If we instead replace the denominators \(E_{1,i}\) and \(E_{2,i}\) by \(|\mathcal{I}_i|=E_{1,i}+E_{2,i}\), and normalize the estimator by \(n\), we obtain an plug-in estimator of the \(\ell_2\) ECE:
\begin{equation}
\label{Hosmer_Lemeshow_version}
T_m
=
\frac{1}{n}\sum_{i=1}^m
\frac{(O_{1,i}-E_{1,i})^2+(O_{2,i}-E_{2,i})^2}{|\mathcal{I}_i|}
=
\sum_{i=1}^m
\frac{|\mathcal{I}_i|}{n}
\left\|
\frac{1}{|\mathcal{I}_i|}
\sum_{j\in\mathcal{I}_i}
\bigl(Y^{(j)}-Z^{(j)}\bigr)
\right\|^2.
\end{equation}
where $|\mathcal{I}_i|/n$ estimates the probability $\mathbb{P}(Z\in B_i)$.
}

{
In what follows, we extend this perspective from binary classification to top-1-to-\(k\) predicted probabilities in the multi-class setting, and we propose a debiased estimator of the \(\ell_2\) ECE that debiases
the within-bin plug-in estimator $\|\sum_{j\in\mathcal{I}_i}
\bigl(Y^{(j)}-Z^{(j)}\bigr) / |\mathcal{I}_i| \|^2$ through a \(U\)-statistic-type approach. For large $n$ and a fixed number of bins $m$, the Hosmer--Lemeshow statistic \eqref{Hosmer_Lemeshow_statistic} is approximately distributed as a chi-square random variable with \(m-2\) degrees of freedom. In contrast, we study the regime in which the number of bins also goes to infinity, and establish the asymptotic distribution of the resulting debiased estimator.
}

\subsection{Top-1-to-k Calibration Error}
We consider evaluating the top-1-to-$k$ calibration of the model $f$ for some $1\le  k \le  K$. 
For any $a\in[K]$, 
let $f(X)_{(a)}$ be the $a$-th largest entry of $f(X)$,
and let $f(X)_{(1:k)}$ denote the largest $k$ entries and
sorted in a non-increasing order; breaking ties arbitrarily if needed. 
Let $r_{1:k} = (r_1, \dots, r_k)^\top$ be the corresponding classes with the largest $k$ predicted probabilities, and let $Y_{r_{1:k}} = (Y_{r_1}, \dots, Y_{r_k})^\top$ be the associated label indicators. 
Let $P$ denote the distribution of $(Z,Y)$, and let $P_Z$ denote the distribution of $Z$.

Then a model is top-$1$-to-$k$ calibrated if 
$P_Z$-almost-surely,

% \vspace{-1em}
\begin{equation}
\label{top_1_to_k_calibration}
\EE{Y_{r_{1:k}}| f(X)_{(1:k)} = Z_{(1:k)}} = Z_{(1:k)}.
\end{equation}
% \vspace{-2em}

Equivalently, for almost every top-$k$ predicted probabilities $z_{(1:k)}$ with respect to the distribution of $Z_{(1:k)}$, given that the predicted probabilities of the top $k$ classes are $z_{(1:k)}$, the probabilities that  $Y_j=1$ are given by $z_j$, for all $j$ among the top $k$ predicted probabilities.
When $k = 1$, definition \eqref{top_1_to_k_calibration} recovers the widely used notion of confidence calibration
 \citep{guo2017calibration}. 
When $k = K$, it becomes multi-class or full calibration  \citep[see e.g.,][etc]{harrell2015regression, kull2019beyond}. 

\begin{figure}
    \centering
    
\begin{tikzpicture}[scale=3]
    % Axes
    \draw[->] (0,0) -- (1.2,0) node[right] {$z_1$};
    \draw[->] (0,0) -- (0,1.2) node[above] {$z_2$};
    
    % Simplex boundaries
    \draw (0,1) -- (1,0);
    
    % Shaded region for \Delta(K,k)
    \fill[gray!20] (0.4,0) -- (1,0) -- (0.5, 0.5) -- (0.2,0.2) -- (0.4,0) -- cycle;
    
    % Labels
    \node at (0.5,0.2) {$\Delta(K,k)$};
    \node[below left] at (0,0) {$0$};
    \node[below right] at (1,0) {$1$};
    \node[above left] at (0,1) {$1$};

    % Dashed lines for z1 = z2
    \draw[dashed] (0,0) -- (1, 1);
    \node at (1.1,0.8) {$z_1 = z_2$};
    
    % Dashed lines for k/K
    \draw[dashed] (0,0.4) node[left] {$\frac{2}{K}$} -- (0.4,0) node[below] {$\frac{2}{K}$};
\end{tikzpicture}
    \caption{Illustration of the truncated Weyl chamber $\Delta(K,k)$ for $K=5$ and $k=2$. The shaded region represents the set of vectors $(z_1, z_2)$ satisfying $\frac{k}{K} \le  z_1 + z_2 \le  1, \, \, z_1 \ge  z_2$.}
    \label{fig:Ck}
\end{figure}

The top-1-to-$k$ (squared) $\ell_2$ Expected Calibration Error (ECE) is the measure of mis-calibration defined as
% \vspace{-1em}
\begin{equation}
\label{top_1_to_k_ece}
\mathrm{ECE}_{1:k}^2:= \EE{\bigl\|\EE{Y_{r_{1:k}}-Z_{(1:k)}|Z_{(1:k)}}\bigr\|^2},
\end{equation}
% \vspace{-3em}

where we recall that  $\|\cdot\| = \|\cdot\|_2$ is the $\ell_2$ norm, and
where dependence on $f$ is not displayed for simplicity.
We can see that $\mathrm{ECE}_{1:k}^2\ge0$, and moreover $f$ is top-1-to-$k$-calibrated---so that \eqref{top_1_to_k_calibration}  holds---if and only if $\mathrm{ECE}_{1:k}^2=0$. In what follows, we will treat $k, K$ as fixed.

\subsection{Estimating the Calibration Error}
\label{part+estim}
{Similar to \eqref{Hosmer_Lemeshow_version}, estimating the top-1-to-$k$ $\ell_2$ ECE requires partition of the space of predictions. This section describes how to partition the space of $Z_{(1:k)}$ and introduces our debiased estimator.}
The top-$k$ predicted probabilities  $f(x)_{(1:k)}$ are sorted in a 
{non-increasing order}, and thus for all $x$ belong to a special set we call
the \emph{truncated Weyl chamber}
\begin{equation}\label{Deck}
\Delta(K, k) :=
\biggl\{(z_1, \dots, z_k):\, z_1 \ge z_2\ge \cdots \ge  z_k \ge 0, \,\, \frac{k}{K} \le  \sum_{i=1}^{k} z_i \le  1  \biggr\}.
\end{equation}
This is obtained by cutting off the tip of the intersection of the simplex with the Weyl chamber, see \Cref{fig:Ck}.
To evaluate model calibration, we define a partition of a superset of this set,
whose size depends on a hyperparameter $m>0$ that we will set later, and construct an estimator where the regression functions appearing in \eqref{top_1_to_k_ece} are estimated by a constant over each bin. 

We now define a collection of sets
$\mathcal{B}_m = \{B_1, \dots, B_{\ell_{m,n}}\}$, 
where $\ell_{m,n}>0$. 
For $k = K$, we let $\mathcal{B}_m$ be an equal-volume binning scheme for the probability simplex $\Delta(K, K) = \Delta_{K-1}$ 
which partitions $\Delta_{K-1}$ into $m^{K-1}$ equal volume simplices is introduced in Section 8.2.3 of \cite{lee2023t}. 
For $k < K$, we instead
start by considering a cubic partition of $\mathbb{R}^{k}$. 
For a positive integer $m \in \mathbb{N}_{>0}$
and a $k$-vector of integers
$\vec{i} = (i_1, \dots, i_k)^\top$, 
let $R_{m, \vec{i}} := \prod_{j=1}^{k} [\frac{i_j}{m K}, \frac{i_j + 1}{m K})$ be a hypercube in $\mathbb{R}^{k}$. 
We take the sets
$R_{m, \vec{i}}$
for which 
$R_{m, \vec{i}} \cap \Delta(K,k) $ is nonempty  to form a collection of sets (also referred to as a partition, despite not partitioning $\Delta(K,k)$)
$\mathcal{B}_m = \{B_1, \dots, B_{\ell_{m,n}}\}$, 
where $\ell_{m,n}>0$ 
is the size of the partition. 
These sets have equal volumes, are disjoint, and their union covers $\Delta(K,k)$.
As $m$ grows, their union approximates $\Delta(K,k)$ more and more closely.
Further, by considering the cases $k=K$ and $k<K$ separately, we see that there exists a constant $c>0$ such that $\ell_{m,n} \le  c m^{\min(k,K-1)}$.

Recall that we are given a calibration data set $\{(X^{(1)}, Y^{(1)}),\dots,(X^{(n)}, Y^{(n)})\}$ 
that is independent of the data used to construct the model $f$.
For all $i\in[n]$, we define the predicted probabilities $Z^{(i)} = f(X^{(i)}) \in \Delta_{K-1}$, and
let $U^{(i)} = Y^{(i)}_{r_{1:k}} - Z^{(i)}_{(1:k)}$ be the difference between the top probability predictions and associated labels.
Further, for $i\in[\ell_{m,n}]$,
we let $\mathcal{I}_{m,n,i} := \{j: Z^{(j)}_{(1:k)} \in B_i, 1\le  j \le  n\}$ be the indices of data points in $B_i$.
By extending and slightly modifying the methods of \cite{lee2023t} from full calibration to top-1-to-$k$ calibration, we study the following  
estimator\footnote{{This estimator extends the method of \cite{lee2023t} from full calibration to top-1-to-$k$ calibration.
Compared to the statistic proposed by \cite{lee2023t}, we replace the factor $1/(n|\mathcal{I}_{m,  i}|)$ 
by $1/[n\bigl(|\mathcal{I}_{m,  i}| - 1\bigr)]$ for each bin, 
As shown in Section 2.6 of the supplementary material, our estimator has a smaller bias.
See also \cite{popordanoska2022consistent,gruber2025optimizing} for other estimators of the squared calibration error.}} of $\mathrm{ECE}_{1:k}^2$:
\begin{equation}
\label{debiased_estimator}
    T_{m} = \frac1n\sum_{\substack{1 \le  i \le  \ell_{m},\, |\mathcal{I}_{m, i}| \ge  2}} \frac{1}{|\mathcal{I}_{m,  i}| - 1} \sum_{a\neq b \in \mathcal{I}_{m,  i}} U^{(a)\top}U^{(b)}.
\end{equation}

{Intuitively, this estimator first partitions the features space into the sets $B_i$.
Then, within each set, it estimates the 
conditional expectation
$\bigl\|\EE{Y_{r_{1:k}}-Z_{(1:k)}|Z_{(1:k)\in B_i }}\bigr\|^2
$
by a U-statistic.
Each $U^{(a)} = Y^{(a)}_{r_{1:k}} - Z^{(a)}_{(1:k)}$, $a \in \mathcal{I}_{m,i}$ 
is a random draw from the distribution $P_{i}$ of $U = (Y_{r_{1:k}}-Z_{(1:k)})\mid Z_{(1:k)\in B_i }$.
To estimate $\|\E U\|^2$, we use the $U$-statistic
$
\frac{1}{(|\mathcal{I}_{m,  i}| - 1)|\mathcal{I}_{m,  i}|} \sum_{a\neq b \in \mathcal{I}_{m,  i}} U^{(a)\top}U^{(b)}.
$
Taking the average over all bins $B_i$, weighted by the estimates $|\mathcal{I}_{m,  i}|/n$ of their probabilities $P(B_i)$,
leads to the above estimator.
To gain further insight, we provide a discussion of connections to
nonparametric quadratic functional estimation in Section 1.1 in the Supplementary Material.}

\subsection{Confidence Interval}
\label{ci}

To develop a confidence interval for the calibration error, we will study the asymptotic distribution of the estimator  $T_{m,n}$ given in \eqref{debiased_estimator}. Typically,
given $\alpha\in (0,1)$,
this process involves the following steps:
\begin{itemize}
    \item Establishing a central limit theorem of the form $ \sqrt{n}(T_{m,n} - \EE{T_{m,n}}) / \sigma \xrightarrow{d} \N(0, 1)$, for some $\sigma>0$. 
    \item Developing a consistent variance estimator $\hat{\sigma} \rightarrow_{p} \sigma$.
    \item Constructing a confidence interval using $[T_{m,n} - z_{\alpha/2} \hat\sigma/\sqrt{n} ,\quad T_{m,n} + z_{\alpha/2} \hat\sigma/\sqrt{n}]$.
\end{itemize}
However, it turns out that the asymptotic distribution of the estimator \eqref{debiased_estimator} 
behaves \emph{differently} for calibrated and mis-calibrated models, and this must be considered in the confidence interval. 
Additionally, this standard approach might lead to confidence intervals that cover \emph{negative values}, which are not meaningful since the calibration error is non-negative. 
In this section, we propose an algorithm to construct confidence intervals for $\mathrm{ECE}_{1:k}^2$, 
accounting for the different behaviors of $T_{m,n}$ and non-negativity of $\mathrm{ECE}_{1:k}^2$. The theoretical justification for our proposed approach is provided in Section \ref{theory_section}.

\paragraph{Different scalings of the estimator:}
To build intuition, consider for a moment the much simpler problem of forming a confidence interval for 
the scalar parameter $\theta^2$ (which shares a rough similarity with $\mathrm{ECE}_{1:k}^2$ in terms of being a quadratic parameter), 
based on a sequence of normal observations $X_n\sim \N(\theta,1/n)$, $n\ge 1$, which can be thought of as taking the mean of an i.i.d.~sample (see e.g., Section 3 of \cite{kuchibhotla2024hulc} for similar problems).
We may attempt to form a confidence interval using the estimator $X_n^2$.
Now, when $\theta=0$,
we have $X_n^2\sim \chi^2(1)/n$, 
but when  $\theta\neq0$,
we have by the delta method that
$n^{1/2} (X_n^2-\theta^2)\xrightarrow{d}\N(0,2)$ as $n\to\infty$.
This shows that the limiting distribution, and even the \emph{scaling}, of the estimator $X_n^2$ differs between the two cases ($X_n^2$ has a $1/n$ scaling when $\theta=0$ and an $1/\sqrt{n}$ scaling when $\theta>0$). 
Similarly, for our calibration problem, $T_{m,n}$ also has different limiting distributions for calibrated and mis-calibrated models.

\paragraph{Limiting distributions:} Under mild regularity conditions, we can establish the following limiting distributions for  $T_{m,n}$: 
\begin{itemize}
    \item If $f$ is top-1-to-$k$ calibrated, i.e., equation \eqref{top_1_to_k_calibration} holds and $\mathrm{ECE}_{1:k}^2=0$, then with $w := \text{Vol}(B_1) $, and
    \begin{equation*}
    \sigma_0^2  := 2\int_{\Delta(K, k)} \bigl(\|Z_{(1:k)}\|_2^2 - 2 \|Z_{(1:k)}\|_3^3 + \|Z_{(1:k)}\|_2^4\bigr) dZ_{(1:k)},
    \end{equation*}
    we have $n\sqrt{w}T_{m,n}/\sigma_{0}  \xrightarrow{d} \N(0, 1)$.
    \item Let $U = Y_{r_{1:k}} - Z_{(1:k)} \in \R^k$ be the difference between the top probability predictions and the associated labels. If $f$ is mis-calibrated, with
    \begin{equation*}
    \sigma_1^2 := \Var{\|\EE{U|Z_{(1:k)}}\|^2} + 4\EE{\EE{U^{\top}|Z_{(1:k)}}\Cov{U|Z_{(1:k)}}\EE{U|Z_{(1:k)}}},    
    \end{equation*}
    we have $\sqrt{n}(T_{m,n} - \EE{T_{m,n}}) / \sigma_{1} \xrightarrow{d} \N(0, 1)$.
\end{itemize}
The scaling factor and asymptotic variance of $T_{m,n}$ differ for calibrated and mis-calibrated models.
In particular, $\sigma_0^2$ only depends on the support of $Z_{(1:k)}$. To construct confidence intervals for the ECE using the limiting distribution of the estimator \eqref{debiased_estimator}, we need to estimate the asymptotic variance $\sigma_1^2$.

\paragraph{Estimator of $\sigma_1^2$: } We propose to estimate $\sigma_1^2$ using a plug-in estimator.
For $i\in[\ell_{m,n}]$,
let define the plug-in estimators of $\EE{U| Z_{(1:k)} \in \mathcal{I}_{m,n, i}}$ and $\Cov{U | Z_{(1:k)} \in \mathcal{I}_{m,n, i}}$, respectively as in in \eqref{pi}. 
We define an estimator of $\sigma_1^2$ as in \eqref{sigma1_plug_in_estimator}.
We will show that $\hat{\sigma}_1^2$ is a consistent estimator of $\sigma_1^2$, i.e. $\hat{\sigma}_1^2 \rightarrow_{p} \sigma_1^2$, which implies $\sqrt{n}(T_{m,n} - \EE{T_{m,n}}) / \hat\sigma_{1} \xrightarrow{d} \N(0, 1)$. 
This ensures the asymptotic validity of the standard confidence interval $[T_{m,n} - z_{\alpha/2} \hat\sigma/\sqrt{n} ,\quad T_{m,n} + z_{\alpha/2} \hat\sigma/\sqrt{n}]$ for mis-calibrated models. 
However, the confidence interval may cover negative values, which---as mentioned above---are not meaningful since $\mathrm{ECE}_{1:k}^2$ is non-negative.
Further, this interval does not account for the different behavior of the calibrated model. Next, we provide an adjusted confidence interval to address these issues.

\paragraph{Adjusted confidence intervals:}
While one could simply truncate the confidence interval to contain non-negative values in an \emph{ad hoc} manner, 
we follow a more principled approach that explicitly takes into account non-negativity at the time of interval construction.
We following the general approach for bounded parameters from \cite{wu2012adjusted}.
We define the positive part estimator $T_{m,n}^{+} = \max\{T_{m,n}, 0\}$
and
$T^* = \EE{T_{m,n}^+}$. 
We construct a confidence interval by inverting hypothesis tests $T^*=t^*$, for each $t^*\ge 0$, 
i.e., collecting $t^*$ such that we fail to reject the null hypothesis $T^{*} = t^*$.
Equivalently, these are the $t^*$ values such that 
$P(|T^{+} - t^*| > |t_{m,n}^{+} - t^*|) > \alpha$, 
where the probability is taken only over the random variable $T^{+}$ that has the same distribution as $T_{m,n}^{+}$, and $t_{m,n}^{+}$ is a fixed realization of $T_{m,n}^{+}$.
Following \cite{wu2012adjusted},
this leads to the interval in \eqref{aci}.
Further, if $T_{m,n}^{+}$ is sufficiently small, we include zero in the interval, leading to 
\begin{equation}
\label{adjusted_KI}
C_{m, n} = 
\begin{cases}
C_{m, n, 1} \cup \{0\}, & \text{if}  \quad T_{m,n}^{+} <   z_{\alpha} \sigma_0/(n\sqrt{w}), \\
C_{m, n, 1}, & \text{otherwise}.
\end{cases}
\end{equation}

\begin{algorithm}
\SetAlgoLined
\KwIn{Calibration data set $\{(X^{(1)}, Y^{(1)}),\dots,(X^{(n)}, Y^{(n)})\}$, model $f$,
number $k$ of top classes for which to check calibration,
number of bins $\ell_{m,n}$, significance level $\alpha\in(0,1)$.}
\KwOut{Confidence interval $C_{m,n}$}

\textbf{Step 1: Partition.}

Record predicted probabilities   $Z^{(i)} = f(X^{(i)})$
and prediction errors
    $U^{(i)} = Y^{(i)}_{r_{1:k}} - Z^{(i)}_{(1:k)}$, $i\in [n]$.\;

For the partition $\{B_1, \dots, B_{\ell_{m,n}}\}$ from \eqref{Deck} from \Cref{part+estim},
define the indices of datapoints in each bin:   $\mathcal{I}_{m,n,i} = \{j: Z^{(j)}_{(1:k)} \in B_i, 1\le  j \le  n\}$, $i\in[\ell_{m,n}]$.\;

\textbf{Step 2: Compute debiased top-1-to-$k$ calibration error estimator.}

% Find debiased top-1-to-$k$ calibration error estimator 

$  T_{m,n} = \frac1n\sum_{\substack{1 \le  i \le  \ell_{m,n},\, |\mathcal{I}_{m,n, i}| \ge  2}} \frac{1}{|\mathcal{I}_{m, n, i}| - 1} \sum_{a\neq b \in \mathcal{I}_{m, n, i}} U^{(a)\top}U^{(b)}$\;

\textbf{Step 3: Compute Variance Estimator.}

Compute the variance of a calibrated model: $\sigma_0^2 = 2\int_{\Delta(K,k)} (\|Z_{(1:k)}\|_2^2 - 2 \|Z_{(1:k)}\|_3^3 + \|Z_{(1:k)}\|_2^4) dZ_{(1:k)}$, \;

Compute the variance estimator of a mis-calibrated model:
\begin{equation}
\begin{split}
\label{sigma1_plug_in_estimator}
\hat{\sigma}_1^2 =&  \sum_{i=1}^{\ell_{m,n}} \frac{|\mathcal{I}_{m, n, i}|}{n} \Bigl\Vert\mathbb{E}_n[U]^{(i)}\Bigr\Vert^4 - \Biggl(\sum_{i=1}^{\ell_{m,n}} \frac{|\mathcal{I}_{m, n, i}|}{n} \Bigl\Vert\mathbb{E}_n[U]^{(i)}\Bigr\Vert^2 \Biggr)^2 \\
& + 4 \sum_{i=1}^{\ell_{m,n}} \frac{|\mathcal{I}_{m, n, i}|}{n} \mathbb{E}_n[U]^{(i) \top}\textnormal{Cov}_n[U]^{(i)} \mathbb{E}_n[U]^{(i)}.
\end{split}
\end{equation}
where 
% For $i\in[\ell_{m,n}]$,
% define per-bin estimators of mean and covariance of the prediction error: 
\begin{equation}\label{pi}
\mathbb{E}_n[U]^{(i)} = \frac{1}{|\mathcal{I}_{m,n, i}|} \sum_{j \in \mathcal{I}_{m, n, i}} U^{(j)}, 
\textnormal{Cov}_n[U]^{(i)} = \frac{1}{|\mathcal{I}_{m,n, i}|} \sum_{j \in \mathcal{I}_{m, n, i}} (U^{(j) \top} U^{(j)} - \mathbb{E}_n[U]^{(i)}\mathbb{E}_n[U]^{(i)\top}).
\end{equation}\;

\textbf{Step 4: Construct Confidence Interval.}

Define positive part $T_{m,n}^{+} = \max\{T_{m,n}, 0\}$\;
and normal quantiles
$z_{\alpha/2} = \Phi^{-1}(1-\alpha/2)$,\;
$z_{\alpha} = \Phi^{-1}(1-\alpha)$.\;
Define adjusted positive confidence interval
\begin{equation}\label{aci}
C_{m, n, 1} = 
\begin{cases}
     [T_{m,n}^{+} - z_{\alpha/2} \hat\sigma_1/\sqrt{n} ,\quad T_{m,n}^{+} + z_{\alpha/2} \hat\sigma_1/\sqrt{n}], & \text{if} \quad T_{m,n}^{+} / 2 \le  T_{m,n}^{+} - z_{\alpha/2} \hat\sigma_1/\sqrt{n}, \\ 
    [\max\{0, T_{m,n}^{+} - z_{\alpha}\hat\sigma_1/\sqrt{n}\}, \quad T_{m,n}^{+} + z_{\alpha/2} \hat\sigma_1/\sqrt{n}], & \text{if} \quad T_{m,n}^{+} - z_{\alpha} \hat\sigma_1/\sqrt{n} < T_{m,n}^{+} / 2, \\
     [T_{m,n}^{+} / 2, \quad T_{m,n}^{+} + z_{\alpha/2} \hat\sigma_1/\sqrt{n}], & 
    \text{otherwise}.
\end{cases}
\end{equation}
Let $C_{m,n} = C_{m,n,1} \setminus \{0\}$, and if the estimator value is small enough that $T_{m,n}^{+} < {z_{\alpha} \sigma_0}/{(n\sqrt{\textnormal{Vol}(B_1)})}$, then 
include zero: $C_{m,n} = C_{m,n} \cup \{0\}$.\;

\Return{$C_{m,n}$}\;

\caption{Confidence Interval for the $\ell_2$ Expected Calibration Error}
\label{CI_algorithm}
\end{algorithm}

The full procedure to construct confidence intervals is summarized in Algorithm \ref{CI_algorithm}. In the study of calibration, it is common to obtain error bars using resampling methods (such as in \cite{vaicenavicius2019evaluating}). Our method is more computationally efficient as it does not require resampling or sample splitting.
In the next section, we also provide theoretical justification demonstrating that the confidence interval \eqref{adjusted_KI} is asymptotically valid.

We also mention that sometimes it is more interpretable to form a confidence interval for the calibration error $\mathrm{ECE}_{1:k}$ (rather that the squared error $\mathrm{ECE}_{1:k}^2$).
One can form a valid interval by taking the square roots of all values in $C_{m,n}$. since $C_{m,n}$ is an interval so that $C_{m,n} = [\gamma_{m,n,-},\gamma_{m,n,+}]$, we can take this to be 
$[\gamma_{m,n,-}^{1/2},\gamma_{m,n,+}^{1/2}]$.

\begin{remark}
[Binning Scheme]\label{bin}
In our theoretical analysis, we provide conditions for the binning scheme (Conditions \ref{ass_1} and \ref{ass_3}) to ensure the validity of the confidence interval \eqref{adjusted_KI}. 
In practice, a rule of thumb for the binning scheme can be to minimize the mean squared estimation error, which has the standard bias-variance decomposition $\EE{(T_{m,n} - \mathrm{ECE}_{1:k}^2)^2} = \EE{(T_{m,n} - \EE{T_{m,n}})^2} + \bigl(\EE{T_{m,n}} - \mathrm{ECE}_{1:k}^2\bigr)^2$. 
In Section 
% \ref{bias_var}, 
2.6 in the supplementary material,
we show that the variance is of order 
$\EE{(T_{m,n} - \EE{T_{m,n}})^2} = O\left( m^{\min(k, K-1)} n^{-2}\right) + O\left( n^{-1} \right)$. 
Intuitively, 
the term $O\left( m^{\min(k, K-1)} n^{-2}\right)$ arises because $T_{m,n}$ involves a summation of $\ell_{m,n}=O\left( m^{\min(k, K-1)}\right)$ 
terms divided by $n$, and $O\left( n^{-1} \right)$ arises  from the central limit theorem.
We also show that the bias is of order $\bigl|\EE{T_{m,n}} - \mathrm{ECE}_{1:k}^2\bigr| = O\left( m^{-2s} \right)$, where $s$ is the H\"older smoothness parameter in Condition \ref{ass_1}. 
The bias results from the discrepancy between $\int_{B_i} \left\|\EE{U|Z_{(1:k)}}\right\|^2 d P_{Z_{(1:k)}} $ and $\bigl\| \int_{B_i}\EE{U|Z_{(1:k)}} d P_{Z_{(1:k)}} \bigr\|^2 / \mu(B_i)$ for each bin, where $\mu(B_i) = P(Z_{(1:k)} \in B_i)$ for all $i\le \ell_{m,n}$. 
Choosing $m$ to minimize the asymptotic order of the mean squared estimation error leads to $m \asymp n^{2 /(4s +  \min(k, K-1))}$.
In the special case that $k=K$, this coincides with the optimal binning scheme proposed in \cite{lee2023t}. 
\end{remark}

\section{Theoretical Analysis}
\label{theory_section}
\subsection{Asymptotic Normality}
In this section, we develop theoretical guarantees 
ensuring that Algorithm \ref{CI_algorithm} provides asymptotically valid confidence intervals.
We first introduce some conditions that are required to develop our theoretical results.
For $(Z,Y)\sim P$,
recall that $U = Y_{r_{1:k}} - Z_{(1:k)} \in \R^k$ is the difference between the labels associated with the top probability predictions and the respective predictions.
Condition \ref{ass_1} requires the underlying calibration curve $Z_{(1:k)} \mapsto\EE{U|Z_{(1:k)}}$ to behave regularly, which allows us to control the bias and variance of the binned estimator. 
\begin{condition}[H\"older smoothness]
\label{ass_1}
Each coordinate of the map $z_{(1:k)} \mapsto\EE{U|Z_{(1:k)}=z_{(1:k)}}$ is 
H\"older continuous with H\"older smoothness parameter $ 0 < s \le  1$ and H\"older  constant $L$, i.e., 
for all $z_{(1:k)}, z^{'}_{(1:k)} \in \Delta(K, k)$ and $j \in\{ 1,\dots, k\}$,
$$\Bigl|\EE{U\Bigl|Z_{(1:k)}=z_{(1:k)}}_j - \EE{U\Bigr|Z_{(1:k)}=z^{'}_{(1:k)}}_j \Bigr|\le  L \Bigl\|z_{(1:k)} - z^{'}_{(1:k)}\Bigr\|^s.$$ 
\end{condition}
For the purposes of theoretical analysis, 
we can index the partition elements such that 
for all $j\le \ell_{m,n}-1$,
$P(f(X)_{(1:k)} \in B_{j}) \ge  P(f(X)_{(1:k)} \in B_{j+1})$, as well as $P(f(X)_{(1:k)} \in B_{\ell_{m,n}}) > 0$. 
Our algorithm does not depend on the ordering of the partition elements, so this indexing is only used for the purpose of stating a required theoretical condition.
Condition \ref{ass_2} ensures that all bins contain a growing expected number of data points, but no bin is too large. 
\begin{condition}[Distribution Balance]
\label{ass_2}
We have that $\supp\left(Z_{(1:k)}\right) = \Delta(K, k)$, 
and as $m, n\rightarrow \infty$,
\begin{equation}\label{db}
n P(f(X)_{(1:k)} \in B_{\ell_{m,n}}) \rightarrow \infty,\,\,\, P(f(X)_{(1:k)} \in B_{1}) \rightarrow 0,\,\,\,
\ell_{m,n}\rightarrow \infty.
\end{equation}
\end{condition}

This is a mild condition that holds if the distribution of the predictions $Z=f(X)$  is sufficiently balanced over the simplex.
The condition $\supp\left(Z_{(1:k)}\right) = \Delta(K, k)$ ensures that the support of the distribution of $Z_{(1:k)}$ is the whole set $\Delta(K, k)$; this holds if $Z_{(1:k)}$ has a strictly positive density, even if the density is arbitrarily close to zero.
Recalling our binning scheme from Section \ref{part+estim}, the condition \eqref{db} holds if $Z_{(1:k)}$ has a density strictly bounded away from zero, and $m  = o(n^{{1}/{\min(k,K-1)}})$.
Comparing with the discussion from Remark \ref{bin},
$m  = o(n^{{1}/{\min(k,K-1)}})$ ensures that the term of order $m^{\min(k, K-1)} n^{-2}$ in the variance is negligible compared to the term of order $n^{-1}$.

Under these conditions, we can establish the asymptotic normality of the proposed estimator \eqref{debiased_estimator}. 
For a perfectly calibrated model, we have the following theorem.
Recall that we are using an equal volume partition, 
and $w= \text{Vol}(B_1)$ denotes the volume of each piece.
\begin{theorem}[Limiting distribution for a top-1-to-$k$ calibrated model]
\label{clt_calibrated_model}
If  $f$ is top-1-to-$k$ calibrated, i.e., equation \eqref{top_1_to_k_calibration} holds and $\mathrm{ECE}_{1:k}^2=0$,
then 
under Conditions \ref{ass_1} and \ref{ass_2}, with 
\begin{equation}
\label{var_calibrated}
\begin{split}
 \sigma_0^2  := 2\int_{\Delta(K, k)} \bigl(\|Z_{(1:k)}\|_2^2 - 2 \|Z_{(1:k)}\|_3^3 + \|Z_{(1:k)}\|_2^4\bigr) dZ_{(1:k)},
\end{split}
\end{equation}
as $n, m\rightarrow \infty$, we have
for $T_{m,n}$ from  \eqref{debiased_estimator} that
$n\sqrt{w}T_{m,n}/\sigma_{0}  \xrightarrow{d} \N(0, 1)$.
\end{theorem}

If the model is mis-calibrated, we have the following theorem.
\begin{theorem}[Limiting distribution for a mis-calibrated model]
\label{clt_miscalibrated_model}
Let
\begin{equation}
\label{var_miscalibrated}
\sigma_1^2 := \Var{\|\EE{U|Z_{(1:k)}}\|^2} + 4\EE{\EE{U^{\top}|Z_{(1:k)}}\Cov{U|Z_{(1:k)}}\EE{U|Z_{(1:k)}}}.    
\end{equation}
Under Conditions \ref{ass_1} and \ref{ass_2},
we have $\sigma_1^2>0$.
If $\mathrm{ECE}_{1:k}^2>0$,
we have for $T_{m,n}$ from  \eqref{debiased_estimator} that
$\sqrt{n}(T_{m,n} - \EE{T_{m,n}}) / \sigma_{1} \xrightarrow{d} \N(0, 1)$
 as $n\rightarrow \infty$.
\end{theorem}

The proof uses a \emph{Poissonization argument} 
inspired by---and extending the arguments of---\cite{beirlant1994asymptotic, beirlant1998asymptotic}. 
We let the number of data points be a Poisson random variable $\tilde{N} \sim \text{Poisson}(n)$ and let $\mathcal{I}(\tilde{N})_{i} = \left\{j: Z^{(j)}_{(1:k)} \in B_i, 1\le  j \le  \tilde{N} \right\}$. Then we have
$\bigl|\mathcal{I}(\tilde{N})_{i}\bigr| \sim \text{Poisson}(n\mu(B_i)),$
and $|\mathcal{I}(\tilde{N})_{i}|, i=1,\dots, \ell_{m,n}$ are mutually independent. Then the Poissonized estimator 
\begin{equation*}
% \label{poissonized_estimator}
 \tilde T(\tilde{N}) = \sum_{\substack{1 \le  i \le  \ell_{m,n},\, |\mathcal{I}(\tilde{N})_{i}| \ge  2}} \frac{1}{n\bigl(|\mathcal{I}(\tilde{N})_{i}| - 1\bigr)}  \Biggl[  \sum_{j_1\neq j_2 \in \mathcal{I}(\tilde{N})_{i}} U^{(j_1)\top}U^{(j_2)} \Biggr]
\end{equation*}
is a sum of independent random variables. 
We show the asymptotic normality of $\tilde T(\tilde{N})$ via moment calculations using the Lyapunov Central Limit Theorem. 
Now, the distribution of $T_{m,n}$ is the same as the conditional distribution of $\tilde T(\tilde{N})$ given $\tilde{N} = n$. 
The limit law of the conditional distribution is established using a \emph{partial inversion} approach for obtaining characteristic functions \citep[see e.g.,][etc]{bartlett1938characteristic,holst1979two, esty1983normal}.

In our proof, we also carefully specify the conditions needed to obtain the limit law of the conditional distribution 
% (see Lemma \ref{de_poissonization_uncorelated}), 
(see Lemma 2.3 in the supplementary material), 
which involves splitting the test statistic into two components. Our analysis also extends the results of \cite{beirlant1998asymptotic} by providing a result that enables Poissonization with correlated components 
% (see Lemma \ref{de_poissonization}). 
(Lemma 2.4 in the supplementary material). 
This novel technical tool may be useful beyond our paper.

\subsection{Adjusted Confidence Interval}
\label{vare}
To construct confidence intervals for ECE using the limiting distribution of the estimator \eqref{debiased_estimator}, we also need to estimate the asymptotic variance. 
For calibrated models, $\sigma_0^2$ defined in \eqref{var_calibrated}
can be computed analytically or using standard numerical integration methods \citep{davis2007methods}. 
For a mis-calibrated model, we propose to estimate $\sigma_1^2$ by the plug-in estimator defined in \eqref{sigma1_plug_in_estimator}. The following proposition shows that $\hat{\sigma}_1^2$ is a consistent estimator of $\sigma_1^2$:
\begin{proposition}
\label{sigma1_hat_convergence}
Under Conditions \ref{ass_1} and \ref{ass_2},
as $n\rightarrow \infty$,
$\hat\sigma_1^2 \rightarrow_{p} \sigma_1^2$.
\end{proposition}
The proof can be found in Section \ref{sigma1_hat_convergence}. We show that $\EE{\hat\sigma_1^2}\to\sigma_1^2$ and that $\Var{\hat\sigma_1^2}\to0$. 
The conclusion follows from Chebyshev’s inequality. Then by Slutsky's theorem (e.g., \cite{van2000asymptotic}, Lemma 2.8), we have the following corollary:
\begin{corollary}
\label{clt_estimated_sigma1}
Under Conditions \ref{ass_1} and \ref{ass_2},
as $n\rightarrow \infty$, $\sqrt{n}(T_{m,n} - \EE{T_{m,n}}) / \hat\sigma_{1} \xrightarrow{d} \N(0, 1)$.
\end{corollary}
Then we have the following theorem demonstrating that the confidence interval \eqref{adjusted_KI} provides asymptotically correct coverage for $\EE{T_{m,n}}$:
\begin{theorem}[Asymptotically Valid Inference for the Mean]
\label{infer_mean}
Given a model $f$, for any value 
of $\mathrm{ECE}_{1:k}^2 \in [0,\infty)$
from \eqref{top_1_to_k_ece},
Under Conditions \ref{ass_1} and \ref{ass_2}, as $n\rightarrow \infty$,
$P(\EE{T_{m,n}} \in C_{m, n}) \rightarrow 1-\alpha$.
\end{theorem}
The proof can be found in Section 
% \ref{proof_KI_mean}, 
2.7 in the supplementary material,
and is a direct application of Corollary \ref{clt_estimated_sigma1}. 
Theorem \ref{infer_mean} ensures asymptotically correct coverage for $\EE{T_{m,n}}$, 
which can be viewed as a \emph{smoothed} form of $\mathrm{ECE}_{1:k}^2$, and 
with $\mu(B_i) = P(Z_{(1:k)} \in B_i)$ for all $i\le \ell_{m,n}$,
has the explicit form: 
\begin{equation*}
\begin{split}
\EE{T_{m,n}} & = \frac{1}{n}\sum_{i=1}^{\ell_{m,n}} \EE{|\mathcal{I}_{m,n, i}| I(|\mathcal{I}_{m,n, i}| \ge  2)} \bigl\|\EE{U | Z_{(1:k)} \in B_i}\bigr\|^2\\
& = \frac{1}{n}\sum_{i=1}^{\ell_{m,n}} 
n\mu(B_i)\left[1 - \bigl(1 - \mu(B_i)\bigr)^{n-1}\right]  \left\| \frac{\int_{B_i}\EE{U|Z_{(1:k)}} d P_{Z_{(1:k)}}}{\mu(B_i)} \right\|^2 \\
& \approx \sum_{i=1}^{\ell_{m,n}} \frac{\left\| \int_{B_i}\EE{U|Z_{(1:k)}} d P_{Z_{(1:k)}} \right\|^2}{\mu(B_i)}.
\end{split}
\end{equation*}
Unfortunately, $T_{m,n}$ is not an unbiased estimator of $\mathrm{ECE}_{1:k}^2 = \sum_{i=1}^{\ell_{m,n}} \int_{B_i}\|\EE{U|Z_{(1:k)}}\|^2 d P_{Z_{(1:k)}} $. The bias mainly results from the discrepancy between $\bigl\| \int_{B_i}\EE{U|Z_{(1:k)}} d P_{Z_{(1:k)}} \bigr\|^2 / \mu(B_i)$ and $\int_{B_i} \left\|\EE{U|Z_{(1:k)}}\right\|^2 d P_{Z_{(1:k)}} $ for each bin. 
However, we show that we can obtain a valid confidence interval for $\mathrm{ECE}_{1:k}^2$, under the following additional condition:
\begin{condition}
\label{ass_3}
The binning scheme satisfies $n^{{1}/{(4s)}} = o(m)$. 
\end{condition}
Comparing with the discussion from Remark \ref{bin},
the condition $n^{{1}/{4s}} = o(m)$ ensures that the squared bias of order $m^{-4s}$ is negligible compared to the term of order $n^{-1}$, i.e. $\sqrt{n} (\EE{T_{m,n}}-\mathrm{ECE}_{1:k}^2) \rightarrow_{p} 0$.
If $Z_{(1:k)}$ has a density strictly bounded away from zero,
Condition \ref{ass_3}, together with Condition \ref{ass_2}, requires the binning scheme used to construct the estimator to satisfy $n^{1/4s} \ll  m \ll  n^{{1}/{\min(k,K-1)}}$. 
For instance, for a Lipschitz continuous map $z_{(1:k)} \mapsto\EE{U|z_{(1:k)}}$ (i.e., $s=1$) 
and top-1 calibration (i.e., $k=1$), 
it requires $ n^{{1}/{4}} \ll  m \ll   n$. 
In particular, since $s\le 1$, the condition requires $k < 4$.
In practice, often only the top few probabilities are used, 
and so this is a mild condition. 
{ Additionally, when the model is perfectly calibrated, the estimator is unbiased. Therefore, for hypothesis testing applications, our method remains valid for all $k$. }
With Condition~\ref{ass_3}, 
we can show that the confidence interval provides asymptotically valid coverage for $\mathrm{ECE}_{1:k}^2$.

\begin{theorem}[Asymptotically Valid Inference for $\mathrm{ECE}_{1:k}^2$]
\label{infer_ece}
Under Conditions \ref{ass_1}, \ref{ass_2} and \ref{ass_3},
we have for any value 
of $\mathrm{ECE}_{1:k}^2 \in [0,\infty)$
% from \eqref{top_1_to_k_ece} 
that 
$P(\mathrm{ECE}_{1:k}^2 \in C_{m,n}) \rightarrow 1-\alpha$  as $n\rightarrow \infty$. 
\end{theorem}
The proof can be found in Section 
% \ref{proof_KI_ece}, 
2.8 in the supplementary material,
and follows the proof of Theorem \ref{infer_mean} with an additional argument showing that $\sqrt{n} (\EE{T_{m,n}}-\mathrm{ECE}_{1:k}^2) \rightarrow_{p} 0$.

\section{Experiments}
We perform experiments on both simulated and empirical datasets to support our theoretical results. 
In particular, we use a variety of simulations to both check the finite sample performance of our confidence intervals as well as to empirically compare their length to those of existing methods.
Further we illustrate our methods on two empirical data sets. The first example is the standard CIFAR  image classification data set, which is meant to showcase our method on a popular and widely used benchmark.
The second example is an
Alzheimer’s disease prediction dataset discussed in \cite{wang2023bias}, 
and is meant to illustrate the applicability of our method to a different domain (medical data), while also being an example where our method leads to a different conclusion from a standard method (the bootstrap). {Additional experiments on simulated datasets of different sizes, sensitivity analyses with respect to the number of bins, and further experiments on CIFAR and large language models with verbalized confidence are provided in Section 3 of the Supplementary Material. }

\subsection{Simulated Data}
\label{simulated_data}
We use simulated data to demonstrate our method's ability to generate valid confidence intervals. We consider three data distributions:

\begin{itemize}
    \item Setting 1: We generate {$n = 100$} i.i.d.~data points $\bigl(Z^{(i)}, Y^{(i)}\bigr)$, $i\in[n]$, from a $K = 2$-class classification problem with predicted probability $Z \sim \text{Unif}(\Delta_1)$.
    A one-hot encoded label $Y = (Y_1,Y_2)^\top$ is generated such that
    % \vspace{-1em}
    \begin{equation}\label{lr}
    P(Y_1 = 1 | Z) = \left(1 + \exp\{-\beta \log{[Z_1/(1-Z_1)]}\}\right)^{-1},   
    \end{equation}
    % \vspace{-4em}
    
    for some $\beta \in \mathbb{R}$. We evaluate the top-1 calibration error (i.e. $k=1$) with $mK = 20$.  
    
    \item Setting 2: This setting is as Setting 1, except  the features $Z$ are non-uniformly distributed with $Z_1 \sim \text{Beta}(5, 0.5)$.  
    \item Setting 3: We generate {$n = 100$} i.i.d.~data points $\bigl(Z^{(i)}, Y^{(i)}\bigr)$, $i\in[n]$, from a $K = 10$-class classification problem with predicted probability $Z \sim \text{Unif}(\Delta_9)$. 
    The one-hot encoded label $Y$ is generated such that
    % \vspace{-1em}
    \begin{equation*}
    P(Y_{r_1} = 1 | Z_{(1)}) = Z_{(1)} - \beta, \, \, P(Y_{r_2} = 1 | Z_{(2)}) = Z_{(2)} + \beta, \, \, P(Y_{r_j} = 1 | Z_{(j)}) = Z_{(i)}, \, \forall j > 2,
    \end{equation*}
    % \vspace{-4em}

    for some $\beta \in \mathbb{R}$. We evaluate the top-1-to-2 calibration error (i.e. $k=2$) with $mK = 10$.  
\end{itemize}

\begin{figure}
    \centering
    \includegraphics[scale=0.25]{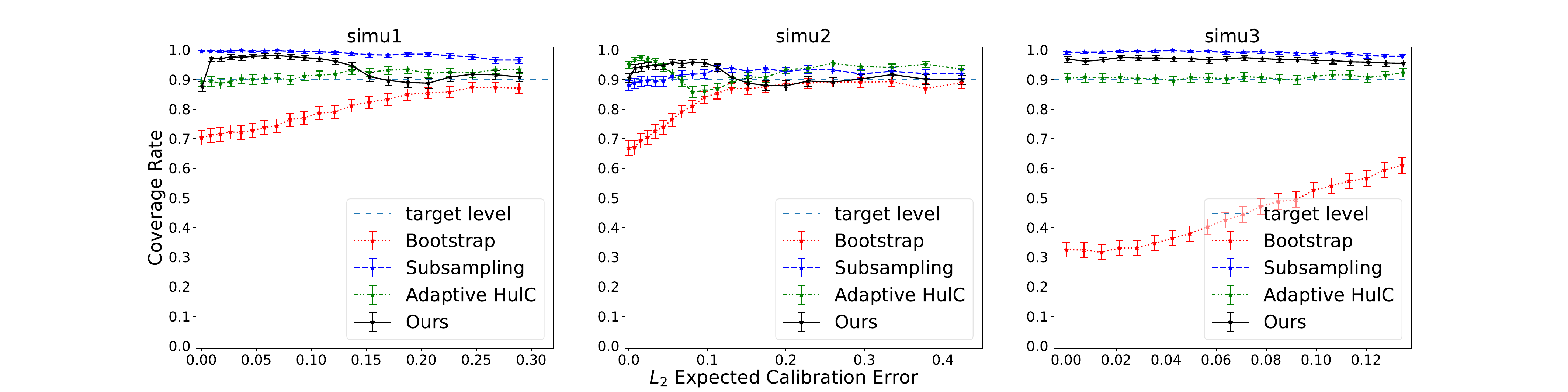}
    \caption{Coverage Rate vs. Calibration Error. For each setting and each value of $\beta$, we compute the coverage rate of the confidence intervals over 1000 datasets. The error bars for coverage rates are generated using the Clopper–Pearson method \citep{clopper1934use}. }
    \label{simulation_coverage}
\end{figure}

Settings 1 and 2 consider binary classification problems, differing only in the distribution of $Z$. Setting 3 addresses a multiclass classification problem, evaluating the top-1-to-2 calibration error. 
The value of $\beta$ controls the calibration error of the model.
For Settings 1 and 2, data is generated with $\beta \in \{0, 0.05, 0.1, \dots, 0.95, 1\}$, where $\beta = 1$ corresponds to a multi-class calibrated model. 
For Setting 3, data is generated with $\beta \in \{0, 0.005, 0.01, \dots, 0.1\}$, where $\beta = 0$ corresponds to a multi-class calibrated model. 
For each value of $\beta$, we generate 1000 datasets and compute the estimator $T_{m,n}$ from \eqref{debiased_estimator}. For Settings 1 and 2, we use a cubic partition with width $1/(mK) = 1 / 20$, and for Setting 3, we set width $1/(mK) = 1 / 10$.
We generate 90\% confidence intervals using our method \eqref{adjusted_KI} for each dataset, 
and report the coverage rate of the confidence intervals over the 1000 datasets. 
As seen in Figure \ref{simulation_coverage},
our method leads to valid confidence intervals, even for a relatively moderate sample size of $n=100$.

\begin{figure}
    \centering
    \includegraphics[scale=0.25]{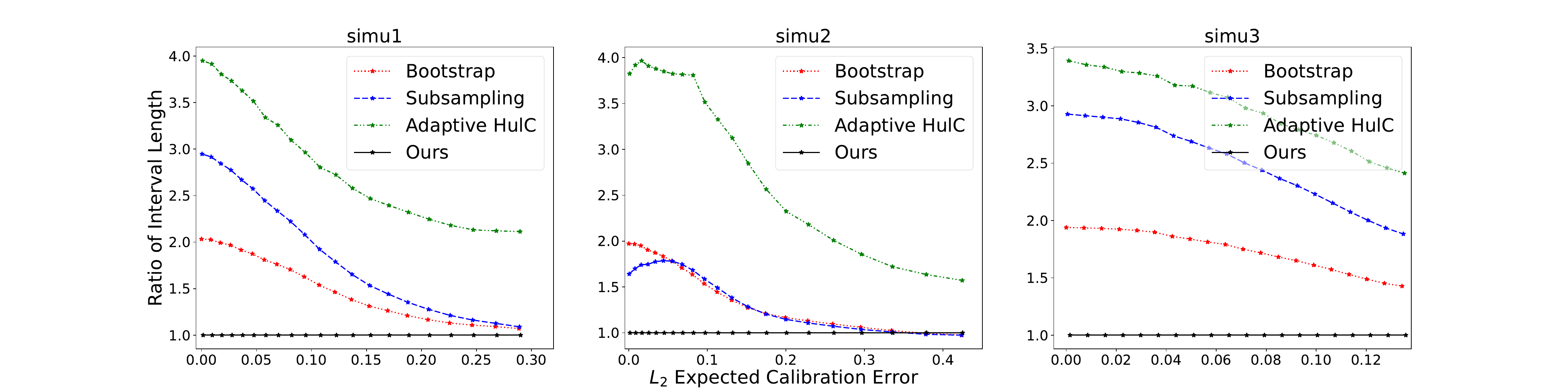}
    \caption{
    % Confidence Interval Length vs. Calibration Error: Over 1000 datasets, we compute the average length of confidence intervals, with error bars representing the 5th and 95th percentiles. In the second row, we report the ratio of the average confidence interval lengths of various methods relative to our method.
    Confidence Interval Length vs. Calibration Error: Over 1000 datasets, we report the ratio of the average confidence interval lengths of various methods relative to our method.
    }
    \label{simulation_interval_length}
\end{figure}

For comparison, we also generate confidence intervals using the bootstrap \citep{efron1979bootstrap}, subsampling \citep{politis1994large} with a subsample size of $\lfloor\sqrt{n}\rfloor$, and the Adaptive HulC method \citep{kuchibhotla2024hulc}. As shown in Figure \ref{simulation_coverage}, the bootstrap fails to generate valid confidence intervals in finite samples, with coverage rates significantly smaller than the target level for models with small ECE. While subsampling and the Adaptive HulC generate valid intervals with coverage rates at or above the target level, 
their interval lengths are larger than those produced by our method.
In Figure \ref{simulation_interval_length}, 
% we report the length of confidence intervals and the ratio of the average interval lengths of other methods relative to our method. 
we report the ratio of the average interval lengths of other methods relative to our method, the length of confidence intervals can be found in Section 3 of the supplementary material. 
All other methods generate much wider confidence intervals than ours, especially for models with small ECE.

To further demonstrate the effect of the length of confidence intervals, we consider a hypothesis testing problem with the null hypothesis $H_0: \mathrm{ECE}_{1:k}^2 = 0$. For all methods, we reject the null hypothesis if the confidence interval does not contain zero,
and report the power of each method over 1000 datasets. 
We also compare our method with the T-Cal method \citep{lee2023t} in Settings 1 and 2. T-Cal is designed for testing full calibration, so we consider full calibration in these settings. The threshold for the T-Cal method is obtained through 1000 Monte Carlo simulations, as suggested in Section 4.1 of \cite{lee2023t}. 
As shown in Figure \ref{simulation_type2_error}, our method exhibits greater power than the subsampling and adaptive HulC methods. 
The T-Cal method and ours have very similar power, but our method's threshold can be computed analytically through $\sigma_0$ defined in \eqref{var_calibrated}, making the method computationally more efficient.

\begin{figure}
    \centering
    \includegraphics[scale=0.25]{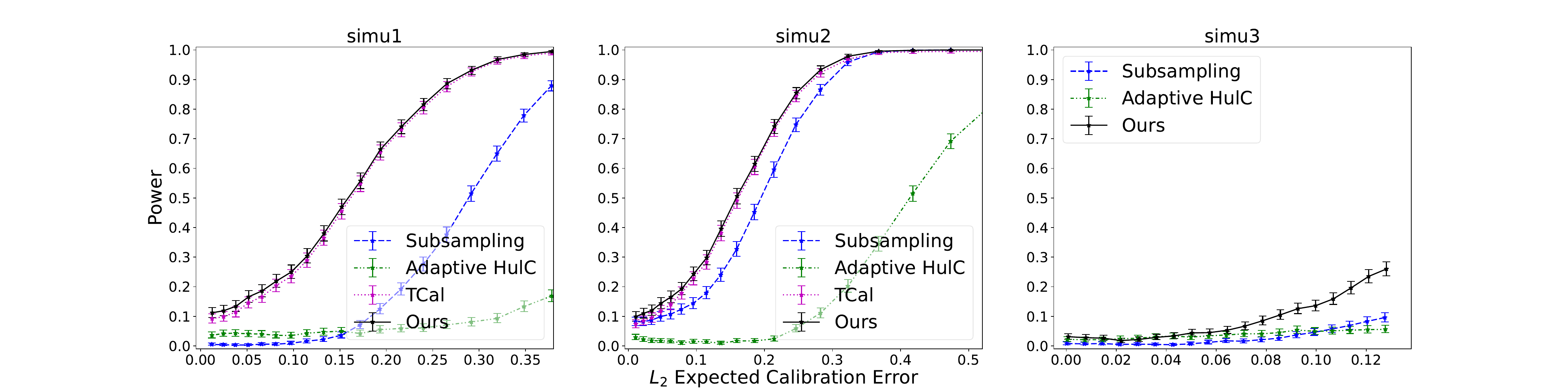}
    \caption{Power vs. Calibration Error.  For each setting and each value of $\beta$, we compute the power (percentage of null hypothesis rejections) over 1000 datasets. The error bars for power are generated using the Clopper–Pearson method.}
    \label{simulation_type2_error}
\end{figure}

\subsection{Results on Empirical Datasets}
We apply our method to generate confidence intervals for the ECE of machine learning models on several datasets.
\subsubsection{CIFAR}

\begin{table}
\centering
    \begin{tabular}{lcccc}
    \toprule
\multirow{2}{*}{} & \multicolumn{2}{c}{DenseNet40} &   \multicolumn{2}{c}{ResNet110} \\
                  & $\sqrt{T_{m,n}}$        & CI & $\sqrt{T_{m,n}}$        & CI   \\
\midrule
No Calibration      &    20.35\% & [19.55\%, 21.13\%] & 20.28\% & [19.46\%, 21.07\%]   \\                  
Temperature Scaling       &  5.20\%  & [4.26\%, 5.99\%] &  5.84\%  & [4.92\%, 6.64\%]    \\
Matrix Scaling & 38.33\% & [37.52\%, 39.12\%] & 38.00\% & [37.19\%, 38.81\%]  \\
Focal Loss & 7.66\%  & [6.80\%, 8.43\%] & 6.18\%  & [5.27\%, 6.97\%]  \\
MMCE & 5.58\%  & [4.63\%, 6.38\%] & 5.53\%  & [4.60\%, 6.32\%]  \\
\bottomrule          
\end{tabular}
\caption{The values of $\sqrt{T_{m,n}}$ and our confidence intervals for top-1 $\mathrm{ECE}$ of models trained on CIFAR100.}
\label{tab:cifar100}
\end{table}

% We evaluate the top-1 $\ell_2$ ECE for models trained on the CIFAR10 and CIFAR100 datasets. 
We evaluate the top-1 $\ell_2$ ECE for models trained on the CIFAR100 dataset, additional results on the CIFAR10 dataset can be found in Section 3 of the supplementary material. 
We consider two models: DenseNet40 \citep{huang2017densely} and ResNet110 \citep{he2016deep}, 
% and WideResNet-28-10 \citep{zagoruyko2016wide}, 
fit using the optimization methods and hyperparameter settings from the respective papers. Additionally, we apply four calibration methods: temperature scaling, matrix scaling \citep{guo2017calibration}, focal loss \citep{mukhoti2020calibrating}, and MMCE loss \citep{kumar2018trainable}. 
We randomly select 10\% of the training data set as the calibration data for posthoc calibration methods. For each method, we compute the estimator $T_{m,n}$ with partition width $1 / (mK) = 1 / 50$ on the test set.
The value of $\sqrt{T_{m,n}}$ and the 90\% confidence interval for $\mathrm{ECE}_{1:1}$ are reported in 
% Table \ref{tab:cifar10} and 
Table \ref{tab:cifar100}. 

Similar to the T-Cal method \citep{lee2023t}, our method can test if a model is significantly mis-calibrated by checking if the confidence interval covers zero.
Additionally, for models obtained using various calibration methods, our method provides a quantitative evaluation of whether the calibration of two models is significantly different.
For instance, Figure \ref{cifar100_diagram} shows the reliability diagram of two DenseNet40 models obtained using temperature scaling and focal loss. From the plot, it is evident that the difference between the confidences and accuracies of the focal loss model is larger than that of the temperature scaling model. Quantitatively, the $T_{m,n}$ estimator of focal loss is larger. Our method further indicates that this difference is statistically significant.

\begin{figure}
    \centering
    \includegraphics[scale=0.3]{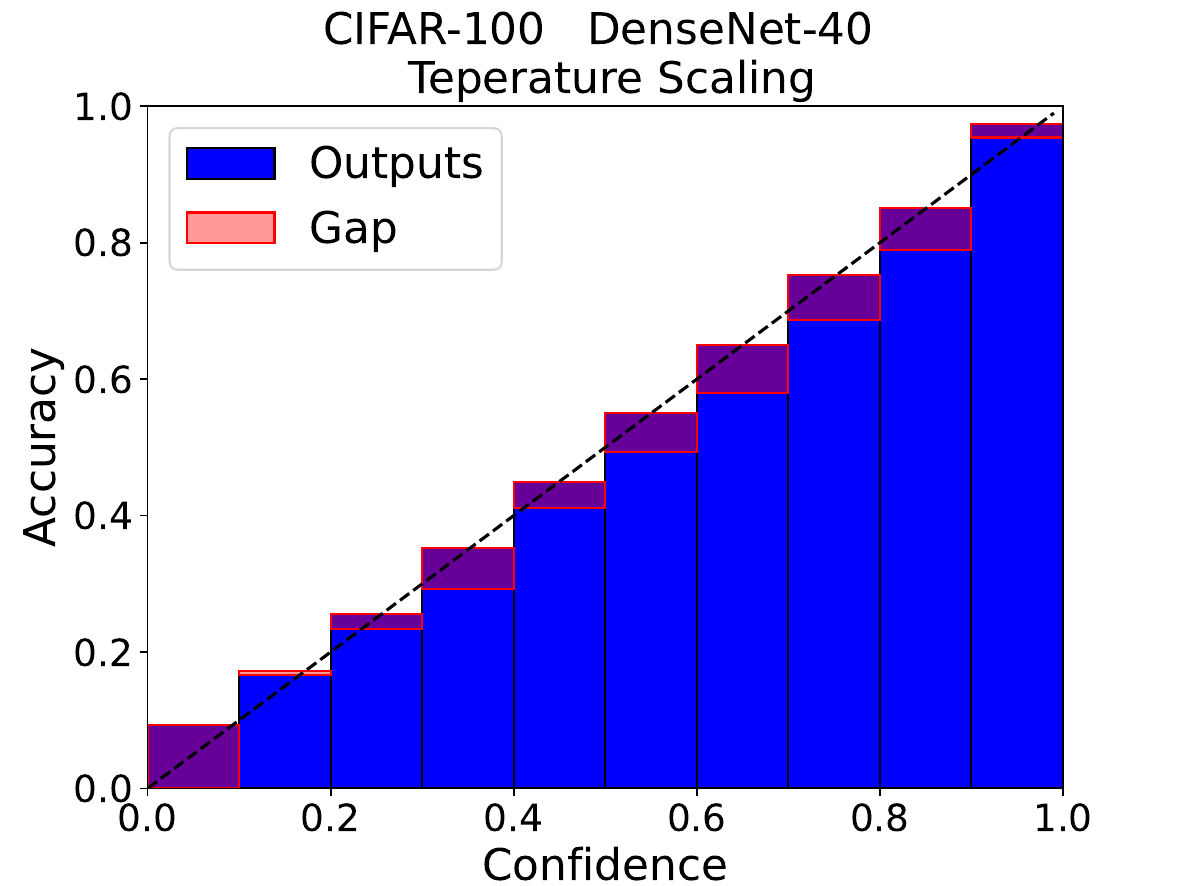}
    \includegraphics[scale=0.3]{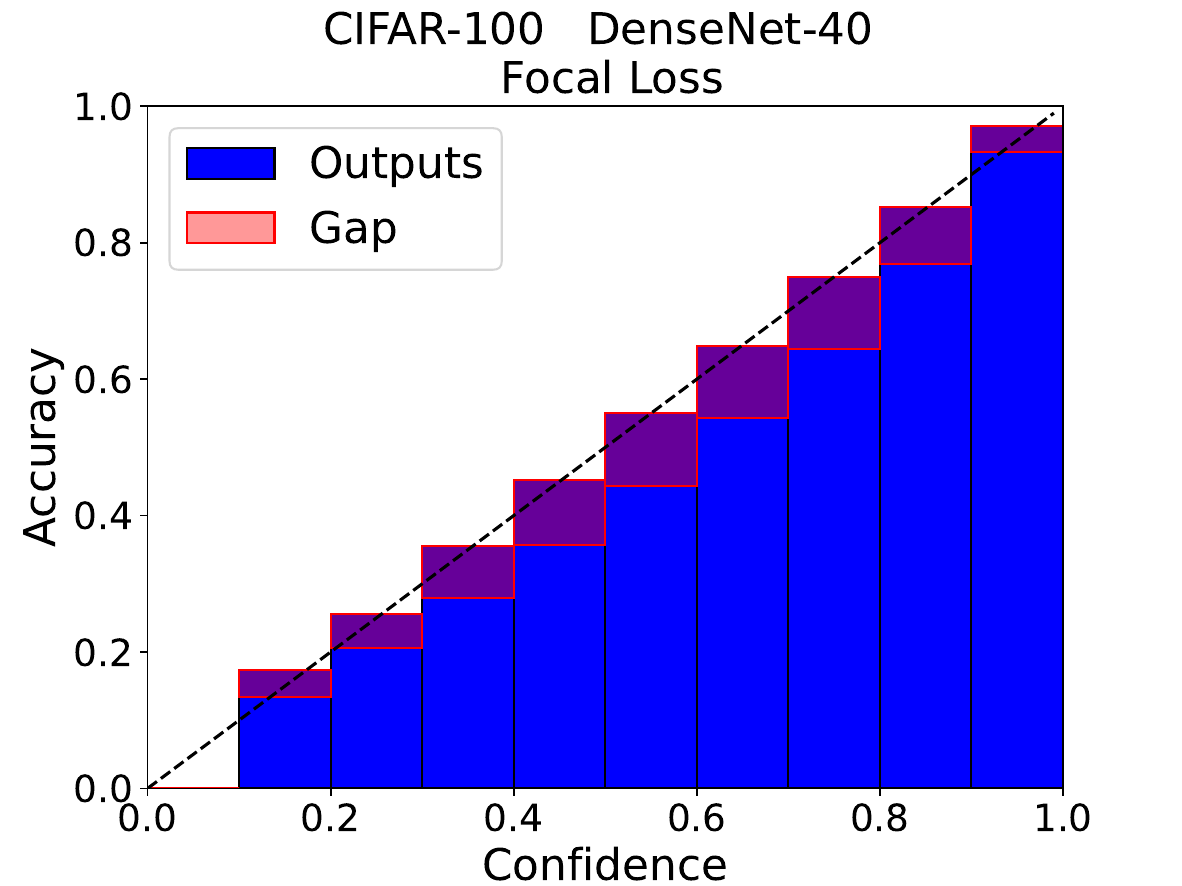}
    \caption{Reliability diagrams of a Densenet-40 model trained on CIFAR-100, calibrated by temperature scaling and focal loss. The dashed lines correspond to perfect calibration. One can observe that the model calibrated by focal loss is more poorly calibrated, and this is confirmed by our method.}
    \label{cifar100_diagram}
\end{figure}

\subsubsection{Alzheimer’s Disease Prediction}

We also apply our method to evaluate the calibration of an Alzheimer’s Disease Prediction Model developed in \cite{wang2023bias}. 
This model uses an ensemble of various machine learning models including neural nets, random forests, LightGBM, among others. The model uses multi-modal data, including MRI features and demographic information, to predict Alzheimer's disease, via binary classification between  Alzheimer's disease and controls. 
It is shown in \cite{wang2023bias} that the ensemble model is capable of delivering statistically fair performance across diverse sub-populations in terms of AUC. We further examine the top-1 calibration of the model, calculating the estimator $T_{m,n}$ with bin width $1 / (mK) = 20$ on the test set and generating confidence intervals for $\mathrm{ECE}_{1:1}$. 
We also generate confidence intervals via Bootstrap. The estimator takes value $\sqrt{T_{m,n}} = 4.52\%$, and the confidence intervals for the two methods are reported in Table \ref{AD_data}. 
The bootstrap suggests that 
the model is mis-calibrated, while our proposed method suggests the opposite. 
As shown in the experiments on simulated data in Section \ref{simulated_data}, the bootstrap fails to generate valid confidence intervals for ECE. This example highlights how a reliance on bootstrap methods might lead to questionable conclusions in model calibration evaluation.

\begin{table}
    \centering
    \begin{tabular}{c|c}
    \hline
        Methods & CI of ECE  \\
        \hline
        Bootstrap & [2.97\%, 11.30\%]  \\
        Ours & [0, 8.27\%] \\
    \hline
    \end{tabular}
    \caption{\small Confidence intervals for top-1 $\ell_2$ ECE for the Alzheimer’s Disease prediction model.}
    \label{AD_data}
\end{table}

\section{Discussion}

This paper addresses the problem of constructing confidence intervals for the top-1-to-$k$ $\ell_2$ Expected Calibration Error $\mathrm{ECE}_{1:k}$. We analyzed the asymptotic distribution of a debiased estimator of $\mathrm{ECE}_{1:k}^2$ and proposed an algorithm to construct confidence intervals that consider the different asymptotic behaviors of the estimator and the non-negativity of $\mathrm{ECE}_{1
}^2$. 
On the theoretical side, the asymptotic distribution of the debiased estimator is derived using a Poissonization argument. In our proof, we carefully specified the conditions needed to obtain the limit law of the conditional distribution and extended existing results to enable Poissonization with correlated components. This technical tool may be useful beyond our paper.
Numerical experiments demonstrate that our proposed method provides valid confidence intervals and is more powerful compared to other resampling-based methods.

Interesting future directions include extending the theoretical result to top-1-to-$k$ $l_p$-ECE with general $p \ge 1$ and investigating unequal volume binning schemes or kernel estimators for the calibration error.

\section*{Acknowledgments}

This work was supported in part by
the ASSET Center for 
Safe, Explainable and Trustworthy
AI-Enabled Systems at
Penn,  the 
ARO, NSF, ONR, and the Sloan Foundation. 
We thank 
Larry Wasserman for inspiring us to work on this problem, and
Arun Kumar Kuchibotla
for helpful discussion and feedback on earlier versions of the manuscript.
We also thank Sebastian Gruber for catching a typo in the earlier versions.

{\small
\setlength{\bibsep}{0.2pt}
\bibliographystyle{unsrt}
% \bibliographystyle{apalike}
% \bibliography{Bibliography-MM-MC}
\bibliography{ref.bib}
}

\appendix
\newpage
\section{Further Discussion}
\label{fudi}

\subsection{Connection to Nonparametric Quadratic Functional Estimation}
\label{np-q-est}
{
Here, we provide additional explanation on the connection between our estimator $T_{m,n}$ and nonparametric functional estimation to better motivate its use for estimating $\mathrm{ECE}_{1:k}^2$.
    Broadly speaking, our method estimates the ECE in several steps: (1) first, it partitions the space; (2) second, in each component of the space, it estimates the squared functional in each interval by a U-statistic. 
    While this is a simple idea, explaining it precisely requires a certain level of detail and notational overhead.
    Below, we make an analogy to more classically well-known estimation problem,  hoping that this will ease the burden of understanding or method.
}

{
    Let $f$ be a probability density function of a random variable $X$. The estimation of the quadratic integral functional $\int f^2(x) dx$ has been well studied in the broader statistical literature \citep{kerkyacharian1996estimating, laurent1996efficient}. 
    Consider 
    the Haar basis functions
    where $\phi_{ji}(x) = 2^{j/2}\phi(2^jx - i), \phi(x) = 1_{[0,1]}(x)$. 
    These are simply shifted and scaled versions of the ``mother wavelet"
    $\phi(x) = 1_{[0,1]}(x)$, 
    which is just the indicator function of the unit interval, equal to unity on $[0,1]$, and zero otherwise.
}

    % The quadratic wavelet estimator discussed in section 3 of the classic Annals of Statistics paper
    % \cite{kerkyacharian1996estimating} 
    % first decomposes $\int f^2(x) dx = \sum_j \int [\phi_j(x)f(x)]^2 dx$, which holds due to the 
    % orthonormality of the wavelet basis.
    % Then, it estimates $\int [\phi_j(x)f(x)]^2 dx$ by
    % \begin{equation*}
    %     \hat{\theta}_j = \frac{2}{n(n-1)} \sum_{1\leq j_1 < j_2 \leq n}\sum_i \phi_{ji}(X_{j_1}) \phi_{ji}(X_{j_2}),
    % \end{equation*}
    % This estimator can be viewed as first partitioning  
    % $\mathbb{R}$ into the intervals
    % $B_{i} = [i / 2^{j}, (i+1)/2^{j}]$, 
    % and then using a U-statistic within each intervals 
    % to estimate 
    % $\sum_j \int [\phi_j(x)f(x)]^2 dx$ in an unbiased way.
    % Note that $\rm{Vol}(B_i) = 1 / 2^{j}$ is the length of the interval, and so $\phi_{jk}(x) =  \rm{Vol}(B_k)^{-1/2} 1_{B_k}$. 

{
    The quadratic wavelet estimator of $\int f^2(x)dx$ discussed in section 3 of the classic Annals of Statistics paper
    \cite{kerkyacharian1996estimating} is given by 
    \begin{equation*}
        \hat{\theta}_j = \frac{2}{n(n-1)} \sum_{1\leq j_1 < j_2 \leq n}\sum_i \phi_{ji}(X_{j_1}) \phi_{ji}(X_{j_2}),
    \end{equation*}
    which is an unbiased estimator of $\sum_i\left(\int \phi_{ji}(x)f(x)dx\right)^2$.
    At a fixed scale $j$, this estimator can be viewed as first partitioning  
    $\mathbb{R}$ into the intervals
    $B_{i} = [i / 2^{j}, (i+1)/2^{j}]$, 
    and then using a U-statistic within each interval:
    \begin{equation*}
        \hat{\theta}_{ji} = \frac{2}{n(n-1)} \sum_{1\leq j_1 < j_2 \leq n}\phi_{ji}(X_{j_1}) \phi_{ji}(X_{j_2}),
    \end{equation*}
    to estimate $\int_{B_{i}} f^2(x)dx$.
}

{
    If we let $\rm{Vol}(B_i) = 1 / 2^{j}$ be the length of the interval, then $\phi_{ji}(x) =  \rm{Vol}(B_i)^{-1/2} 1_{B_i}$. $\hat{\theta}_{ji}$ is an unbiased estimator of $\left(\int \phi_{ji}(x)f(x)dx\right)^2 = \left(\int_{B_i}f(x)dx\right)^2 / \rm{Vol}(B_i)$, which approximates $\int_{B_{i}} f^2(x)dx$ as $j$ increases and the length of $B_i$ gets smaller.
}

{
    Our estimator $T_{m,n}$ can be written in a similar form. Recall that our goal is to estimate the quadratic integral functional 
    \[
    \mathrm{ECE}_{1:k}^2 = \int \bigl\|\EE{Y_{r_{1:k}}-Z_{(1:k)}|Z_{(1:k)}}\bigr\|^2 dZ_{(1:k)}
    \]
    We partition $\Delta(K,k)$, the support of top-$k$ predicted probabilities, into $\mathcal{B}_m = \{B_1, \dots, B_{\ell_{m}}\}$, where $\ell_{m}>0$ denotes the number of bins.
    We are given a calibration data set $\{(X^{(1)}, Y^{(1)})$, $\dots$, $(X^{(n)}, Y^{(n)})\}$. For all $i\in[n]$, we define the predicted probabilities $Z^{(i)} = f(X^{(i)}) \in \Delta_{K-1}$, and let $U^{(i)} = Y^{(i)}_{r_{1:k}} - Z^{(i)}_{(1:k)}$ be the difference between the top probability predictions and associated labels. For each bin $B_i$, we let $\mathcal{I}_{m,n,i} := \{j: Z^{(j)}_{(1:k)} \in B_i, 1\le  j \le  n\}$ be the indices of data points in $B_i$.
}

{
    We can then define the function 
    \[
    \phi_{i}(Z) = \biggl( \frac{|\mathcal{I}_{m, n, i}| - 1}{n-1}\biggr)^{-1/2} 1_{B_i}(Z),
    \]
    which takes a similar form as the Haar basis functions, 
    but is defined over the cut simplex $\Delta(K,k)$ rather than $\mathbb{R}$, and $\rm{Vol}(B_i)$ is replaced by $(|\mathcal{I}_{m, n, i}| - 1)  / (n-1)$, an empirical estimate of $P(B_i)$. 
    Using this definition, our estimator can be rewritten as:
    \begin{equation*}
    \begin{split}
    T_{m,n} & = \frac1n\sum_{\substack{1 \le  i \le  \ell_{n},\, |\mathcal{I}_{m,n, i}| \ge  2}} \frac{1}{|\mathcal{I}_{m, n, i}| - 1} \sum_{a\neq b \in \mathcal{I}_{m, n, i}} U^{(a)\top}U^{(b)}, \\
    & = \frac{2}{n(n-1)} \sum_{1\leq j_1 < j_2 \leq n}\sum_{1\leq i\leq l_{m,n} } U^{(j_1)\top}U^{(j_2)} \phi_{i}\Bigl(Z^{(j_1)}_{(1:k)}\Bigr) \phi_{i}\Bigl(Z^{(j_2)}_{(1:k)}\Bigr).
    \end{split}
    \end{equation*}
    This form highlights the structural similarity between our estimator and classical estimators of quadratic functionals.
}

\subsection{Further Discussion of Related Work}
\label{compare_tcal}

{
\paragraph{Comparison between discrimination and calibration} For a classifier, strong discrimination does not necessarily require or imply good calibration. As discussed in \cite{guo2017calibration}, many modern machine learning models have achieved better discrimination performance but are often worse in terms of calibration. 
}

{
While discrimination focuses on correctly ranking or classifying input data, calibration concerns the reliability of the predicted probabilities. Many classification models are trained using cross-entropy loss, where the softmax or logistic transformation of model outputs is interpreted as class probabilities. These probabilities are often used in downstream decision-making. Calibration ensures that such probability estimates are meaningful and trustworthy.
}

{
    Poor calibration can lead to over- or under-confidence in predictions. In high-stakes domains such as healthcare, 
    poor calibration could mislead patients and healthcare providers, resulting in suboptimal treatment strategies \cite{van2019calibration}. For instance, consider a decision rule such as “if the predicted probability of relapse in the next month exceeds 0.7, schedule a follow-up”, it is problematic if an ML model predicts 0.7, the true average probability could be 0.5 or 0.9. This discrepancy could have costly consequences (schedule follow-up when not needed, etc).
}

{
\paragraph{Connection between calibration and AI:} calibration can also be applied to generative AI models such as large language models (LLMs), see e.g., \cite{jiang2021can,kadavath2022language,zhang2024study}.
The simplest, yet important, setting is when the language models are used to answer multiple-choice questions. 
This occurs widely in many popular benchmarks which often involve multiple choice test problems from various domains.
Such problems can be viewed as classification and so techniques for calibration of classifiers can be directly used.
A second setting is to calibrate the claims for correctness of an LLM, for instance by asking the LLM ``what is the probability that your answer is correct?" or a similar prompt.
This becomes a binary classification setting and so standard calibration can again be used.
For both of the above settings, our method could be used directly to provide confidence intervals for the calibration error.
}

{
Additionally, there are also a number of 
other ways to apply calibration to LLMs, see e.g.,  \cite{huang2024uncertainty,band2024linguistic}, etc.
Similar questions about uncertainty quantification of the calibration error also arise in these settings.
Our techniques could possibly be extended to cover these scenarios.
}

{
\paragraph {Comparison to T-Cal method \citep{lee2023t}} \cite{lee2023t} used a similar debiased estimator and developed an optimal testing procedure for calibration. Here, we emphasize the key differences between our method and the approach in reference \cite{lee2023t} and highlight our contributions:
    \begin{itemize}
       \item[1] {\bf Asymptotic Distributions of Estimator $T_{m,n}$:} We characterize 
        the asymptotic distribution of the debiased estimator $T_{m,n}$, for all non-negative $\mathrm{ECE}_{1:k}^2$. 
        In contrast, the \cite{lee2023t} does not have any results about the asymptotic distribution, only some bounds on the mean and variance, which are not useful for our work. 
        For instance, in Section A.1 of \cite{lee2023t}, the authors discussed that their analysis mainly focuses on finding a lower bound on mean difference under null and alternative hypotheses $\E_{P_1}[T_{m,n}] - \E_{P_0}[T_{m,n}]$ and upper bound on the variance $\text{Var}_{P_0}[T_{m,n}]$ and $\text{Var}_{P_1}[T_{m,n}]$.
        In addition to the new Poissonization argument highlighted in the paper (Lemma \ref{de_poissonization}), our proof requires a more refined analysis. 
        \item[2] {\bf Explicit Critical Value:} While \cite{lee2023t} introduces an optimal testing procedure, its analysis only provides the asymptotic order of the critical value. In practice, the critical values are chosen via resampling methods (as noted on page 16 of \cite{lee2023t}), and in their simulations, Monte Carlo sampling is used to estimate them. In contrast, our analysis derives an exact form for the critical value $z_{\alpha} \sigma_0$, which can be computed explicitly or using standard numerical integration methods.
              \item[3] Some additional extensions include:
        \begin{itemize}
            \item[(1)] We extend the analysis from full calibration to top-1-to-$k$ calibration, which is more practical in real-world applications.
            \item[(2)] We refine the debiased estimator by replacing the factor $1/(n|\mathcal{I}_{m,  i}|)$ with $1/[n\bigl(|\mathcal{I}_{m,  i}| - 1\bigr)]$ for each bin. As shown in Remark \ref{small_bias}, our estimator has a smaller bias.
        \end{itemize}
    \end{itemize}
}

\paragraph{Broader Work on Uncertainty Quantification for Machine Learning.}
Beyond calibration, an important line of work aims to quantify the uncertainty in the predictions of machine learning models by constructing prediction sets, as opposed to point prediction sets.
The idea of prediction sets dates back at least to the pioneering works of \cite{Wilks1941}, \cite{Wald1943}, \cite{scheffe1945non}, and \cite{tukey1947non,tukey1948nonparametric}.
More recently conformal prediction has emerged as a prominent methodology for constructing prediction sets \citep[see, e.g.,][]{saunders1999transduction,vovk1999machine,vovk2005algorithmic,Vovk2013, lei2013distribution,angelopoulos2021gentle,guan2023localized, romano2020classification,liang2023conformal,dobriban2023symmpi}. 
Predictive inference methods
\citep[e.g.,][etc]{geisser2017predictive}  
have been developed under various assumptions
\citep[see, e.g.,][]{bates2021distribution,park2021pac,park2022pac,sesia2022conformal,qiu2023prediction,li2022pac,kaur2022idecode,si2023pac,lee2024simultaneous}.

{
\paragraph{Comparison to Conformal Methods}
 % We would like to emphasize the differences between our method and the conformal approach.
Both conformal prediction and calibration are methods to improve uncertainty quantification in machine learning.
In classification problems, standard split conformal methods construct prediction sets for labels, ensuring that the true label is covered with a specified probability. These methods provide marginal coverage guarantees (the probability is taken over both the input and label distributions, rather than being conditioned on the input). Additionally, conformal methods may alter the decision-making workflow by producing a set of predictions rather than a single probabilistic output.
}

{
In contrast, calibration requires that the predicted probabilities accurately reflect the likelihood of events, ensuring that the model gives reliable probabilistic forecasts. For instance, \cite{zhao2021calibrating} demonstrated that the Bayes decision rule based on well-calibrated probabilities will lead to optimal expected risk. Our method provides a rigorous framework for evaluating model calibration, offering asymptotically valid confidence intervals for the Expected Calibration Error. 
}

{
\subsection{Discussion of the HulC Method}
\label{hudi}
In our simulated data experiments, we compared our method with the HulC method \citep{kuchibhotla2024hulc}. Our confidence interval is asymptotically exact, the HulC method can provide finite-sample valid confidence intervals in some cases. Here we would like to clarify that the finite-sample coverage guarantee of the HulC method does not trivially hold in our context.
}

{
For the HulC method,
a lower bound on the finite sample coverage bound only holds if there is a known upper bound on the median bias of the estimator, and an exact coverage bound (say between $1-\alpha$ and $1-\alpha+\epsilon$ for some $\epsilon>0$) only holds if we have both upper and lower bounds on the bias. Unfortunately, it is not clear how to prove such bounds.
This median bias $\Delta$ of the estimator $T_{m,n}$ is not zero in finite samples. Our result only implies that the asymptotic median bias is zero.
According to \cite{kuchibhotla2024hulc}, if the number of splits $B_{\alpha,\Delta}$ is chosen such that $(1/2 - \Delta)^{B_{\alpha,\Delta}} + (1/2 + \Delta)^{B_{\alpha,\Delta}} \leq \alpha$ with $\Delta = 0$, then the HulC interval is asymptotically valid. 
Finite-sample validity only holds for $n > N_{\alpha}$, where $N_{\alpha}$ is a quantity depending on the convergence rate of the median bias.
 Unfortunately, we do not know what $N_{\alpha}$ is in our setting.
 }

{
Furthermore, we use Algorithm 1 from the HulC paper \cite{kuchibhotla2024hulc}, where the number of splits is randomized. As stated in \cite{kuchibhotla2024hulc}, this, does not imply coverage validity,
    because with non-zero probability, we can have a smaller number of splits than required.
    In our experiments, we use the adaptive HulC, where the median bias is estimated via subsampling and the number of splits is also randomized as in Algorithm 1 in \cite{kuchibhotla2024hulc}. Therefore, the finite-sample coverage guarantee does not trivially hold in our context.
}

\section{Proofs}
In this section, we present the proofs for our theoretical results in Section 3. The proofs are organized as follows: 
Section \ref{moments_calculation} details the preliminary calculations of the moments of the Poissonized variable \eqref{MN}, which are essential for establishing the asymptotic normality of \eqref{MN} using the Lyapunov Central Limit Theorem. 
The proofs of the asymptotic normality for the estimator $T_{m,n}$ (Theorems 3.3 and 3.4) are provided in Sections \ref{proof_clt_calibrated_model} and \ref{proof_clt_miscalibrated_model}. These proofs use a {Poissonization argument}. Our analysis extends the results from \cite{beirlant1998asymptotic} by introducing a result that facilitates Poissonization with correlated components (see Lemma \ref{de_poissonization}). 
The convergence of the variance estimator (Proposition 3.5) is established in Section \ref{proof_sigma1_hat_convergence}, which involves calculating the estimator's mean and variance, with the conclusion drawn via Chebyshev’s inequality. 
Section \ref{bias_var} analyzes the mean and variance of the estimator $T_{m,n}$, guiding the selection of the number of bins. 
Finally, the asymptotic correct coverage rates (Theorems 3.7 and 3.9) are established in Sections \ref{proof_KI_mean} and \ref{proof_KI_ece}.

{\bf Notations.}
In the proofs, we will rely on the following additional notations.
We denote the imaginary unit by $\i $.
For two strictly positive sequences $(a_n)_{n\ge1}$ and $(b_n)_{n\ge1}$,
we write $a_n \sim b_n$ if $a_n/b_n\to 1$ as $n\to\infty$.

\subsection{Moment Calculations}
\label{moments_calculation}

We start with some preliminary moment calculations that will be used later in the arguments.
With notation defined in 
% \Cref{form},
Section 2 of the main text,
let $(\underline{Z},\underline{Y})=(Z_{(1:k)},Y_{r_{1:k}})$.
Let $B$ be a subset of $\Delta(K,k)$, 
such that $\mu(B) = P(Z_{(1:k)} \in B)>0$, 
and
$N\sim \text{Poisson}(n \mu(B))$. We allow $B=B_{(n)}$ to depend on $n$, such that $\lim_{n\rightarrow \infty} \mu(B_{(n)}) = 0$ and $\lim_{n\rightarrow \infty} n \mu(B_{(n)}) = \infty$, but  for simplicity, we do not display the subscript $n$ unless it is required.

Let $(\underline{Z}(B), \underline{Y}(B))$ be a random variable such that $P\bigl((\underline{Z}(B), \underline{Y}(B)) \in S\bigr) = P\bigl((Z_{(1:k)}, Y_{r_{1:k}}) \in S \mid Z_{(1:k)} \in B\bigl)$ for any Borel set $S$.
Further, let $\{(\underline{Z}^{(1)}(B),\underline{Y}^{(1)}(B)),\dots, ( \underline{Z}^{(N)}(B),\underline{Y}^{(N)}(B))\}$ 
be an i.i.d.~sample from the distribution of $(\underline{Z}(B), \underline{Y}(B))$ and 
define $V^{(i)} = \underline{Y}^{(i)}(B) - \underline{Z}^{(i)}(B)$ and
$V = \underline{Y}(B) - \underline{Z}(B)$. We first study the moments of 
\begin{equation}\label{mn}
\begin{split}
M_n & := \frac{N^2}{n(N-1)I(N\ge  2)}\Biggl[\biggl\|\frac{1}{N}\sum_{i=1}^{N}V^{(i)} \biggr\|^2 - \frac{1}{N}\sum_{i=1}^{N}\left\|V^{(i)}\right\|^2 \Biggr] \\
& = \frac{2}{n(N-1)I(N\ge  2)}\sum_{1\le  j_1 < j_2 \le  N} {V^{(j_1) \top}}V^{(j_2)}.
\end{split}
\end{equation}
Above, we use the convention that
$\sum_{i=1}^0 V^{(i)}=0$, 
$\sum_{i=1}^0 \bigl\|V^{(i)}\bigl\|^2=0$, $\sum_{1\le  j_1 < j_2 \le  1} {V^{(j_1) \top}}V^{(j_2)} = 0$
and
$0/0=0$, so that $M_n=0$ when $N=0$ or $N = 1$.

\subsubsection{Preliminary Calculations}
Denote the coordinates of 
$V$ as $V  = (V_1, \ldots, V_k)^\top$
and those of 
$V^{(i)}$, $i\in[N]$ as
$V^{(i)}  = (V^{(i)}_1, \ldots, V^{(i)}_k)^\top$.
Conditioned on $N$, $V^{(i)}$, $i\in[N]$ is an i.i.d.~sample, so we have
\begin{equation*}
\begin{split}
\EE{{V^{(1) \top}}V^{(2)}} & = \EE{V^{(1) \top}}\EE{V^{(2)}} = \bigl\|\EE{V}\bigr\|^2, \\ 
\EE{\bigl[{V^{(1) \top}}V^{(2)}\bigr]^2} & = 
\sum_{i=1}^{k}\sum_{j=1}^{k}\EE{V^{(1)}_i V^{(2)}_i V^{(1)}_j V^{(2)}_j} = \sum_{i=1}^{k}\sum_{j=1}^{k}\left(\EE{V_i V_j}\right)^2.
\end{split}
\end{equation*}
Hence, $\Var{{V^{(1) \top}}V^{(2)}} = \sum_{i=1}^{k}\sum_{j=1}^{k}\left(\EE{V_i V_j}\right)^2 - \bigl\|\EE{V}\bigr\|^4$.
Moreover, 
\begin{equation*}
\begin{split}
& \Cov{{V^{(1) \top}}V^{(2)}, {V^{(1) \top}}V^{(3)}} \\
 = & \EE{{V^{(2) \top}}}\EE{V^{(1)}{V^{(1) \top}}}\EE{V^{(3)}} - \EE{{V^{(2) \top}}}\EE{V^{(1)}}\EE{{V^{(1) \top}}}\EE{V^{(3)}} \\
= & \EE{V^{\top}}\Cov{V}\EE{V}.
\end{split}
\end{equation*}

Recalling that $I(\cdot)$ is the indicator of an event,
with $M_n$ from \eqref{mn},
we have conditional on $N$ that 
\begin{equation}
\label{conditional_mean}
\begin{split}
\EE{M_n | N}  &= \frac{N I(N\ge  2)}{n}  \EE{{V^{(1) \top}}V^{(2)}} = \frac{N  I(N\ge  2)}{n} \bigl\|\EE{V}\bigr\|^2. 
\end{split}
\end{equation} 
Using the results above, 
\begin{equation}
\label{conditional_var}
\begin{split}
&\Var{M_n | N} =  \frac{2 N(N-1) I(N\ge  2) \cdot (\Var{{V^{(1) \top}}V^{(2)}} 
 + 2(N-2) \Cov{{V^{(1) \top}}V^{(2)}, {V^{(1) \top}}V^{(3)}})}
 {n^2(N-1)^2} \\
 &= \frac{2 N I(N\ge  2)}{n^2(N-1)} \left[\sum_{i=1}^{k}\sum_{j=1}^{k}\left(\EE{V_i V_j}\right)^2 - \bigl\|\EE{V}\bigr\|^4 \right]
 +\frac{4N(N-2)I(N\ge  2)}{n^2(N-1)}  \EE{V^{\top}}\Cov{V}\EE{V}.
\end{split}
\end{equation} 
Next, since  $N\sim \text{Poisson}(n \mu(B))$, we have
\begin{align}\label{emn}
\EE{M_n} &= \EE{\EE{M_n|N}} = 
\mu(B)\Bigl(1 - e^{-n\mu(B)}\Bigr) \bigl\|\EE{V}\bigr\|^2.
\end{align}
With $c_n = n\mu(B)$,
\begin{equation*}
 \EE{N^2 I(N\ge  2)}
 = \sum_{j= 2}^\infty \frac{\bigl(j(j-1) + j \bigr) e^{-c_n} c_n^j}{j!}
 =
 c_n^2\sum_{j= 2}^\infty \frac{e^{-c_n} c_n^{j-2}}{(j-2)!} + c_n\sum_{j= 2}^\infty \frac{e^{-c_n} c_n^{j-1}}{(j-1)!}
 = c_n^2 + c_n \bigl(1- e^{-c_n}\bigr).
\end{equation*}
Thus, we have
\begin{equation}
\label{partition_covariance}
\begin{split}
 &\Cov{N,M_n} =  \EE{\EE{(N - \EE{N})(M_n - \EE{M_n})| N} }   \\
 &=  \frac{ \bigl\|\EE{V}\bigr\|^2}{n}\Cov{N, N I(N\ge  2)}
 =  \frac{ \bigl\|\EE{V}\bigr\|^2}{n} 
 \bigl(\EE{N^2I(N\ge  2)}
 -\EE{N} \EE{N I(N\ge  2)}\bigr)\\
 &= \frac{ \bigl\|\EE{V}\bigr\|^2}{n} 
 \biggl( c_n^2 + c_n\Bigl(1 - e^{-c_n}\Bigr)
 -c_n\Bigl(c_n - c_ne^{-c_n}\Bigr) \biggr)
 = \mu(B)\Bigl(1 + \bigl(c_n - 1\bigr)e^{-c_n}\Bigr)\bigl\|\EE{V}\bigr\|^2. 
\end{split}
\end{equation}
Moreover, by the law of total variance, we have 
\begin{equation}
\label{partition_variance}
\begin{split}
& \Var{M_n} 
= \EE{\Var{M_n | N}} + \Var{\EE{M_n | N}} \\ 
=&  \frac{2}{n^2} \biggl\{\EE{\frac{N I(N\ge  2)}{N-1}} \biggl[\sum_{i=1}^{k}\sum_{j=1}^{k}\left(\EE{V_i V_j}\right)^{2} - \bigl\|\EE{V}\bigr\|^4 \biggr] \biggr\}\\
 &+ \frac{1}{n^2}\left\{ \Var{N I(N\ge  2)} \bigl\|\EE{V}\bigr\|^{\textcolor{black}{4}}  +\EE{\frac{4 N (N-2)I(N\ge  2)}{N-1}} \EE{V^{\top}}\Cov{V}\EE{V}\right\}.
\end{split}
\end{equation}

Since $N \sim \text{Poisson}(c_n)$, by Lemma \ref{poisson_inverse_ineq},
\begin{equation}
\label{poisson_order_-1}
\left\vert \EE{\frac{I(N\ge  2)}{N - 1}} - \frac{1}{c_n}\right\vert \le  \frac{9}{c_n^2}. 
\end{equation}
Then 
consider $B = B_{(n)}$ depending on $n$,
and let $n\to\infty$.
Since $c_n = n\mu(B_{(n)}) \rightarrow \infty$, we have 
\begin{equation}
\label{poisson_order_0}
\EE{\frac{N I(N\ge  2)}{N - 1}} = O(1),  \quad \left\vert \EE{\frac{4N (N-2)I(N\ge  2)}{N - 1}} - 4c_n\right\vert = O(1),
\end{equation}
and
\begin{equation}
\label{poisson_order}
\begin{split}
 & \bigl\vert\Var{N I(N\ge  2)}- c_n\bigr\vert 
= \Bigl\vert\ c_n^2 + c_n(1-e^{-c_n}) - c_n^2(1-e^{-c_n})^2  - c_n\Bigr\vert\\
&=  e^{-c_n}\bigl(2c_n^2-c_n -c_n^2e^{-c_n}\bigr) = o(1).
\end{split}
\end{equation}

Letting $\bar{V}^{i,j} = {V^{(i) \top}}V^{(j)} - \EE{{V^{(i) \top}}V^{(j)}}$, by the independence of $V^{(i)}$ for different $i$, 
$\EE{\bar{V}^{i,j}\bar{V}^{s,t}} = 0$ if $\{s,t\}\cap \{i,j\}=\emptyset$. 
So if $\EE{\prod_{l=1}^6 \bar{V}^{i_l, j_l} } \neq 0$, there can be at most seven different indices in $(i_1,\dots i_6, j_1, \dots, j_6)$.
Also since $\|V^{(i)}\|_{\infty} \le  1$, we have $\bar{V}^{i,j} \le  2k$ for all $i,j$. 
Since $N / (N - 1) \le  2$ for $N \ge  2$, 
we have that
for some $\kappa_1>0$ depending only on $k$, for $N\ge 2$,
\begin{equation*}
\begin{split}
& \EE{|M_n - \EE{M_n}|^{3} | N}  \le  \sqrt{\EE{|M_n - \EE{M_n}|^{6} | N}}  \\
& \le  \sqrt{\kappa_1 \frac{N(N-1)(N-2)(N-3)(N-4)(N-5)(N-6)I(N \ge  6) + \kappa_1}{n^{6}N^{6}}},
\end{split}
\end{equation*}
whereas $\EE{|M_n - \EE{M_n}|^{3} |0} = 0$.
Therefore, for some $\kappa_2, \kappa_3>0$ depending only on $k$, by Jensen's inequality
\begin{equation}
\begin{split}
\label{third_moment}
\EE{|M_n - \EE{M_n}|^{3}} & = \EE{\EE{|M_n - \EE{M_n}|^{3} | N}} \\
& \le  \frac{\E\sqrt{(\kappa_1N + \kappa_1/N^6)I(N\ge 1)}}
{n^{3}}
\le  \kappa_2 \frac{\sqrt{c_n}}{n^3} + \kappa_3 \frac{1}{n^3}.
\end{split}
\end{equation}
Next, since $c_n = \EE{N} = n\mu(B)$, 
by
relations (5.21) and (5.22) on page 121 of \cite{kendall1946advanced},
the central moments $\EE{(N-c_n)^{j}}$, $j\ge 2$,  satisfy
\begin{equation*}
\EE{(N-c_n)^{j}} = (j-1)c_n\EE{(N-c_n)^{j-2}} + c_n \frac{d}{dc}\EE{(N-c)^{j-1}}\vert_{c_n},
\end{equation*}
where $\frac{d}{dc}\EE{(N-c)^{j-1}}$ is the derivative of $c\mapsto\EE{(N-c)^{j-1}}$ with respect to $c$.
In particular,
$\EE{(N-c_n)^{2}}= \EE{(N-c_n)^{3}} = c_n$, while
$$\EE{(N-c_n)^{4}} = 3c_n^2 + c_n, \,
\EE{(N-c_n)^{5}} = 10c_n^2 + c_n, \,
\textnormal{ and }\,
\EE{(N-c_n)^{6}} = 15c_n^3 + 25c_n^2 + c_n.$$
Thus, 
we have 
$\EE{|N - \EE{N}|^3} \le  \sqrt{\EE{|N - \EE{N}|^6}} = \sqrt{c_n + 25c_n^2 + 15 c_n^3}$. 
Therefore, 
if $c_n\ge \kappa_0$ for some $\kappa_0$, then 
for some $\kappa_4$ depending only on $k$ and $\kappa_0$,
\begin{equation}
\label{poisson_third_moments}
    \EE{\left\vert\frac{N - \EE{N}}{\sqrt{n}}\right\vert^3} \le  \kappa_4 \mu^{\frac{3}{2}}(B). 
\end{equation}

\subsubsection{Calibrated Model}
In this section we present some simplified moment calculations for the case of a calibrated model. 
If $f$ is top-1-to-$k$ calibrated,
i.e., $\EE{\underline{Y}(B) | \underline{Z}(B) = Z_{(1:k)}} = Z_{(1:k)}$ $P_Z$-almost surely, 
then 
$\EE{V} = 0$. 
By \eqref{emn}, we have 
$\EE{M_n} = 0$,
and by \eqref{partition_covariance}, 
we have
\begin{equation}
\label{partition_covariance_calibrated}
\begin{split}
 \Cov{M_n, N} = 0.
\end{split}
\end{equation}
Further, by \eqref{partition_variance}, we have
\begin{equation*}
\Var{M_n} 
= \frac{2}{n^2} \EE{\frac{N I(N\ge  2)}{N-1}} \biggl[\sum_{i=1}^{k}\sum_{j=1}^{k}\left(\EE{V_i V_j}\right)^2 \biggr]. 
\end{equation*}
Recalling that $V = \underline{Y}(B) - \underline{Z}(B)$ and $\underline{Y}(B) \in \{0, 1\}^k$ with $\sum_{i=1}^k \underline{Y}(B)_i \le 1$, 
we have that for $i \neq j$, $\EE{\underline{Y}(B)_i \underline{Y}(B)_j} = 0$ and   $\EE{\underline{Y}(B)_i^2} = \EE{\underline{Y}(B)_i} = \EE{\underline{Z}(B)_i}$. Therefore, 
\begin{equation*}
\begin{split}
& \left(\EE{V_i V_j}\right)^2  =  \EE{\underline{Z}(B)_i \underline{Z}(B)_j}^2, \, \text{if} \quad i\neq j, 
\,\text{ and  }\\
& \left\{\EE{V_i^2}\right\}^2 = \EE{\underline{Z}(B)_i}^2 - 2\EE{\underline{Z}(B)_i}\EE{\underline{Z}(B)_i^2} + \EE{\underline{Z}(B)_i^2}^2.
\end{split}
\end{equation*}
In conclusion, we have 
\begin{equation}
\label{partition_variance_calibrated}
\begin{split}
 \Var{M_n} 
= & \EE{\frac{2N I(N\ge  2)}{n^2 (N-1)}} 
\biggl[\sum_{i,j=1}^{k}\left\{\EE{\underline{Z}(B)_i \underline{Z}(B)_j}\right\}^2 \biggr] \\
& + \EE{\frac{2N I(N\ge  2)}{n^2 (N-1)}} 
\biggl[\sum_{i=1}^{k}\EE{\underline{Z}(B)_i}\bigl(\EE{\underline{Z}(B)_i}-2\EE{\underline{Z}(B)_i^2}\bigr)
\biggr].
\end{split}
\end{equation}

Recalling that  $\bar{V}^{i,j} = {V^{(i) \top}}V^{(j)} - \EE{{V^{(i) \top}}V^{(j)}}$ 
and $\EE{V} = 0$, 
we have
$\EE{\prod_{l=1}^3 \bar{V}^{i_l, j_l} } = 0$ whenever an index appears only once in $(i_1, i_2, i_3, j_1, j_2, j_3)$.
So if $\EE{\prod_{l=1}^3 \bar{V}^{i_l, j_l} } \neq 0$, there can be at most three distinct indices in 
$(i_1, i_2, i_3, j_1, j_2, j_3)$. 
Therefore, for some constant $\kappa_5$,
\begin{equation*}
\begin{split}
\EE{|M_n |^{3} | N} \le  \kappa_5 \frac{N(N-1)(N-2)}{n^{3}N^{3}}.
\end{split}
\end{equation*}
Thus
\begin{equation}
\label{third_moment_calibrated}
\EE{|M_n - \EE{M_n}|^{3}}  = \EE{|M_n |^{3}} = \EE{\EE{|M_n |^{3} | N}} \le  \frac{\kappa_5}{n^3}. 
\end{equation}
Finally, we present the following lemma, which was used in the above moment calculations.
\begin{lemma} 
\label{poisson_inverse_ineq}
If $N$ is a $\text{Poisson}(a)$ random variable for $a>0$, then 
\begin{equation*}
\left\vert \EE{\frac{I(N\ge  2)}{N - 1}} - \frac{1}{a}\right\vert \le  \frac{9}{a^2}. 
\end{equation*}
\begin{proof}
The proof is similar to the proof of Lemma 1 of \cite{beirlant1998asymptotic}.
By Jensen’s inequality, 
\begin{equation*}
\begin{split}
& \EE{\frac{I(N\ge  2)}{N - 1}}  \ge  \EE{\frac{1}{N} \Bigl| N \ge  2} P(N \ge  2) \\
& \ge  \frac{1}{\EE{N| N \ge 2}}P(N \ge  2) = \frac{P(N\ge  2)^2}{\EE{N}} = \frac{1}{a} - \frac{2(a+2)e^{-a} -(a+1)^2 e^{-2a} }{a} \ge  \frac{1}{a} - \frac{9}{a^2}.
\end{split}
\end{equation*}
For $N \ge  2$, $(N-1)(N+11) \ge  (N+1)(N+2)$, and so  $\frac{1}{N - 1} \le  \frac{1}{N + 1} + \frac{9}{(N+1)(N+2)}$. Therefore
\begin{equation*}
\begin{split}
\EE{\frac{I(N\ge  2)}{N - 1}} &\le  \EE{\frac{I(N\ge  2)}{N+1 1}} + \EE{\frac{9I(N\ge  2)}{(N+1)(N+2)}} \\
& =  \sum_{j=2}^{\infty} \frac{1}{j+1}\frac{e^{-a} a^{j}}{j!} + \sum_{j=2}^{\infty} \frac{9}{(j+1)(j+2)}\frac{e^{-a} a^{j}}{j!}\le  \frac{1}{a} + \frac{9}{a^2}.
\end{split}
\end{equation*}
This finishes the proof.
\end{proof}

\end{lemma}

\subsection{Proof of Theorem 
% \ref{clt_calibrated_model}
3.3
}
\label{proof_clt_calibrated_model}

The proof uses a Poissonization argument 
inspired by, and extending, the arguments of \cite{beirlant1994asymptotic, beirlant1998asymptotic}.
To clarify our technical contribution, we will state the key results of \cite{beirlant1998asymptotic} in the context of calibration and demonstrate how our analysis more precisely states the required conditions, while also extending the results.

\subsubsection{Poissonization}
Let $\tilde{N} \sim \text{Poisson}(n)$ and $\{(Z^{(1)}, Y^{(1)}),\dots, (Z^{(\tilde{N})}, Y^{(\tilde{N})})\}$ be an i.i.d.~sample following the distribution $P$.  
Recall that we use $\mathcal{B}_m = \{B_1, \dots, B_{\ell_{m,n}}\}$ to denote our partition.
For $i\in [\ell_{m,n}]$,
letting $\mathcal{I}(\tilde{N})_{i} = \left\{j: Z^{(j)}_{(1:k)} \in B_i, 1\le  j \le  \tilde{N} \right\}$ be the indices of data points falling into $B_i$, we have
\begin{equation*}
\bigl|\mathcal{I}(\tilde{N})_{i}\bigr| = \sum_{j=1}^{\tilde{N}} I\left(Z^{(j)}_{(1:k)} \in B_i\right) \sim \text{Poisson}(n\mu(B_i)),
\end{equation*}
and $|\mathcal{I}(\tilde{N})_{i}|$, $i\in[\ell_{m,n}]$ are mutually independent. Then the scaled Poissonized estimator is defined as 
\begin{equation*}
\begin{split}
 T(\tilde{N}) & = \sqrt{w} \sum_{\substack{1 \le  i \le  \ell_{m,n},\, |\mathcal{I}(\tilde{N})_{i}| \ge  2}} \,\sum_{j_1\neq j_2 \in \mathcal{I}(\tilde{N})_{i}} U^{(j_1)\top}U^{(j_2)}/(|\mathcal{I}(\tilde{N})_{i}| - 1\bigr),
 \end{split}
\end{equation*}
where $w = \text{Vol}(B_i) $ is the volume of any partition element, $i\in [\ell_{m,n}]$. 
Notice that when $\tilde{N} = n$, 
$T(\tilde{N}) / (\tilde{N}\sqrt{w}) $ becomes the estimator $T_{m,n}$ defined in 
equation (4) in the main text.
% \eqref{debiased_estimator}.

For $t,v\in \R$,
we define the auxiliary random variable
\begin{equation}\label{MN}
M(\tilde{N}) = t \left(T(\tilde{N}) - \EE{T(\tilde{N})}\right)   + v \frac{\tilde{N} - n}{\sqrt{n}}.
\end{equation}
The key result of the Poissonization approach in \cite{beirlant1998asymptotic} is the following lemma, which states the asymptotic normality of the original statistic $T(n)$, given the asymptotic normality of $M(\tilde{N})$:
\begin{lemma}[Proposition 1 of \cite{beirlant1998asymptotic}]
\label{de_poissonization_original}
Suppose that for 
some $\rho>0$ and
all $t,v\in \R$,
\begin{equation}
\label{original_lemma_condition}
\begin{split}
    \Phi_{n}(t,v) &:= \EE{e^{\i M(\tilde{N})}} \rightarrow e^{-t^2\rho^2 / 2 - v^2 / 2} \quad \text{as} \quad n\rightarrow \infty.
\end{split}
\end{equation}
Then 
\begin{equation*}
\frac{T(n) - \EE{T(n)} } {\rho}\xrightarrow{d} \N(0, 1).
\end{equation*}
\end{lemma}
In the proof of the lemma, we need to bound the characteristic function of $M(\tilde{N})$ to apply the dominated convergence theorem (refer to equations \eqref{cchfphi}, \eqref{cf_norm_bound}, \eqref{dominated_function}, and \eqref{dominated_convergence} for details; see also equation (10) in \cite{beirlant1998asymptotic}). 
However, with the condition \eqref{original_lemma_condition} alone, it is not clear  how to  construct a dominating function so that the dominated convergence theorem can be applied.

To address this issue, we introduce a 
\emph{novel splitting strategy}, where we split the test statistic $T(\tilde{N})$ into two components parameterized by a scalar  $a\in[0,1]$.
The parameterization is such that in the limit $a\to 1$, we recover the original test statistic. 
For $a<1$, we can bound the characteristic function of a split version of $M(\tilde{N})$ by the norm of the characteristic function of a Poisson variable. This approach allows us to use the dominated convergence theorem to conclude the proof. 
Next, we will introduce the split version of the statistic $T(\tilde{N})$.

For $0 \le  a \le  1$, let $\Delta(a)\subset \Delta(K, k)$ be the \emph{core} of the truncated Weyl chamber $\Delta(K, k)$ 
% from \eqref{Deck}
, defined as
$$\Delta(a)  = \Delta(a,K, k) := \{(z_1, \dots, z_k) \in \Delta(K, k): z_1 \le  1/K + a (1 - 1/K) \}. $$

\begin{figure}
    \centering
\begin{tikzpicture}[scale=3]
% Axes
\draw[->] (0,0) -- (1.2,0) node[right] {$z_1$};
\draw[->] (0,0) -- (0,1.2) node[above] {$z_2$};

% Simplex boundaries
\draw (0,1) -- (1,0);

% Shaded region for \Delta(K,k)
\fill[gray!20] (0.4,0) -- (0.2,0.2) -- (0.5,0.5) -- (0.6,0.4) -- (0.6,0) -- (0.4,0) -- cycle;

% Dashed lines for z1 = z2
\draw[dashed] (0,0) -- (1, 1);
\node at (1.1,0.8) {$z_1 = z_2$};

% Labels
\node[below left] at (0,0) {$0$};
\node[below right] at (1,0) {$1$};
\node[above left] at (0,1) {$1$};

% Dashed lines for k/K
\draw[dashed] (0,0.4) node[left] {$\frac{2}{K}$} -- (0.4,0) node[below] {$\frac{2}{K}$};

% Core region \Delta(a)
\def\a{0.5}
\def\CValue{5}
\draw[dashed] (0,0.6) node[left] {$\frac{1}{K}+a(1-\frac{1}{K})$} -- (0.4,0.6);
\draw[dashed] (0.6,0) 
-- (0.6,0.4);

\end{tikzpicture}
    \caption{Illustration of the set $\Delta(a)$ for $K=5$, $k=2$, and $a=1/2$. 
    }
    \label{fig:Da}
\end{figure}

Then 
$\Delta(0) = \{(1/K, 1/K,\ldots, 1/K)^\top\}$
contains only the uniform probability distribution over $K$ indices, while
$\Delta(1) = \Delta(K,k)$ contains the entire truncated Weyl chamber.  
We will first study the estimator defined on $\Delta(a)$ and then let $a \rightarrow 1$.

For $i\in [\ell_{m,n}]$,
letting $b(B_i) := \max_{(z_1, \dots, z_k) \in B_i} z_1$, 
we can reorder $\mathcal{B}_m$ such that $b(B_i)$ are non-decreasing in $i\in [\ell_{m,n}]$. 
For $a\in [0,1]$,
let $\ell_{m,n}(a)$ be the largest index $i\in [\ell_{m,n}]$ such that $b(B_i) \le  1/K + a (1 - 1/K)$.
If no such index exists, then we define $\ell_{m,n}(a)=0$, and define all sums over $1, \ldots, \ell_{m,n}(a)$ as zero below.
Then $\{B_1, \dots, B_{\ell_{m,n}(a)} \}$ 
forms a partition of a subset of $\Delta(a)$.
We point out that even though not all $B_i$, $i\le \ell_{m,n}$ are contained in $\Delta(K,k)$, for $a>0$, $B_1, \dots, B_{\ell_{m,n}(a)} $ are all contained in $\Delta(a)$ by definition.

For $\xi>0$, 
define the following quantities:
\begin{equation}\label{tsd}
\begin{split}
 T(\tilde{N}, a,\xi) & = \sqrt{\xi} \sum_{\substack{1 \le  i \le  \ell_{m,n}(a),\, |\mathcal{I}(\tilde{N})_{i}| \ge  2}} \,\sum_{j_1\neq j_2 \in \mathcal{I}(\tilde{N})_{i}} U^{(j_1)\top}U^{(j_2)}/(|\mathcal{I}(\tilde{N})_{i}| - 1\bigr), \\
  T'(\tilde{N}, a,\xi) & = \sqrt{\xi} \sum_{\substack{\ell_{m,n}(a) + 1 \le  i \le  \ell_{m,n},\, |\mathcal{I}(\tilde{N})_{i}| \ge  2}}\, \sum_{j_1\neq j_2 \in \mathcal{I}(\tilde{N})_{i}} U^{(j_1)\top}U^{(j_2)}/(|\mathcal{I}(\tilde{N})_{i}| - 1\bigr). \\
 \end{split}
\end{equation}
Now,
recalling that  $w = \text{Vol}(B_i) $ is the volume of any partition element $i\in [\ell_{m,n}]$,
for any $a\in[0,1]$
we can split the sums arising in our estimator 
$T_{m,\tilde{N}}$ 
from equation (4) in the main text
% from \eqref{debiased_estimator}
    into the two non-overlapping components
$T(\tilde{N}, a):=T(\tilde{N}, a,w)$
and
$T'(\tilde{N}, a):=T'(\tilde{N}, a,w)$.
When $a = 1$, 
$T(\tilde{N}, a) / (\tilde{N}\sqrt{w}) $ recovers the estimator $T_{m,\tilde{N}}$. 
% defined in \eqref{debiased_estimator}.
Moreover, for any 
$\tilde{N}$ and $a \in[0,1]$,
and $T(\tilde{N}, 1) = T(\tilde{N}, a) + T'(\tilde{N}, a)$.

Moreover, for $t,v\in \R$,
we define the auxiliary random variables (adding the same term as in \eqref{MN})
\begin{equation}
\label{poissonized_variable}
\begin{split}
M(\tilde{N}, a) & = t \left(T(\tilde{N}, a) - \EE{T(\tilde{N}, a)}\right)   + v \frac{\tilde{N} - n}{\sqrt{n}},    \\
M'(\tilde{N}, a) & = t \left(T'(\tilde{N}, a) - \EE{T'(\tilde{N}, a)}\right)   + v \frac{\tilde{N} - n}{\sqrt{n}}.  
\end{split}
\end{equation}
With these two random variables, 
we can present the following version of Lemma \ref{de_poissonization_original}:

\begin{lemma}[Poissonization with uncorrelated components]
\label{de_poissonization_uncorelated}
Suppose that for some constant $0 < a_0 < 1$, for all $a \in (a_0, 1)$,
there exist $\rho_a, \in \R$, such that 
the limit $\lim_{a\rightarrow 1} \rho_a = \rho$ exists.
Suppose further that for all $t,v\in \R, a \in (a_0, 1)$,
\begin{equation}\label{lemma_condition_corrected}
\begin{split}
    \Phi_{n, a}(t,v) &:= \EE{e^{\i M(\tilde{N}, a)}} \rightarrow e^{-t^2\rho_a^2 / 2 - v^2 / 2} \quad \textnormal{as} \quad n\rightarrow \infty,\, \textnormal{and}\\
    \Phi'_{n, a}(t,v) &:= \EE{e^{\i M'(\tilde{N}, a)}} \rightarrow e^{-t^2(\rho^2 - \rho_a^2) / 2  - v^2 / 2} \quad \textnormal{as} \quad n\rightarrow \infty.
\end{split}
\end{equation}
Then 
\begin{equation*}
\frac{T(n, 1) - \EE{T(n, 1)} } {\rho  }\xrightarrow{d} \N(0, 1).
\end{equation*}
\end{lemma}
In our proof, the updated condition \eqref{lemma_condition_corrected} is necessary to construct the dominating function \eqref{dominated_function} to prove the convergence in \eqref{dominated_convergence}.

Theorem 
% \ref{clt_calibrated_model} 
3.3
can be proved using
Lemma \ref{de_poissonization_uncorelated}. However, it is restricted to the case where $T(\tilde{N}, a)$ and $\tilde{N}$ are asymptotically uncorrelated, which is insufficient to prove Theorem 
% \ref{clt_miscalibrated_model}.
3.4. 
To address this challenge, we further extend the results to allow for correlations between the two components of the Poissonized variables from \eqref{poissonized_variable}:
\begin{lemma}[Poissonization with correlated components]
\label{de_poissonization}
Suppose that for some constant $0 < a_0 < 1$, 
and for all $a \in (a_0, 1)$,
there exist $\rho_a, \lambda_a \in \R$ with $\rho_a^2 - \lambda_a^2 > 0$, such that 
the limits $\lim_{a\rightarrow 1} \rho_a = \rho>0$, $\lim_{a\rightarrow 1} \lambda_a = \lambda$ exist.
Suppose further that for all $t,v\in \R, a \in (a_0, 1)$,
\begin{equation}\label{phin}
\begin{split}
    \Phi_{n, a}(t,v) &:= \EE{e^{\i M(\tilde{N}, a)}} \rightarrow e^{-t^2\rho_a^2 / 2 - tv\lambda_a - v^2 / 2} \quad \text{as} \quad n\rightarrow \infty,\\
    \Phi'_{n, a}(t,v) &:= \EE{e^{\i M'(\tilde{N}, a)}} \rightarrow e^{-t^2(\rho^2 - \rho_a^2) / 2 - tv(\lambda - \lambda_a) - v^2 / 2} \quad \text{as} \quad n\rightarrow \infty.
\end{split}
\end{equation}
Then 
\begin{equation}\label{Tn1}
\frac{T(n, 1) - \EE{T(n, 1)} } { \sqrt{\rho^2 - \lambda^2} }\xrightarrow{d} \N(0, 1).
\end{equation}
\end{lemma}
Lemma \ref{de_poissonization_uncorelated} can be viewed as a special case of Lemma \ref{de_poissonization} with $\lambda_a = 0$. Therefore, we only need to prove Lemma \ref{de_poissonization}. The proof can be found in Section \ref{proof_de_poissonization}.

\subsubsection{Proof of Theorem 
% \ref{clt_calibrated_model}
3.3
}

Using Lemma \ref{de_poissonization}, we will next to prove a central limit theorem for $M(\tilde{N}, a)$ and $M'(\tilde{N}, a)$.

By \eqref{poisson_order_-1} and \eqref{partition_variance_calibrated},
for $a>0$,
$\Var{T(\tilde{N}, a)} $ equals
\begin{equation*}
\begin{split}
&\sum_{q=1}^{\ell_{m,n}(a)} w \left(\sum_{i=1}^{k}\sum_{j=1}^{k}\left\{\EE{U^{(i)} U^{(j)} |Z_{(1:k)} \in B_q}\right\}^2 - \left\|\EE{U | Z_{(1:k)} \in B_q}\right\|^4\right)  + O\left(\frac{w\ell_{m,n}}{n\min{\mu(B_q)}}\right)\\
=&  O\left(\frac{w\ell_{m,n}}{n\min{\mu(B_q)}}\right) + 2w\sum_{q=1}^{\ell_{m,n}(a)} \sum_{i,j=1}^{k}\left\{\EE{Z_{(i)} Z_{(j)}|Z_{(1:k)} \in B_q}\right\}^2\\
& +2w \sum_{q=1}^{\ell_{m,n}(a)} 
\left( \sum_{i=1}^{k}\EE{Z_{(i)}|Z_{(1:k)} \in B_q}\biggl(\EE{Z_{(i)}|Z_{(1:k)} \in B_q}-2\EE{Z_{(i)}^2|Z_{(1:k)} \in B_q}\biggr) \right)\\
 \rightarrow & 
2\int_{\Delta(a)} \bigl(\|Z_{(1:k)}\|_2^2 + 2 \|Z_{(1:k)}\|_3^3 + \|Z_{(1:k)}\|_2^4\bigr) dZ_{(1:k)}
=: \sigma_{0,a}^2.
\end{split}
\end{equation*}
In the last step, we have used the definition of the Riemann integral of a continuous function on $\supp\left(Z_{(1:k)}\right) \cap \Delta(a) = \Delta(a)$.
For instance, for the convergence to the integral of $\|Z_{(1:k)}\|_2^2$, we have used the fact that $Z_{(1:k)} \mapsto \|Z_{(1:k)}\|_2^2$ is a continuous map and
\begin{equation*}
\begin{split}
\inf_{Z_{(1:k)} \in B_q} \|Z_{(1:k)}\|_2^2 \le  \sum_{i,j=1}^{k}\left\{\EE{Z_{(i)} Z_{(j)}|Z_{(1:k)} \in B_q}\right\}^2 \le  \sup_{Z_{(1:k)} \in B_q} \|Z_{(1:k)}\|_2^2
\end{split}
\end{equation*}
Further, we have used that for $a>0$ and any $0<a'<a$, 
the partition up to index $\ell_{m,n}(a)$ covers  $\Delta(a')$ for sufficiently large $m$, and that $a\mapsto \sigma_{0,a}^2$ is continuous.

Similarly, $T'(\tilde{N}, a)$ is a sum over bins partitioning $\Delta(K,k) \setminus \Delta(a)$, and thus
\begin{equation*}
\Var{T'(\tilde{N}, a)} \rightarrow 2\int_{ \Delta(K,k) \setminus \Delta(a)} \bigl(\|Z_{(1:k)}\|_2^2 + 2 \|Z_{(1:k)}\|_3^3 + \|Z_{(1:k)}\|_2^4\bigr) dZ_{(1:k)}
= \sigma_0^2 - \sigma_{0,a}^2.
\end{equation*}
By \eqref{partition_covariance_calibrated}, 
$\Cov{T(\tilde{N}, a), \frac{\tilde{N}}{\sqrt{n}}}  = \Cov{T'(\tilde{N}, a), \frac{\tilde{N}}{\sqrt{n}}} = 0$.
Letting for all $i\in[\ell_{m,n}]$,
\begin{equation}\label{mbi}
M(B_i) = \sum_{j_1\neq j_2 \in \mathcal{I}(\tilde{N})_{i}} U^{(j_1)\top}U^{(j_2)}/(|\mathcal{I}(\tilde{N})_{i}| - 1\bigr),
\end{equation}
we have by \eqref{third_moment_calibrated} that
\begin{equation*}
\begin{split}
 \EE{t^{3}n^3 \sqrt{w^{3}}\sum_{\substack{1 \le  i \le  \ell_{m,n},\, |\mathcal{I}(\tilde{N})_{i}| \ge  2}} \frac{1}{n^3} \bigl|M(B_i)- \EE{M(B_i)}\bigr|^{3}} \le  \kappa_6 t^{3}\ell_{m,n}w\sqrt{w} \rightarrow  0.
\end{split}
\end{equation*}
Then by the Lyapunov Central Limit Theorem, we have 
$\frac{M(\tilde{N}, a)}{t^2\sigma_{0,a}^2 + v^2} \xrightarrow{d} \N(0,1)$,
and 

$\frac{M'(\tilde{N}, a)}{t^2(\sigma_0^2 - \sigma_{0,a}^2) + v^2} \xrightarrow{d} \N(0,1)$.
The proof is concluded by applying Lemma \ref{de_poissonization}.

\subsection{Proof of Theorem 
% \ref{clt_miscalibrated_model}
3.4
}
\label{proof_clt_miscalibrated_model}

Recalling \eqref{tsd},
define, for any $a\in[0,1]$,  
$T(\tilde{N}, a):=T(\tilde{N}, a,n^{-1})$
and
$T'(\tilde{N}, a):=T'(\tilde{N}, a,n^{-1})$.
While this notation coincides with the differently scaled test statistics from the proof of Theorem 
% \ref{clt_calibrated_model}, 
3.3,
no confusion will arise.
For $t,v\in \R$,
recall $M(\tilde{N}, a)$ and $M'(\tilde{N}, a)$  from \eqref{poissonized_variable}.
Using Lemma \ref{de_poissonization}, we will next  prove a central limit theorem for $M(\tilde{N}, a)$ and $M'(\tilde{N}, a)$.
By \eqref{partition_variance}, \eqref{poisson_order_0} and \eqref{poisson_order}, 
\begin{equation*}
\begin{split}
&\Var{T(\tilde{N}, a)}  
=  \sum_{i=1}^{\ell_{m,n}(a)}4\EE{U^{\top}|Z_{(1:k)} \in B_i}\Cov{U|Z_{(1:k)} \in B_i}\EE{U|Z_{(1:k)} \in B_i} \mu(B_i) \\
& + \sum_{i=1}^{\ell_{m,n}(a)}  \bigl\|\EE{U | Z_{(1:k)} \in B_i}\bigr\|^4 \mu(B_i)  + O\left(\frac{\ell_{m,n}}{n}\right)
 \\
  \rightarrow & \EE{\|\EE{U|Z_{(1:k)}}\|^4 I\bigl(Z_{(1:k)}\in \Delta(a)\bigr)} \\
  &+ 4\EE{\EE{U^{\top}|Z_{(1:k)}}\Cov{U|Z_{(1:k)}}\EE{U|Z_{(1:k)}}I\bigl(Z_{(1:k)}\in \Delta(a)\bigr)}=: \rho_a^2
 .
\end{split}
\end{equation*}
Further, by \eqref{partition_covariance} we have
\begin{equation*}
\begin{split}
\Cov{T(\tilde{N}, a), \frac{\tilde{N}}{\sqrt{n}}} & = \sum_{i=1}^{\ell_{m,n}(a)} \left(\bigl\|\EE{U | Z_{(1:k)} \in B_i}\bigr\|^2 \mu(B_i) - O\bigl(n \max\mu(B_i) e^{-n\min{\mu(B_i)}}\bigr) \right)\\
& \rightarrow \EE{\|\EE{U|Z_{(1:k)}}\|^2I\bigl(Z_{(1:k)}\in \Delta(a)\bigr)} =: \lambda_a.
\end{split}
\end{equation*}

Now $\Var{\tilde{N} / \sqrt{n}} = 1$, therefore 
\begin{equation*}
\begin{split}
\Var{M(\tilde{N}, a)} &= t^2 \Var{T(\tilde{N}, a)} + 2tv \Cov{T(\tilde{N}, a), \frac{\tilde{N}}{\sqrt{n}}} + v^2\Var{\tilde{N} / \sqrt{n}}\\
& \rightarrow t^2 \rho_a^2 + 2tv \lambda_a + v^2.
\end{split}
\end{equation*}
Let $\rho^2 = \EE{\|\EE{U|Z_{(1:k)}}\|^4} + 4\EE{\EE{U^{\top}|Z_{(1:k)}}\Cov{U|Z_{(1:k)}}\EE{U|Z_{(1:k)}}}$ and 

$\lambda = \EE{\|\EE{U|Z_{(1:k)}}\|^2}$. 
Then similar calculations lead to
\begin{equation*}
\Var{T'(\tilde{N}, a)} \rightarrow \rho^2 - \rho_a^2,
\quad\textnormal{and}
\quad \Var{M'(\tilde{N}, a)} \rightarrow t^2 (\rho^2 - \rho_a^2) + 2tv (\lambda - \lambda_a) + v^2.
\end{equation*}
Moreover, as $a \rightarrow 1$,  $\rho_a^2 \rightarrow \rho^2$ and $\lambda_a \rightarrow \lambda$. 

Recall $M(B_i)$,  $i\in[\ell_{m,n}]$,
 from \eqref{mbi}.
To apply the Lyapunov Central Limit Theorem, we will next show that
\begin{equation*}
\EE{t^{3}\sum_{\substack{1 \le  i \le  \ell_{m,n},\, |\mathcal{I}(\tilde{N})_{i}| \ge  2}} 
n^{-\frac{3}{2}} \bigl|M(B_i) - \EE{M(B_i)}\bigr|^{3} } \rightarrow 0,
\end{equation*}
and 
$\EE{v^{3}\sum_{\substack{1 \le  i \le  \ell_{m,n},\, 
|\mathcal{I}(\tilde{N})_{i}| \ge  2}} \bigl\vert|\mathcal{I}(\tilde{N})_{i}| - \mu(B_i)\bigr\vert^{3}/n^{\frac{3}{2}}} \rightarrow 0$.
By \eqref{third_moment}, we have
\begin{equation*}
\begin{split}
 \EE{t^{3}\sum_{\substack{1 \le  i \le  \ell_{m,n},\, |\mathcal{I}(\tilde{N})_{i}| \ge  2}} \frac{1}{n^{\frac{3}{2}}} \bigl|M(B_i)- \EE{M(B_i)}\bigr|^{3}} \le  \frac{\kappa_2 t^{3}}{n\sqrt{\min{\mu(B_i)}}} + \kappa_3\frac{t^3 \ell_{m,n}}{n^{\frac{3}{2}}}\rightarrow  0.
\end{split}
\end{equation*}
By \eqref{poisson_third_moments}, we have
\begin{equation*}
\EE{v^{3}\sum_{\substack{1 \le  i \le  \ell_{m,n},\\ |\mathcal{I}(\tilde{N})_{i}| \ge  2}} \frac{\bigl| |\mathcal{I}(\tilde{N})_{i}| - \mu(B_i)\bigr|^{3}}{n^{\frac{3}{2}}}} \le  \kappa_4 v^3 \sum_{\substack{1 \le  i \le  \ell_{m,n},\\ |\mathcal{I}(\tilde{N})_{i}| \ge  2}} \mu^{\frac{3}{2}}(B_i) \le  \kappa_4 v^3\max{\sqrt{\mu(B_i)}} \rightarrow 0.
\end{equation*}
Thus, we conclude 
\begin{equation*}
\begin{split}
& \EE{\sum_{\substack{1 \le  i \le  \ell_{m,n},\, |\mathcal{I}(\tilde{N})_{i}| \ge  2}} \Biggl|\frac{  t  \bigl(M(B_i)- \EE{M(B_i)}\bigr) + v\bigl(|\mathcal{I}(\tilde{N})_{i}|  - \mu(B_i)\bigr)  }{\sqrt{n}}\Biggr|^3}\\
&\le   4 \EE{|t|^{3}\sum_{\substack{1 \le  i \le  \ell_{m,n},\\ |\mathcal{I}(\tilde{N})_{i}| \ge  2}} \frac{1}{n^{\frac{3}{2}}} |M(B_i)- \EE{M(B_i)}|^{3}} + 4 \EE{|v|^{3}\sum_{\substack{1 \le  i \le  \ell_{m,n},\\ |\mathcal{I}(\tilde{N})_{i}| \ge  2}} \frac{\bigl| |\mathcal{I}(\tilde{N})_{i}| - \mu(B_i)\bigr|^{3}}{n^{\frac{3}{2}}}}  \rightarrow 0.
\end{split}
\end{equation*}
Then by the Lyapunov Central Limit Theorem,
$\frac{M(\tilde{N}, a)}{t^2 \rho_a^2 + 2tv \lambda_a + v^2} \xrightarrow{d} \N(0,1)
$
and 

$\frac{M'(\tilde{N}, a)}{t^2 (\rho^2 - \rho_a^2) + 2tv (\lambda - \lambda_a) + v^2} \xrightarrow{d} \N(0,1)$.
The theorem follows by applying Lemma \ref{de_poissonization}.

\subsection{Proof of Lemma \ref{de_poissonization}}
\label{proof_de_poissonization}

\begin{proof}
Define $T^c(n,a)=T(n, a) - \EE{T(n, a)}$ for all $n,a$.
For fixed $a < 1$ and $t\in \R$, clearly 
\begin{equation*}
\EE{\exp (\i t T^c(n,a)) } = \EE{\exp\left(\i tT^c(\tilde{N},a)\right) \Big| \tilde{N} = n}.
\end{equation*}
For any non-negative integer $j$,
recall that $P(\tilde{N} = j) = e^{-n} n^{j}/(j!)$.
Note that for all $u\in \R$, 
$$\sum_{j=0}^{\infty} e^{\i uj}  \EE{\exp\left(\i tT^c(\tilde{N},a)\right) \Big| \tilde{N} = j} P(\tilde{N} = j)
= \EE{ e^{\i u\tilde{N}+\i tT^c(\tilde{N},a)}},$$ 
and  in particular since $|\exp(\i x)|\le 1$ for all $x\in \mathbb{R}$,
the series on the left hand side converges. 
Then using that $\int_{-\pi}^{\pi} e^{\i u j} du = 0$ for any integer $j \neq 0$, 
that $\int_{-\pi}^{\pi} e^{\i u j} du = 2\pi$ for $j=0$,
and by the Dominated Convergence Theorem,
we conclude  by the change of variables $v = u\sqrt{n}$ that
for all non-negative integers $n$,
\begin{equation}
\label{inverse_fourier}
\begin{split}
&2\pi P(\tilde{N} = n) \EE{\exp \left(\i tT^c(\tilde{N},a)\right) \Big| \tilde{N} = n}\\
 =& \int_{-\pi}^{\pi} e^{-\i un}  \EE{\exp\left(\i tT^c(\tilde{N},a) + \i u\tilde{N}\right)} du
\\
= &  \frac{1}{\sqrt{n}}\int_{-\sqrt{n}\pi}^{\sqrt{n}\pi}  \EE{\exp\left(\i tT^c(\tilde{N},a) + \i v\frac{\tilde{N} - n}{\sqrt{n}}\right)} dv. 
\end{split}
\end{equation}

By Stirling's formula  \citep{robbins1955remark}, $2\pi P(\tilde{N} = n) = \frac{2\pi e^{-n} n^{n}}{n!} \sim \sqrt{\frac{2\pi}{n}}$ as $n\rightarrow \infty$. 
Then by recalling $\Phi_{n,a}$ from \eqref{phin},
\eqref{inverse_fourier} gives
\begin{equation}\label{cchfphi}
\EE{\exp\left(\i tT^c(\tilde{N},a)\right) \Big| \tilde{N} = n} \sim \frac{1}{\sqrt{2\pi}} \int_{-\pi\sqrt{n}}^{\pi\sqrt{n}} \Phi_{n, a}(t,v) dv.
\end{equation}

Next, note that $\tilde{N} = \sum_{i=1}^{\ell_{m,n}} |\mathcal{I}(\tilde{N})_{i}|$. 
Letting
$\tilde{N}(a) = \sum_{i=\ell_{m,n}(a) + 1}^{\ell_{m,n}} |\mathcal{I}(\tilde{N})_{i}|$, we have $\tilde{N}(a) \sim \text{Poisson}(n \mu_n(a))$ with $\mu_n(a) \rightarrow 1 -\mu\bigl(\Delta(a)\bigr) =: \mu_a$ as $n\rightarrow \infty$. 
We consider the case where $\mu_a > 0$ first.
Writing 
$\tilde{N} = \tilde{N}(a)+ (\tilde{N}-\tilde{N}(a))$,
using  the independence of $|\mathcal{I}(\tilde{N})_{i}|$ across $i\in[\ell_{m,n}]$,
and recalling the characteristic function $\Phi_{n,a}$ of $M(\tilde{N}, a)$ from \eqref{poissonized_variable},
we have for any $v\in \R$ that
\begin{equation}
\label{cf_norm_bound}
\begin{split}
&|\Phi_{n,a}(t,v)| 
=
\Bigl| \EE{e^{\i M(\tilde{N}, a)}}\Bigr|
= 
\Bigl|\EE{e^{\i t T^c(\tilde{N},a)   + v (\tilde{N} - n)/\sqrt{n}}}\Bigr|\\
&
= \Bigl|\EE{e^{\i t T^c(\tilde{N}, a)   
+ \i v [\tilde{N} - \tilde{N}(a) - (n - n\mu_n(a))]/\sqrt{n}}} 
\EE{e^{\i v (\tilde{N}(a) - n\mu_n(a))/\sqrt{n}}}\Bigr| \\
& \le  \biggl| \EE{e^{\i v (\tilde{N}(a) - n\mu_n(a))/\sqrt{n}}} \biggr| =: g_n(v).
\end{split}
\end{equation}

Since the characteristic function of a $\text{Poisson}(\lambda)$ random variable at $t\in \R$ is $\exp \{\lambda \bigl(\exp\{\i t\}  - 1 \bigr) \}$, 
we find for any $v\in \R$ that by a Taylor expansion of the cosine function at zero,
$g_n(v) = e^{n\mu_n(a) [\cos(v/\sqrt{n}) - 1]} \rightarrow e^{-\frac{\mu_a v^2}{2}}$
as $n\to\infty$.
Moreover, by the change of variables $u  = v/\sqrt{n}$,
\begin{equation*}
\int_{-\pi\sqrt{n}}^{\pi\sqrt{n}} g_n(v) dv = \sqrt{n} \int_{-\pi}^{\pi} e^{n \mu_n(a) [\cos(u) - 1]} du.
\end{equation*}
Now, 
let $\tilde{N}'(a) \sim \text{Poisson}(n\mu_n(a))$ be a Poisson random variable independent of $\tilde{N}(a)$, so  that
\begin{equation*}
\EE{e^{\i u (\tilde{N}(a) - \tilde{N}'(a))}}  = \biggl| \EE{e^{\i u (\tilde{N}(a) - n\mu_n(a))}} \biggr|^2 
= (g_{n}(u\sqrt{n}))^2
= e^{2n\mu_n(a) [\cos(u) - 1]}.
\end{equation*}
Using the fact that $\tilde{N}(a) - \tilde{N}'(a)$ only takes integer values, 
and that $\int_{-\pi}^{\pi} e^{\i u j} du = 0$ for nonzero integers $j$, we have 
\begin{equation*}
\begin{split}
&\int_{-\pi}^{\pi} e^{2n\mu_n(a) [\cos(u) - 1]} du  = \int_{-\pi}^{\pi}\EE{e^{\i u (\tilde{N}(a) - \tilde{N}'(a))}}du \\
& = 2\pi P(\tilde{N}(a) - \tilde{N}'(a) = 0) = 2\pi \sum_{j=0}^{\infty} e^{-2n\mu_n(a)} \frac{\bigl(n\mu_n(a)\bigr)^{-2k}}{(k!)^2} =2\pi e^{-2n\mu_n(a)} I_0(2n\mu_n(a)),  
\end{split}
\end{equation*}
where $I_0(\cdot)$ is the modified Bessel function of the first kind. 
It is known that $I_0(2n) \sim e^{2n\mu_n(a)}/\sqrt{4\pi n\mu_n(a)}$  as $n\to\infty$  \citep{abramowitz1968handbook}.
Therefore,  as $n\to\infty$,  
\begin{equation}
\label{dominated_function}
\int_{-\pi\sqrt{2n}}^{\pi\sqrt{2n}} g_{2n}(v) dv = \sqrt{2n} \int_{-\pi}^{\pi} e^{2n\mu_n(a)\cos(u) - 2n\mu_n(a)} du \rightarrow \sqrt{\frac{2\pi}{\mu_a}} = \int_{-\infty}^{\infty} e^{-\frac{\mu_a v^2}{2}} dv.
\end{equation}
Therefore 
for all $t\in \R$,
$v\mapsto\Phi_{n,a}(t,v)$ is dominated by $v\mapsto g_n(v)$
from \eqref{cf_norm_bound}
and $\int_{-\pi\sqrt{n}}^{\pi\sqrt{n}} g_n(v) dv$ converges to $\int_{-\infty}^{\infty} e^{-\frac{\mu_a v^2}{2}} dv$  as $n\to\infty$. 

By \eqref{phin}, $\Phi_{n, a}(t,v) \rightarrow e^{-t^2\rho_a^2 / 2 - tv\lambda_a - v^2 / 2} = e^{-t^2(\rho_a^2 -\lambda_a^2)/ 2  - (v+t\lambda_a)^2 / 2}$  as $n\to\infty$. 
Then 
from \eqref{cchfphi} and
by a variant of the dominated convergence theorem (e.g. \cite{rao1973linear} p. 136),
\begin{equation}
\label{dominated_convergence}
    \EE{\exp\left(\i tT^c(\tilde{N},a)\right) \Big| \tilde{N} = n} \rightarrow \frac{1}{\sqrt{2\pi}} \int_{-\infty}^{\infty} e^{-t^2(\rho_a^2 -\lambda_a^2)/ 2  - (v+t\lambda_a)^2 / 2} dv = e^{-t^2(\rho_a^2 -\lambda_a^2)/ 2 }.
\end{equation}
If $\mu_a = \mu(\Delta(a)) = 0$, 
then we have $P(\tilde{N}(a) = 0) = 1$. 
In that case, we must have $\rho_a = \lambda_a = 0$, 
and so \eqref{dominated_convergence} holds trivially.

This shows that $T^c(n,a) \xrightarrow{d} \N(0, \rho_a^2 -\lambda_a^2)$. 
Following the same procedure, we have $T'(n, a) - \EE{T'(n, a)} \xrightarrow{d} \N(0, \rho^2 - \rho_a^2 -(\lambda - \lambda_a)^2)$.
We can finish the proof by applying
Lemma 5 of \cite{le_cam1958theoreme}; provided below with a proof for completeness:
\begin{lemma}[Lemma 5 of \cite{le_cam1958theoreme}]
\label{le_cam_lemma}
Suppose 
that for any $a\in[0,1]$, we have 
two sequences
of random variables $(T_{n,a})_{n\ge1}, (T'_{n,a})_{n\ge1}$ 
satisfying $T_{n,a} \xrightarrow{d} T_a$, $T'_{n,a} \xrightarrow{d} T'_a$ as $n\rightarrow \infty$, 
and that as $a \rightarrow 1$, $T_a \xrightarrow{d} T$, $T'_a \xrightarrow{d} 0$, i.e., $T'_a$ converges to a point mass at zero (or, equivalently, $T'_a$ converges in probability to zero). 
If for all $n\ge 1$,
$T_{n,a} + T'_{n,a}$ does not depend on $a$, then 
\begin{equation*}
\lim_{a \rightarrow 1} \lim_{n\rightarrow \infty} \EE{e^{\i t (T_{n,a} + T'_{n,a})}} = \lim_{n\rightarrow \infty} \EE{e^{\i t (T_{n,a} + T'_{n,a})}} = \EE{e^{it T}}.
\end{equation*}
\begin{proof}
Let $u, v:\mathbb{C}\to \mathbb{C}$ be bounded continuous functions. 
For any $a\in[0,1]$, consider
$A(a) := \limsup_{n\rightarrow \infty}$ $  \EE{u(T_{n,a})v(T'_{n,a})}$.
Since $T_{n,a} \xrightarrow{d} T_a$ and 
$T'_{n,a} \xrightarrow{d} T'_a$, by Prohorov's theorem (see e.g., Theorem 2.4 of \cite{van2000asymptotic}) and the definition of $\limsup$, 
there exists a sub-sequence $(n_j)_{j\ge 1}$ of the positive integers 
such that the joint distribution of $(T_{n_j,a}, T'_{n_j,a})$ converges as $j \rightarrow \infty$, 
and we have
\begin{equation*}
    A(a) = \lim_{j\rightarrow \infty}  \EE{u(T_{n_j,a})v(T'_{n_j,a})}.
\end{equation*}
The limiting law of $(T_{n_j,a}, T'_{n_j,a})$
has the same marginal distributions
as $T_a$ and $T'_a$, respectively.
Let $(T_a, T'_a)$ represent the joint limit.
Then, since weak convergence implies the convergence of expectations of bounded   continuous functions, we have
$A(a) = \EE{u(T_a)v(T'_a)}$.

As $a \rightarrow 1$, $T_a \xrightarrow{d} T$, and $T'_a$ converges to the point mass at zero, 
thus by Slutsky's theorem (see e.g., Lemma 2.8 of \cite{van2000asymptotic}), 
we have
$\lim_{a \rightarrow 1}\EE{u(T_a)v(T'_a)} = v(0) \EE{u(T)}$.
Therefore,
\begin{equation}\label{abc}
\lim_{a \rightarrow 1}\limsup_{n\rightarrow \infty}  \EE{u(T_{n,a})v(T'_{n,a})} = v(0) \EE{u(T)}.
\end{equation}
Similarly, we have 
\begin{equation*}
\lim_{a \rightarrow 1}\liminf_{n\rightarrow \infty}  \EE{u(T_{n,a})v(T'_{n,a})} = v(0) \EE{u(T)}.
\end{equation*}
Now take
$u: z\mapsto  e^{\i t z}$ and 
$v: z\mapsto  e^{\i t z}$. 
Then,
since by assumption $T_{n,a} + T'_{n,a}$ does not depend on $a$,
$\liminf_{n\rightarrow \infty}$ $  \EE{u(T_{n,a})v(T'_{n,a})}$ and 
$\limsup_{n\rightarrow \infty}  \EE{u(T_{n,a})v(T'_{n,a})}$
do not depend on $a$ either.
Hence, since their limits as $a\to 1$ are equal, they must be equal for all $a\in[0,1]$.
Thus, we must have 
\begin{equation*}
\liminf_{n\rightarrow \infty}  \EE{e^{\i t (T_{n,a} + T'_{n,a})}} = \limsup_{n\rightarrow \infty}  \EE{e^{\i t (T_{n,a} + T'_{n,a})}} = \lim_{n\rightarrow \infty}  \EE{e^{\i t (T_{n,a} + T'_{n,a})}} = \EE{e^{it T}},
\end{equation*}
where the last equation follows from \eqref{abc}. This finishes the proof.
\end{proof}
\end{lemma}
Taking $T_{n,a} = T(n,a) - \EE{T(n,a)}$ and $T'_{n,a} = T'(n,a) - \EE{T'(n,a)}$, we notice that 
$T_{n,a} + T'_{n,a} = T(n,1) - \EE{T(n,1)}$, which does not depend on $a$.
Moreover, the conditions required by Lemma \ref{le_cam_lemma} on the limiting distributions of
 $(T_{n,a})_{n\ge1}, (T'_{n,a})_{n\ge1}$ 
as $n\to\infty$ hold by our previous analysis.
Since $\lim_{a\rightarrow 1} \rho_a = \rho$
and
$\lim_{a\rightarrow 1} \lambda_a = \lambda$, 
we have $\N(0, \rho_a^2 -\lambda_a^2)\xrightarrow{d} \N(0, \rho^2 - \lambda^2)$
and
$ \N(0, \rho^2 - \rho_a^2 -(\lambda - \lambda_a)^2) \xrightarrow{d} 0$. 
Then by Lemma \ref{le_cam_lemma},  \eqref{Tn1} follows.
This finishes the proof.
\end{proof}

\subsection{Proof of Proposition 
% \ref{sigma1_hat_convergence}
3.5
}
\label{proof_sigma1_hat_convergence}
To prove $\hat\sigma_1^2 \rightarrow_{p} \sigma_1^2$, we will show that $\EE{\hat\sigma_1^2}$ is close to $\sigma_1^2$ and that $\Var{\hat\sigma_1^2}$ tends to zero. 
Then the conclusion  will follow from Chebyshev’s inequality.
For notational simplicity, let $\underline{Z} = Z_{(1:k)}$ and $\underline{Y} = Y_{r_{1:k}}$.
Recall that we define $U = \underline{Y} - \underline{Z}$, 
and let $F: \underline{Z}\mapsto \EE{U|\underline{Z}}$, $G:\underline{Z} \mapsto \EE{U U^{\top}|\underline{Z}}$. Then we can write $\sigma_1^2$ as 
\begin{equation}
\begin{split}
\sigma_1^2 &=  \sum_{i=1}^{\ell_{m,n}} \int_{B_i}\|F(\underline{Z})\|^4 dP_{\underline{Z}} - \Biggl( \sum_{i=1}^{\ell_{m,n}} \int_{B_i}\|F(\underline{Z})\|^2 dP_{\underline{Z}} \Biggr)^2 \\
& + 4 \sum_{i=1}^{\ell_{m,n}} \int_{B_i} F(\underline{Z})^{\top}G(\underline{Z}) F(\underline{Z}) dP_{\underline{Z}} - 4 \sum_{i=1}^{\ell_{m,n}} \int_{B_i}\|F(\underline{Z})\|^4 dP_{\underline{Z}}.
\end{split}
\end{equation}

We will first prove the following lemma, which shows the H\"older smoothness of $\underline{z}\mapsto\EE{\underline{Y} |\underline{Z}= \underline{z}}$ and of $G$.
\begin{lemma}
\label{holder_smooth_cov}
Under Condition 
% \ref{ass_1}, 
3.1,
$\underline{z}\mapsto\EE{\underline{Y} | \underline{Z}=\underline{z}}$ and $\underline{z}\mapsto\EE{U U^{\top}|\underline{Z}=\underline{z}}$ are H\"older continuous
on $\Delta(K, k)$
with H\"older smoothness parameter $0 < s \le  1$.
\end{lemma}
\begin{proof}
For $\underline{z}^{(1)}, \underline{z}^{(2)} \in  \Delta(K, k)$
and $j \in \{1, 2,\dots, k\}$,  we have
by Condition 
% \ref{ass_1} 
3.1
that
\begin{equation*}
\begin{split}
&\biggl| \EE{\underline{Y}^{(1)} \Bigl| \underline{Z}^{(1)}=\underline{z}^{(1)}}_j - \EE{\underline{Y}^{(2)} \Bigl| \underline{Z}^{(2)}=\underline{z}^{(2)}}_j \biggr| \\
\le &  \biggl| \EE{U^{(1)} \Bigl| \underline{Z}^{(1)}=\underline{z}^{(1)}}_j - \EE{U^{(2)} \Bigl| \underline{Z}^{(2)}=\underline{z}^{(2)}}_j \biggr| 
+ \bigl|\underline{z}^{(1)}_j - \underline{z}^{(2)}_j \bigr| 
 \\
 \le &   L\bigl\|\underline{z}^{(1)} - \underline{z}^{(2)} \bigr\|^s + \bigl|\underline{z}^{(1)}_j - \underline{z}^{(2)}_j \bigr|^s 
\le  (L+1)\bigl\|\underline{z}^{(1)} - \underline{z}^{(2)} \bigr\|^s.
\end{split}
\end{equation*}
Therefore $\underline{z}\mapsto \EE{\underline{Y} | \underline{z}}$ is H\"older continuous with H\"older smoothness parameter $s$. 

Next, for $\underline{z}\in  \Delta(K, k)$ 
and $i,j\in[k]$, let
$\EE{U U^{\top}|\underline{Z}=\underline{z}}_{i,j}$ denote the element at the $i$-th row and $j$-th column of $\EE{U U^{\top}|\underline{Z}=\underline{z}}$.
Then 
\begin{equation*}
G(\underline{z})_{i,j} = \EE{U U^{\top}|\underline{Z}=\underline{z}}_{i,j} = \EE{\underline{Y}_i \underline{Y}_j| \underline{Z}=\underline{z}} - \EE{\underline{Y}_i | \underline{Z}=\underline{z}} \underline{z}_j  - \EE{\underline{Y}_j | \underline{Z}=\underline{z}} \underline{z}_i + \underline{z}_i \underline{z}_j.
\end{equation*}
Let $\underline{z}^{(1)}, \underline{z}^{(2)} \in  \Delta(K, k)$.
By Condition 
% \ref{ass_1}
3.1 
and as $\|\underline{z}^{(1)}\|_{\infty} \le  1, \|\underline{z}^{(1)}\|_{\infty} \le  1$, we have 
\begin{equation*}
\begin{split}
& \biggl| \EE{\underline{Y}^{(1)}_i \Bigl| \underline{z}^{(1)}=\underline{z}^{(1)}} \underline{z}^{(1)}_j - \EE{\underline{Y}^{(2)}_i \Bigl| \underline{Z}^{(2)}=\underline{z}^{(2)}} \underline{z}^{(2)}_j \biggr| \\
\le &   \biggl| \EE{\underline{Y}^{(1)}_i \Bigl| \underline{Z}^{(1)}=\underline{z}^{(1)}} \underline{z}^{(1)}_j - \EE{\underline{Y}^{(2)}_i \Bigl| \underline{Z}^{(2)}=\underline{z}^{(2)}} \underline{z}^{(1)}_j \biggr| \\
& + \biggl| \EE{\underline{Y}^{(2)}_i \Bigl| \underline{Z}^{(2)}=\underline{z}^{(2)}} \underline{z}^{(1)}_j - \EE{\underline{Y}^{(2)}_i \Bigl| \underline{Z}^{(2)}=\underline{z}^{(2)}} \underline{z}^{(2)}_j \biggr| \\
\le &  L\bigl\|\underline{z}^{(1)} - \underline{z}^{(2)}\bigr\|^s  + \bigl|\underline{z}^{(1)}_j - \underline{z}^{(2)}_j \bigr| 
\le  (L+1)\bigl\|\underline{z}^{(1)} - \underline{z}^{(2)} \bigr\|^s. 
\end{split}
\end{equation*}
Further, for all $i,j$,
\begin{equation*}
\begin{split}
\bigl| \underline{z}^{(1)}_i \underline{z}^{(1)}_j - \underline{z}^{(2)}_i \underline{z}^{(2)}_j \bigr| &\le   \bigl|  \underline{z}^{(1)}_i \underline{z}^{(1)}_j - \underline{z}^{(1)}_i \underline{z}^{(2)}_j \bigr|  + \bigl|  \underline{z}^{(1)}_i \underline{z}^{(2)}_j - \underline{z}^{(2)}_i \underline{z}^{(2)}_j  \bigr| \\
&\le   \bigl|\underline{z}^{(1)}_j - \underline{z}^{(2)}_j \bigr|  + \bigl|\underline{z}^{(1)}_i - \underline{z}^{(2)}_i \bigr| 
\le   2\bigl\|\underline{z}^{(1)} - \underline{z}^{(2)} \bigr\|^s. 
\end{split}
\end{equation*}
Note that $\underline{Y}_i \underline{Y}_j = 0$ for $i \neq j$ and $\underline{Y}_i  \underline{Y}_j= \underline{Y}_i$ for $i = j$. Then, by Condition 
% \ref{ass_1},
3.1,
we have 
$\Bigl| G(\underline{z}^{(1)})_{i,j} - G(\underline{z}^{(2)})_{i,j}\Bigr| \le  L'\bigl\|\underline{z}^{(1)} - \underline{z}^{(2)} \bigr\|^s$
for some $L'$. This finishes the proof.
\end{proof}

Recall from 
% \Cref{vare} 
Algorithm 1
that
for $i\in[\ell_{m,n}]$,
we have
$\mathbb{E}_n[U]^{(i)} = \frac{1}{|\mathcal{I}_{m,n, i}|} \sum_{j \in \mathcal{I}_{m, n, i}} U^{(j)}$ 
and $\textnormal{Cov}_n[U]^{(i)} 
= \frac{1}{|\mathcal{I}_{m,n, i}|}$ $ \sum_{j \in \mathcal{I}_{m, n, i}} (U^{(j) \top} U^{(j)} - \mathbb{E}_n[U]^{(i)}\mathbb{E}_n[U]^{(i)\top})$.
Letting $N_i = |\mathcal{I}_{m, n, i}|$ for $i\in[\ell_{m,n}]$, 
we will first calculate $\EE{\frac{N_i}{n} \Bigl\Vert\mathbb{E}_n[U]^{(i)}\Bigr\Vert^4 }$. 
Similar to the calculation of \eqref{emn}, 
for distinct values $j_1, j_2, j_3, j_4 \in \mathcal{I}_{m,n, i}$, we have 
$$\EE{U^{(j_1) \top}U^{(j_2)}U^{(j_3 ) \top}U^{(j_4)}} = \|\EE{U | \underline{Z} \in B_i}\|^4.$$ 
Since $U^{(i)}$ is a bounded vector for all $i$, we have
\begin{equation*}
\begin{split}
& \EE{\frac{N_i}{n} \Bigl\Vert\mathbb{E}_n[U]^{(i)}\Bigr\Vert^4 \Bigl| N_i} \\
= & \frac{N_i}{n} \frac{N_i(N_i - 1)(N_i - 2)(N_i - 3)I(N_i \ge  4)}{N_i^4}\|\EE{U | \underline{Z} \in B_i}\|^4 + \frac{N_i}{n} \frac{1}{N_i^4} O(N_i^3)\\
= & \frac{N_i I(N_i \ge  4)}{n}\|\EE{U | \underline{Z} \in B_i}\|^4 +  O\Bigl(\frac{1}{n}\Bigr).  
\end{split}
\end{equation*}
Thus, denoting $\zeta_i = \frac{\int_{B_i}F(\underline{Z}) d P_{\underline{Z}}}{\mu(B_i)}$,
\begin{equation*}
\begin{split}
\EE{\frac{N_i}{n} \Bigl\Vert\mathbb{E}_n[U]^{(i)}\Bigr\Vert^4 } & 
= \EE{\EE{\frac{N_i}{n} \Bigl\Vert\mathbb{E}_n[U]^{(i)}\Bigr\Vert^4 \Bigl| N_i}} \\
& = \frac{\EE{N_iI(N_i \ge  4)}}{n}\|\EE{U | \underline{Z} \in B_i}\|^4 + O\Bigl(\frac{1}{n}\Bigr) 
= \frac{n\mu(B_i)}{n} \| \zeta_i\|^4 + O\Bigl(\frac{1}{n}\Bigr).
\end{split}
\end{equation*}
By Condition 
% \ref{ass_1} 
3.1
and the mean value theorem, $ \Bigl\| F(z) - \zeta_i \Bigr\| = O(m^{-s})$ for $z \in B_i$. 
Then we have
\begin{equation}
\label{l2_diff}
\begin{split}
& \|F(\underline{z})\|^2 - \|\zeta_i\|^2 
= (F(\underline{z}) - \zeta_i)^{\top}(F(\underline{z}) + \zeta_i)
=  O\Bigl(\frac{1}{m^s}\Bigr),
\end{split}
\end{equation}
and
\begin{equation}
\label{term1}
\begin{split}
& \|F(\underline{z})\|^4 - \|\zeta_i\|^4\
=  (F(\underline{z}) - \zeta_i)^{\top}
(F(\underline{z}) + \zeta_i)
(\|F(\underline{z})\|^2 + \|\zeta_i\|^2)
= O\Bigl(\frac{1}{m^s}\Bigr).
\end{split}
\end{equation}
 Therefore, for some constant $\kappa_7$,
\begin{equation*}
\Biggl|\int_{B_i}\|F(\underline{Z})\|^4 dP_{\underline{Z}} - \EE{\frac{N_i}{n} \Bigl\Vert\mathbb{E}_n[U]^{(i)}\Bigr\Vert^4 }\Biggr| \le  \kappa_7 \Bigl(\frac{\mu(B_i)}{m^s} + \frac{1}{n}\Bigr).
\end{equation*}
Similarly, for $\mathbb{E}_n[U]^{(i) \top}\textnormal{Cov}_n[U]^{(i)} \mathbb{E}_n[U]^{(i)}$, using the fact that 
$$\EE{U^{(j_1) \top}U^{(j_2)}U^{(j_2 ) \top}U^{(j_3)}} = \EE{U | \underline{Z} \in B_i}^{\top} \EE{U U^{\top}| \underline{Z} \in B_i} \EE{U|\underline{Z} \in B_i}$$
for distinct values $j_1, j_2, j_3 \in \mathcal{I}_{m, n, i}$,
we find
\begin{equation*}
\begin{split}
& \EE{\frac{N_i}{n}  \mathbb{E}_n[U]^{(i) \top}\textnormal{Cov}_n[U]^{(i)} \mathbb{E}_n[U]^{(i)} \Bigl| N_i } + \EE{\frac{N_i}{n} \Bigl\Vert\mathbb{E}_n[U]^{(i)}\Bigr\Vert^4 \Bigl| N_i} \\
&=  \frac{N_i}{n} \frac{N_i(N_i - 1)(N_i - 2)I(N_i \ge  3)}{N_i^3}\EE{U | \underline{Z} \in B_i}^{\top} \EE{U U^{\top}| \underline{Z} \in B_i} \EE{U|\underline{Z} \in B_i} + \frac{N_i}{n} \frac{1}{N_i^3} O(N_i^2) \\
&=  \frac{N_i I(N_i \ge  3)}{n} \EE{U | \underline{Z} \in B_i}^{\top} \EE{U U^{\top}| \underline{Z} \in B_i} \EE{U|\underline{Z} \in B_i} + O\Bigl(\frac{1}{n}\Bigr).
\end{split}
\end{equation*}
Therefore, with $\theta_i = \frac{\int_{B_i}G(\underline{Z}) d P_{\underline{Z}}}{\mu(B_i)}$,
\begin{equation*}
\begin{split}
& \EE{\frac{N_i}{n}  \mathbb{E}_n[U]^{(i) \top}\textnormal{Cov}_n[U]^{(i)} \mathbb{E}_n[U]^{(i)}  } 
= \EE{\EE{\frac{N_i}{n}  \mathbb{E}_n[U]^{(i) \top}\textnormal{Cov}_n[U]^{(i)} \mathbb{E}_n[U]^{(i)} \Bigl| N_i }} \\
&=  \mu(B_i) \zeta_i^{\top} \theta_i \zeta_i - \mu(B_i) \| \zeta_i\|^4 + O\Bigl(\frac{1}{n}\Bigr).
\end{split}
\end{equation*}
Note that 
for all $\underline{z}$,
$F(\underline{z})^{\top}G(\underline{z}) F(\underline{z}) = \sum_{a,b=1}^{k} F(\underline{z})_{a}G(\underline{z})_{a,b} F(\underline{z})_{b}$, where $F(\underline{z})_{a}$ denotes the $a$-th element of the vector $F(\underline{z})$ and $G(\underline{z})_{a,b}$ denotes the element in the $a$-th row and $b$-th column of $G(\underline{z})$, for all $a,b\in[k]$. 
By Condition 
% \ref{ass_1}, 
3.1,
Lemma \ref{holder_smooth_cov}, and the mean value theorem, we have 
\begin{equation*}
| \zeta_i - F(\underline{Z}) |_a = O\Bigl(\frac{1}{m^s}\Bigr), \quad | \theta_i - F(\underline{Z}) |_{a,b} = O\Bigl(\frac{1}{m^s}\Bigr),
\end{equation*}
for any $a,b \in \{1,\dots, k\}. $
Therefore, since $F(\underline{z})$ and $G(\underline{z})$ are bounded as a function of $\underline{z}$, for some constant $\kappa_9$ we have
\begin{equation}
\label{term2}
\begin{split}
& \Biggl| \EE{\frac{N_i}{n}  \mathbb{E}_n[U]^{(i) \top}\textnormal{Cov}_n[U]^{(i)} \mathbb{E}_n[U]^{(i)}} -  \int_{B_i} F(\underline{Z})^{\top}G(\underline{Z}) F(\underline{Z}) dP_{\underline{Z}} + \int_{B_i}\|F(\underline{Z})\|^4 dP_{\underline{Z}}\Biggr| \\
&\le  \kappa_{10} \Bigl(\frac{\mu(B_i)}{m^s} + \frac{1}{n}\Bigr).
\end{split}
\end{equation}

Similarly, we find
\begin{equation*}
\begin{split}
& \EE{\frac{N_i^2}{n^2} 
\Bigl\Vert\mathbb{E}_n[U]^{(i)}\Bigr\Vert^4 | N_i}  \\
= &  \frac{N_i^2}{n^2} \frac{N_i(N_i - 1)(N_i - 2)(N_i - 3)I(N_i \ge  4)}{N_i^4}\|\EE{U | \underline{Z} \in B_i}\|^4 + \frac{N_i^2}{n^2} \frac{1}{N_i^4} O(N_i^3)\\
= & \frac{N_i^2 I(N_i \ge  4)}{n^2}\|\EE{U | \underline{Z} \in B_i}\|^4 +  O\Bigl(\frac{1}{n}\Bigr).  
\end{split}
\end{equation*}
Thus,
\begin{equation*}
\begin{split}
\EE{\frac{N_i^2}{n^2} \Bigl\Vert\mathbb{E}_n[U]^{(i)}\Bigr\Vert^4 } & = \frac{n^2\mu(B_i)^2 + n\mu(B_i)(1 - \mu(B_i))}{n^2}\| \zeta_i\|^4 + O\Bigl(\frac{1}{n^2}\Bigr) + O\Bigl(\frac{1}{n}\Bigr)\\
 & = \mu(B_i)^2 \| \zeta_i\|^4 + O\Bigl(\frac{1}{n}\Bigr).
\end{split}
\end{equation*}
Further for any $i,j\in[\ell_{m,n}]$,  $i\neq j$,
\begin{equation*}
\begin{split}
& \EE{\frac{N_i N_j}{n^2} \Bigl\Vert\mathbb{E}_n[U]^{(i)}\Bigr\Vert^2 \Bigl\Vert\mathbb{E}_n[U]^{(j)}\Bigr\Vert^2 \Bigl| N_i, N_j} \\
&=  \frac{N_i N_j}{n^2} \frac{N_i(N_i - 1)I(N_i \ge  1)}{N_i^2} \frac{N_j(N_j - 1)I(N_j \ge  1)}{N_j^2} \bigl\|\EE{U| \underline{Z} \in B_i}\bigr\|^2 \bigl\|\EE{U| \underline{Z} \in B_j}\bigr\|^2 \\
& + \frac{N_i N_j}{n^2} \frac{N_i N_j^2}{N_i^2 N_j^2} O(1) + \frac{N_i N_j}{n^2} \frac{N_i^2 N_j}{N_i^2 N_j^2} O(1) + \frac{N_i N_j}{n^2} \frac{1}{N_i^2 N_j^2} O(N_i N_j)\\
&=  \frac{N_i N_j}{n^2} \bigl\|\EE{U| \underline{Z} \in B_i}\bigr\|^2 \bigl\|\EE{U| \underline{Z} \in B_j}\bigr\|^2 + \frac{N_i}{n^2} O(1) +  \frac{N_j}{n^2} O(1) + O\Bigl( \frac{1}{n^2} \Bigr).
\end{split}
\end{equation*}
Now, observe that $\EE{N_i N_j} = n(n-1)\mu(B_i)\mu(B_j)$. Therefore,
\begin{equation*}
\begin{split}
& \EE{\frac{N_i N_j}{n^2} \Bigl\Vert\mathbb{E}_n[U]^{(i)}\Bigr\Vert^2 \Bigl\Vert\mathbb{E}_n[U]^{(j)}\Bigr\Vert^2 } \\
&=  \frac{n(n-1)\mu(B_i)\mu(B_j)}{n^2}\| \zeta_i\|^2\|\zeta_j\|^2 + \frac{n\mu(B_i)}{n^2}O(1) + \frac{n\mu(B_j)}{n^2} O(1) + O\Bigl( \frac{1}{n^2} \Bigr) \\
&=  \mu(B_i)\mu(B_j) \| \zeta_i\|^2\|\zeta_j\|^2 + \bigl(\mu(B_i) \mu(B_j) + \mu(B_i) + \mu(B_j) \bigr) O\Bigl( \frac{1}{n} \Bigr) + O\left(\frac{1}{n^2}\right).
\end{split}
\end{equation*}
Using \eqref{l2_diff}, we find 
that for some constant $\kappa_{11}$,
\begin{equation}
\label{term3}
\begin{split}
& \Biggl|  \EE{\frac{N_i^2}{n^2} \Bigl\Vert\mathbb{E}_n[U]^{(i)}\Bigr\Vert^4 } - \Biggr( \int_{B_i}\|F(\underline{Z})\|^2 dP_{\underline{Z}} \Biggr)^2\Biggr|  \le   \kappa_{11}\Bigl(\frac{\mu(B_i)^2}{m^s} +\frac{1}{n} \Bigr) \\
& \Biggl|  \EE{\frac{N_i N_j}{n^2} \Bigl\Vert\mathbb{E}_n[U]^{(i)}\Bigr\Vert^2 \Bigl\Vert\mathbb{E}_n[U]^{(j)}\Bigr\Vert^2} - \Biggr( \int_{B_i}\|F(\underline{Z})\|^2 dP_{\underline{Z}} \Biggr) \Biggr( \int_{B_j}\|F(\underline{Z})\|^2 dP_{\underline{Z}} \Biggr)\Biggr| \\
&\le   \kappa_{11}\Bigl(\frac{\mu(B_i)\mu(B_j)}{m^s} + \frac{\mu(B_i)\mu(B_j) + \mu(B_i) + \mu(B_j)}{n} +\frac{1}{n^2} \Bigr),
\end{split}
\end{equation}
Combining \eqref{term1}, \eqref{term2} and \eqref{term3}, we have  for some constant $\kappa_{12}$ that
\begin{equation}
\label{sigma1_mean_diff}
\Bigl| \EE{\hat{\sigma}_1^2} - \sigma_1^2 \Bigr| \le  \kappa_{12} \Bigl( \frac{1}{m^s} + \frac{\ell_{m,n}}{n} + \frac{\ell_{m,n}^2}{n^2}\Bigr).
\end{equation}
Next, we will show $\Var{\hat{\sigma}_1^2} \rightarrow 0$ as $n \rightarrow \infty$.
Let $\vec{N} = (N_1, \dots, N_{\ell_{m,n}})$ be 
the counts of datapoints in each bin. 
Since $U^{(i)}$ is a bounded vector for each $i\in[\ell_{m,n}]$,
we have that $\mathbb{E}_n[U]^{(i)}$ and $\textnormal{Cov}_n[U]^{(i)} $ are also bounded vectors, and therefore
\begin{equation*}
\EE{\frac{N_i}{n} \Bigl\Vert\mathbb{E}_n[U]^{(i)}\Bigr\Vert^4 \Bigl| \vec{N}} \le  \kappa_{13}\frac{N_i}{n}, \quad \Var{\frac{N_i}{n} \Bigl\Vert\mathbb{E}_n[U]^{(i)}\Bigr\Vert^4 \Bigl| \vec{N}} \le  \kappa_{13}\frac{N_i^2}{n^2},
\end{equation*}
for some constant $\kappa_{13}$. Note that for $i\neq j $, $\Bigl\Vert\mathbb{E}_n[U]^{(i)}\Bigr\Vert^4$ and $\Bigl\Vert\mathbb{E}_n[U]^{(j)}\Bigr\Vert^4$ are independent conditioned on $\vec{N}$.
Hence, we have
\begin{equation*}
\EE{\sum_{i=1}^{\ell_{m,n}}\frac{N_i}{n} \Bigl\Vert\mathbb{E}_n[U]^{(i)}\Bigr\Vert^4 \Bigl| \vec{N}} \le  \kappa_{13}\sum_{i=1}^{\ell_{m,n}}\frac{N_i}{n}
\quad
\textnormal{and}
\quad \Var{\sum_{i=1}^{\ell_{m,n}}\frac{N_i}{n} \Bigl\Vert\mathbb{E}_n[U]^{(i)}\Bigr\Vert^4 \Bigl| \vec{N}} \le  \kappa_{13}\sum_{i=1}^{\ell_{m,n}}\frac{N_i^2}{n^2}.
\end{equation*}
By the law of total variance, we have
\begin{equation}
\label{var_hat_sigma_1}
\begin{split}
& \Var{\sum_{i=1}^{\ell_{m,n}}\frac{N_i}{n} \Bigl\Vert\mathbb{E}_n[U]^{(i)}\Bigr\Vert^4 }\\
= & \Var{\EE{\sum_{i=1}^{\ell_{m,n}}\frac{N_i}{n} \Bigl\Vert\mathbb{E}_n[U]^{(i)}\Bigr\Vert^4 \Bigl| \vec{N}}} + \EE{\Var{\sum_{i=1}^{\ell_{m,n}}\frac{N_i}{n} \Bigl\Vert\mathbb{E}_n[U]^{(i)}\Bigr\Vert^4 \Bigl| \vec{N}}} \\
\le &  \kappa_{13}^2 \Biggl(\sum_{i=1}^{\ell_{m,n}}\sqrt{\Var{\frac{N_i}{n}}} \Biggr)^2 + \kappa_{13}^2 \sum_{i=1}^{\ell_{m,n}} \EE{\frac{N_i^2}{n^2}}\\
\le &  \kappa_{13}^2 \Bigl(\frac{1}{n \min \sqrt{\mu(B_i)}} \sum_{i=1}^{\ell_{m,n}} \mu(B_i) \Bigr)^2 + \kappa_{13}^2 \sum_{i=1}^{\ell_{m,n}} \frac{n^2\mu(B_i)^2 + n\mu(B_i) (1 - \mu(B_i))}{n^2} \\
\le &  \kappa_{14} \Bigl(\frac{1}{n^2 \min_{i\in[\ell_{m,n}]} \mu(B_i)} + \max \mu(B_i) + \frac{1}{n} \Bigr) 
 \rightarrow 0.
\end{split}
\end{equation}
Similarly, we find 
\begin{equation}
\label{var_hat_sigma_2}
\begin{split}
\Var{\sum_{i=1}^{\ell_{m,n}}\frac{N_i}{n} \Bigl\Vert\mathbb{E}_n[U]^{(i)}\Bigr\Vert^2 } \rightarrow 0, \quad \Var{\sum_{i=1}^{\ell_{m,n}}\frac{N_i}{n} \mathbb{E}_n[U]^{(i) \top}\textnormal{Cov}_n[U]^{(i)} \mathbb{E}_n[U]^{(i)} } \rightarrow 0.
\end{split}
\end{equation}
Note that for a bounded random variable $A$, such that  $|A|\le K$ almost surely
for some constant $c$, we have 
\begin{equation}
\label{var_hat_sigma_3}
\Var{A^2} =  \EE{\bigl(A^2 - \EE{A^2}\bigr)\bigl(A^2 + \EE{A^2}\bigr)} \le  2K^2 \EE{A^2 - \EE{A^2}} = 2K^2 \Var{A}.
\end{equation}
Therefore
$\Var{\Biggl( \sum_{i=1}^{\ell_{m,n}}N_i/n\cdot \Bigl\Vert\mathbb{E}_n[U]^{(i)}\Bigr\Vert^2 \Biggr)^2 } \rightarrow 0$.
In summary, we have 
$\Var{\hat{\sigma}_1^2} 
\rightarrow 0$.
 Then by Chebshev's inequality, 
$\hat{\sigma}_1^2 \rightarrow_p \sigma_1^2$,
finishing the proof.

\subsection{Choice of the Number of Bins}
\label{bias_var}
\subsubsection{Bias of \texorpdfstring{$T_{m,n}$}{T}}
Theorem 
% \ref{clt_calibrated_model} 
3.3
and 
% \ref{clt_miscalibrated_model} 
3.4
show that
when scaled properly, 
$T_{m,n}$ has an limiting normal distribution.
However, for a mis-calibrated model, the expectation of $T_{m,n}$ is not exactly $\mathrm{ECE}_{1:k}^2$. 
In this section, we study the difference between $\EE{T_{m,n}}$ and $\mathrm{ECE}_{1:k}^2$.
Following the calculation for \eqref{emn},
except using $|\mathcal{I}_{m,n, i}| \sim \text{Binomial}(n, \mu(B_i))$ instead of a Poisson distribution, we have 
with $\zeta_i = \EE{U | Z_{(1:k)} \in B_i}  = \frac{\int_{B_i}\EE{U|Z_{(1:k)}} d P_{Z_{(1:k)}}}{\mu(B_i)}$,
\begin{equation}
\label{T_mn_mean}
\begin{split}
\EE{T_{m,n}} & = \frac{1}{n}\sum_{i=1}^{\ell_{m,n}} \EE{|\mathcal{I}_{m,n, i}| I(|\mathcal{I}_{m,n, i}| \ge  2)} \bigl\|\EE{U | Z_{(1:k)} \in B_i}\bigr\|^2\\
& = \frac{1}{n}\sum_{i=1}^{\ell_{m,n}} \Bigl(n\mu(B_i) - n\mu(B_i)\bigl(1 - \mu(B_i)\bigr)^{n-1} \Bigr) \| \zeta_i \|^2.
\end{split}
\end{equation}
By the definition of $\mathrm{ECE}_{1:k}^2$, we have
\begin{equation*}
\mathrm{ECE}_{1:k}^2 = \int \|\EE{U| Z_{(1:k)}}\|^2 dP_{Z_{(1:k)}} = \sum_{i=1}^{\ell_{m,n}} \int_{B_i} \|\EE{U| Z_{(1:k)}}\|^2 dP_{Z_{(1:k)}} .
\end{equation*}
Thus
\begin{equation}
\label{T_mn_bias}
\begin{split}
  \EE{T_{m,n}} - \mathrm{ECE}_{1:k}^2 
= &  -\frac{1}{n}\sum_{i=1}^{\ell_{m,n}}  n\mu(B_i)\bigl(1 - \mu(B_i)\bigr)^{n-1}  \| \zeta_i \|^2 \\
& + \sum_{i=1}^{\ell_{m,n}} \mu(B_i)\| \zeta_i \|^2- 
\int_{B_i} \|\EE{U| Z_{(1:k)}}\|^2 dP_{Z_{(1:k)} }.
\end{split}
\end{equation}
Since $\mu(B_i)\in[0,1]$ for all $i\in[\ell_{m,n}]$ 
and $n\ge 1$,
we have $(1 - \mu(B_i))^{n-1}  \le  \exp\bigl\{-(n-1)\mu(B_i)\bigr\} \le  \exp\bigl\{-(n-1)\min_{i\in[\ell_{m,n}]} \mu(B_i)\bigr\}$.
Further, $|\EE{U|Z_{(1:k)}}|_j \le  1$ for $j = 1,\dots, k$. Thus, we have 
\begin{equation*}
 \bigl| \EE{T_{m,n}} - \mathrm{ECE}_{1:k}^2  \bigr|
 \le   k \exp\bigl\{-(n-1)\min_{i\in[\ell_{m,n}]} \mu(B_i)\bigr\} + \sum_{i=1}^{\ell_{m,n}} \bigintsss_{B_i} \Bigl\|\EE{U|Z_{(1:k)}} -  \zeta_i \Bigr\|^2 dP_{Z_{(1:k)}}.
\end{equation*}
By Condition 
% \ref{ass_1}, 
3.1,
we have 
\begin{equation*}
\begin{split}
|\EE{U|Z_{(1:k)}=z_{(1:k)}}_j -  [\zeta_i]_j| & 
\le  \max_{z_{(1:k)} \in B_i} \EE{U|Z_{(1:k)}=z_{(1:k)}}_j - \min_{z_{(1:k)} \in B_i} \EE{U|Z_{(1:k)}=z_{(1:k)}}_j \\
& \le  L \sup_{z^{(1)}_{(1:k)} \in B_i, z^{(2)}_{(1:k)} \in B_i} \bigl\|z^{(1)}_{(1:k)}  - z^{(2)}_{(1:k)} \bigr\|^s 
\le  \frac{L\kappa_{15}}{m^s} ,
\end{split}
\end{equation*}
for some $\kappa_{15}$ that depends only on $k$. Therefore,
\begin{equation}
\label{estimator_bias}
\begin{split}
\bigl| \EE{T_{m,n}} - \mathrm{ECE}_{1:k}^2  \bigr|
 &\le  k\exp\bigl\{-(n-1)\min_{i\in[\ell_{m,n}]} \mu(B_i)\bigr\} + \sum_{i=1}^{\ell_{m,n}} \frac{k L^2 \kappa_{15}^2}{m^{2s}}{\mu(B_i)} \\
 & \le  \kappa_{16}\Bigl(\exp\Bigl\{-\frac{c n}{m^{k}}\Bigr\} + \frac{1}{m^{2s}}\Bigr),  
\end{split}
\end{equation}
for some $\kappa_{16}$ that depends only on $k$ 
and $c$ that depends only on the distribution of $Z_{(1:k)}$. 
Combining \eqref{estimator_bias} with Theorem 
% \ref{clt_calibrated_model} 
3.3
and 
% \ref{clt_miscalibrated_model}, 
3.4,
we find that $T_{m,n} \rightarrow_{p} \mathrm{ECE}_{1:k}^2$.

\begin{rmk}
 The order $O({m^{-2s}})$ 
 in the upper bound of the bias is optimal. 
 Consider a setting case where $K = 2$, $k = 1$, $P_Z$ is uniform on $\Delta_1$, and $\EE{Y_{r_1}|Z_{(1)}} = 0$. Then $Z_{(1)}$ follows a uniform distribution on $[1/2, 1]$ and $\EE{U|Z_{(1)}} = -Z_{(1)}$. 
 We partition the support of $Z_{(1)}$ 
 into $B_i = [1/2 + (i-1)/(2m), 1/2 + i/(2m)]$
 for $i = 1,\dots, \ell_{m,n} = m$.
 Then we have for all $i\in[m]$ that
\begin{equation*}
\frac{\int_{B_i}\EE{U|Z_{(1:k)}} d P_Z}{\mu(B_i)} = \frac{\int_{\frac{1}{2} + \frac{i-1}{2m}}^{\frac{1}{2} + \frac{i}{2m}} (-2x) dx}{\frac{2}{2m}} = -\frac{1}{2} - \frac{2i - 1}{4m}.
\end{equation*}
Therefore,
\begin{equation*}
\sum_{i=1}^{\ell_{m,n}} \bigintsss_{B_i} \Biggl\|\EE{U|Z_{(1:k)}} -  \frac{\int_{B_i}\EE{U|Z_{(1:k)}} d P_Z}{\mu(B_i)} \Biggr\|^2 dP_Z = m\int^{\frac{1}{4m}}_{-\frac{1}{4m}} 2x^2 dx = \frac{1}{48m^2}.
\end{equation*}   
This shows that the order of the bias bound is optimal.
\end{rmk}
\begin{rmk}
\label{small_bias}
Compared to the debiased estimator proposed by \cite{lee2023t}, 
\begin{equation*}
    T_{m,n}^{\rm{T-Cal}} = \frac1n\sum_{\substack{1 \le  i \le  \ell_{m,n},\, |\mathcal{I}_{m,n, i}| \ge  1}} \frac{1}{|\mathcal{I}_{m, n, i}|} \sum_{a\neq b \in \mathcal{I}_{m, n, i}} U^{(a)\top}U^{(b)},
\end{equation*}
our estimator $T_{m,n}$ replaces the factor $1/(n|\mathcal{I}_{m, n, i}|)$ 
by $1/[n\bigl(|\mathcal{I}_{m, n, i}| - 1\bigr)]$ for each bin. Following the calculation for \eqref{T_mn_mean} and \eqref{T_mn_bias}, we have 
\begin{equation*}
\begin{split}
\EE{T_{m,n}^{\rm{T-Cal}}} & = \frac{1}{n}\sum_{i=1}^{l_{m,n}} \EE{(|\mathcal{I}_{m,n, i}| - 1)I(|\mathcal{I}_{m,n, i}| \geq 1)} \bigl\|\EE{U | Z_{(1:k)} \in B_i}\bigr\|^2\\
& = \frac{1}{n}\sum_{i=1}^{l_{m,n}} \bigl(n\mu(B_i) - 1 + (1 - \mu(B_i))^n \bigr) \| \zeta_i \|^2,
\end{split}
\end{equation*}
and 
\begin{equation*}
\begin{split}
  \EE{T_{m,n}^{\rm{T-Cal}}} - \mathrm{ECE}_{1:k}^2 
= & -\frac{1}{n}\sum_{i=1}^{l_{m,n}} \bigl( 1 - (1 - \mu(B_i))^n \bigr) \| \zeta_i \|^2  \\
& + \sum_{i=1}^{\ell_{m,n}} \mu(B_i)\| \zeta_i \|^2- 
\int_{B_i} \|\EE{U| Z_{(1:k)}}\|^2 dP_{Z_{(1:k)} }.
\end{split}
\end{equation*}
Compared to the bias of $T_{m,n}$ \eqref{T_mn_bias}, the factor $n\mu(B_i)\bigl(1 - \mu(B_i)\bigr)^{n-1}$ in \eqref{T_mn_bias} is replaced by $1 - (1 - \mu(B_i))^n = O(1)$. The bias of $T_{m,n}^{\rm{T-Cal}}$ is larger.

\end{rmk}

\subsubsection{Variance of \texorpdfstring{$T_{m,n}$}{T}}

Let $N_i = |\mathcal{I}_{m, n, i}| \sim \text{Binomial}(n, \mu(B_i))$, $\vec{N} = (N_1, \dots, N_{\ell_{m,n}})$,  $\omega_{1,i} =  \Var{{U^{(1) \top}}U^{(2)} | Z^{(1)}_{(1:k)}, Z^{(2)}_{(1:k)} \in B_i}$, $\omega_{2,i} = \Cov{{U^{(1) \top}}U^{(2)}, {U^{(1) \top}}U^{(3)} |Z^{(1)}_{(1:k)}, Z^{(2)}_{(1:k)}, Z^{(3)}_{(1:k)} \in B_i }$, $\omega_{3,i} = \| \EE{U | Z_{(1:k)} \in B_i}\|^2$, Following the calculation for \eqref{conditional_mean} \eqref{conditional_var}, we have 
\begin{equation*}
\begin{split}
    & \Var{T_{m,n}} = \EE{\Var{T_{m,n}| \vec{N}}} + \Var{\EE{T_{m,n}| \vec{N}}} \\
    & = \sum_{i=1}^{\ell_{m,n}} \EE{\frac{2 N_i I(N_i\ge  2)}{n^2(N_i-1)}} \omega_{1,i} + \sum_{i=1}^{\ell_{m,n}} \EE{\frac{4N_i(N_i-2)I(N_i\ge  2)}{n^2(N_i-1)} } \omega_{2,i} + \Var{\sum_{i=1}^{\ell_{m,n}} \frac{N_i I(N_i\ge  2)}{n} \omega_{3,i}}. \\
\end{split}
\end{equation*}
 Using the fact that $\EE{N_i} = n\mu(B_i)$, $\Var{N_i} = n\mu(B_i) - n\mu(B_i)^2 $, $\Cov{N_i, N_j} = -n\mu(B_i)\mu(B_j)$, $\ell_{m,n} = O(m^{k})$ and
 $P(N_i \le  1) = \exp\{-n \log(\mu(B_i)^{-1})\} + n(1 - \mu(B_i))\exp\{-(n-1)\log(\mu(B_i)^{-1})\}= o(n^{-2})$, we have 
 \begin{equation}
 \Var{T_{m,n}} \le  \kappa_{17} \left(\frac{l_{m,n}}{n^2} + \sum_{i=1}^{\ell_{m,n}}\frac{n\mu(B_i)}{n^2} + \sum_{i=1}^{\ell_{m,n}} \frac{n\mu(B_i)\mu(B_j)}{n^2}  + 
 \frac{1}{n^2} \frac{l_{m,n}^2}{n^2}\right) \le  \kappa_{18}\left(\frac{m^k}{n^2} + \frac{1}{n}  \right),
 \end{equation}
for some constants $\kappa_{17}, \kappa_{18}$.

\subsection{Proof of Theorem 
% \ref{infer_mean}
3.7
}
\label{proof_KI_mean}
If the model is top-1-to-k calibrated, $\EE{T_{m,n}} = \mathrm{ECE}_{1:k}^2 = 0$. Then by the definition of $C_{m,n}$ and Theorem 
% \ref{clt_calibrated_model}, 
3.3, 
we have 
\begin{equation*}
P(\EE{T_{m,n}} \in C_{m,n}) = P\biggl(T_{m,n}^{+} <   z_{\alpha} \sigma_0/(n\sqrt{w})\biggr) = P\biggl(\frac{n\sqrt{w} T_{m,n}}{\sigma_0} <   z_{\alpha} \biggr) \rightarrow 1-\alpha. 
\end{equation*}
If $\mathrm{ECE}_{1:k}^2 \neq 0$, let $A_1, A_2, A_3$ be the events defined as
\begin{equation*}
\begin{split}
A_1&=\{ \EE{T_{m,n}} \in  [T_{m,n} - z_{\alpha/2} \hat\sigma_1/\sqrt{n} ,\quad T_{m,n} + z_{\alpha/2} \hat\sigma_1/\sqrt{n}]\}, \\
A_2&=\{ \EE{T_{m,n}} \in [T_{m,n} / 2, \quad T_{m,n} + z_{\alpha/2} \hat\sigma_1/\sqrt{n}]\}, \\
A_3&=\{ \EE{T_{m,n}} \in [\max\{0, T_{m,n}^{+} - z_{\alpha}\hat\sigma_1/\sqrt{n}\}, \quad T_{m,n}^{+} + z_{\alpha/2} \hat\sigma_1/\sqrt{n}] \setminus \{0\} \}.
\end{split}
\end{equation*}
Further, let $D_1, D_2, D_3$ be the events defined as
\begin{equation*}
\begin{split}
D_1 &=\{T_{m,n} / 2 \le  T_{m,n} - z_{\alpha/2} \hat\sigma_1/\sqrt{n}\},\quad
D_2 =\{T_{m,n} - z_{\alpha/2} \hat\sigma_1/\sqrt{n} < T_{m,n} / 2 \le  T_{m,n} - z_{\alpha} \hat\sigma_1/\sqrt{n}\},\\
D_3 &=\{ T_{m,n}^{+} - z_{\alpha} \hat\sigma_1/\sqrt{n} < T_{m,n}^{+} / 2\}.
\end{split}
\end{equation*}
Then 
\begin{equation*}
P(\EE{T_{m,n}} \in C_{m,n}) = P(A_1\cap D_1) + P(A_2\cap D_2) + P(A_3\cap D_3).
\end{equation*}
Since $T_{m,n} \rightarrow_{p} \mathrm{ECE}_{1:k}^2 > 0$, we have $P(D_1) \rightarrow 1, P(D_2) \rightarrow 0$, and $ P(D_3) \rightarrow 0$. Therefore
\begin{equation*}
0 \le  P(A_1 \cap D_1^c) \le  P(D_1^c) \rightarrow 0, \, 0 \le  P(A_2 \cap D_2) \le  P(D_2) \rightarrow 0, 
\mathrm{and}
\, 0 \le  P(A_3 \cap D_3) \le  P(D_3) \rightarrow 0.
\end{equation*}
By Theorem 
% \ref{clt_miscalibrated_model}, 
3.4, 
we have 
$P(A_1) = 
P\biggl(-z_{\alpha / 2} \le  
{\sqrt{n}(T_{m,n} - \EE{T_{m,n}})}/{\hat{\sigma}_1} \le  z_{\alpha / 2}\biggr) \rightarrow 1-\alpha$.
Then 
$P(A_1 \cap D_1) = P(A_1) - P(A_1 \cap D_1^c) \rightarrow 1-\alpha$.
Therefore 
$P(\EE{T_{m,n}} \in C_{m,n}) \rightarrow 1-\alpha$, finishing the proof.

\subsection{Proof of Theorem 
% \ref{infer_ece}
3.9
}
\label{proof_KI_ece}
By Condition 
% \ref{ass_3} 
3.8
and the bound \eqref{estimator_bias} on the bias, we have 
${\sqrt{n}\bigl| \EE{T_{m,n}} - \mathrm{ECE}_{1:k}^2  \bigr|}/{\hat{\sigma}_1} \rightarrow 0$.
Therefore, 
by Theorem 
% \ref{clt_miscalibrated_model}, 
3.4,
$P\biggl(-z_{\alpha / 2} \le  {\sqrt{n}(T_{m,n} - \mathrm{ECE}_{1:k}^2 )}/{\hat{\sigma}_1} \le  z_{\alpha / 2}\biggr) \rightarrow 1-\alpha$.
Then following the same argument as in Section \ref{proof_KI_mean}, we find
$P(\mathrm{ECE}_{1:k}^2 \in C_{m,n}) \rightarrow 1-\alpha$, finishing the proof.

\subsection{{Finite Sample Analysis}}
\label{fini}

Theorem 3.9 ensures the asymptotic exact coverage rate of our proposed interval. In this section, we further provide an analysis of the finite sample coverage rate of the confidence interval based on the Poissonized estimator $\tilde{T}(\tilde{N})$ \eqref{poisson_estimator}. 

\begin{theorem}[Finite sample coverage bounds]
\label{thm:finite_sample}
Let $\tilde{T}(\tilde{N})$ be the Poissonized estimator defined in \eqref{poisson_estimator}, $\tilde{k} = \min(k, K-1)$ and assume
 that $Z_{(1:k)}$ has a strictly positive density on its support. 
\begin{enumerate}
    \item[\textbf{(i)}] \textbf{Calibrated model.} Suppose $\mathrm{ECE}_{1:k}^2 = 0$ and that Condition 3.1 holds. 
    Then, the coverage of our confidence interval satisfies
    \begin{equation}\label{fsc}
    \left|
    P(\mathrm{ECE}_{1:k}^2 \in C_{m}) 
    - (1-\alpha)\right| \leq \kappa_{21} \left(\frac{1}{m^{\min(\tilde{k}/2,1)}} + \frac{m^{\tilde{k}}}{n} \right),
    \end{equation}
    where $\kappa_{21}$ is a constant and $\sigma_0^2$ is as defined in Theorem 3.3 in the main text.

    \item[\textbf{(ii)}] \textbf{Miscalibrated model.} Suppose $\mathrm{ECE}_{1:k}^2 > 0$  and that Condition 3.1 holds.
    Then, for any $\delta \ge 0$, the coverage error of our confidence interval
    satisfies 
    with probability at least $1 - \delta$ that
\begin{equation}\label{fsu}
    \left|
    P(\mathrm{ECE}_{1:k}^2 \in C_{m}) 
    - (1-\alpha)\right| 
    \leq \kappa_{27} \left(\frac{1}{\sqrt{\delta}} \left(\frac{m^{\tilde{k}/2}}{n} +\frac{1}{m^{\tilde{k}/2}} + \frac{1}{\sqrt{n}}\right) + \frac{m^{2\tilde{k}}}{n^2} +\frac{m^{\tilde{k}}}{n} + \frac{\sqrt{n}}{m^{s}} \right),
\end{equation}
    where $\kappa_{27}$ is a constant and $s$ is the H\"older smoothness parameter from Condition 3.1.
\end{enumerate}
\end{theorem}

Notably, 
the upper bound from \eqref{fsc} has a different form than the MSE upper bound
we discussed in Remark 2.1.
which was of the form
$$\EE{(T_{m} - \mathrm{ECE}_{1:k}^2)^2}=
O\left( m^{\min(k, K-1)} n^{-2}
+ n^{-1} + m^{-2s} \right).$$ 
This makes sense, as the two bounds are for qualitatively different convergence rates.
Moreover, in \eqref{fsu}, we can choose $\delta = 1/n^a$ for some small $a>0$ to obtain a high-probability bound that converges to zero with an appropriate choice of $m$.

\subsubsection{Calibrated Model}
We first consider the case of a calibrated model with $\mathrm{ECE}_{1:k}^2 = 0$. Our goal is to derive an upper bound on the difference between the coverage rate $P\bigl(\tilde T(\tilde{N}) \leq z_{\alpha} \sigma_0/(n\sqrt{w})\bigr)$ and the target level $\alpha$.
    
Recall that we let the number of data points be a Poisson random variable $\tilde{N} \sim \text{Poisson}(n)$ and let $w = \text{Vol}(B_i) $ be the volume of any partition element. 

Defining $\mathcal{I}(\tilde{N})_{i} = \left\{j: Z^{(j)}_{(1:k)} \in B_i, 1\le  j \le  \tilde{N} \right\}$, 
we then have $\bigl|\mathcal{I}(\tilde{N})_{i}\bigr| \sim \text{Poisson}(n\mu(B_i)),$
and $|\mathcal{I}(\tilde{N})_{i}|, i=1,\dots, \ell_{m}$ are mutually independent. The Poissonized estimator is given by 
\begin{equation}
\label{poisson_estimator}
     \tilde T(\tilde{N}) = \sum_{\substack{1 \le  i \le  \ell_{m},\, |\mathcal{I}(\tilde{N})_{i}| \ge  2}} \frac{1}{n\bigl(|\mathcal{I}(\tilde{N})_{i}| - 1\bigr)}  \Biggl[  \sum_{j_1\neq j_2 \in \mathcal{I}(\tilde{N})_{i}} U^{(j_1)\top}U^{(j_2)} \Biggr],
    \end{equation}
which is a sum of independent random variables. 

% We follow the notation in \Cref{sec:clt_calibrated_model}, 
For all $i\in[\ell_{m}]$, let\footnote{Set $M(B_i) = 0$ if $|\mathcal{I}(\tilde{N})_{i}| \leq 2$.}
\begin{equation*}
M(B_i) = \sum_{j_1\neq j_2 \in \mathcal{I}(\tilde{N})_{i}} U^{(j_1)\top}U^{(j_2)}/(|\mathcal{I}(\tilde{N})_{i}| - 1\bigr),
\end{equation*}
and define the rescaled estimator 
\begin{equation*}
T(\tilde{N}) = n\sqrt{w} \tilde T(\tilde{N}) = \sum_{\substack{1 \le  i \le  \ell_{m}
% ,\, |\mathcal{I}(\tilde{N})_{i}| \ge  2
}
}  \sqrt{w} M(B_i).
\end{equation*}
Following calculations similar to those in \Cref{proof_clt_calibrated_model}, and using \eqref{partition_variance_calibrated} and \eqref{poisson_order_-1},
$\Var{T(\tilde{N})} $ equals
\begin{equation*}
\begin{split}
& \sum_{q=1}^{\ell_{m}} w \Var{M(B_q)} 
=  O\left(\frac{w\ell_{m}}{n\min_{i\in[\ell_{m}]}{\mu(B_i)}}\right) + 2w\sum_{q=1}^{\ell_{m}} \sum_{i,j=1}^{k}\left\{\EE{Z_{(i)} Z_{(j)}|Z_{(1:k)} \in B_q}\right\}^2\\
& +2w \sum_{q=1}^{\ell_{m}} 
\left( \sum_{i=1}^{k}\EE{Z_{(i)}|Z_{(1:k)} \in B_q}\biggl(\EE{Z_{(i)}|Z_{(1:k)} \in B_q}-2\EE{Z_{(i)}^2|Z_{(1:k)} \in B_q}\biggr) \right)\\
= & O\left(\frac{w\ell_{m}}{n\min_{i\in[\ell_{m}]}{\mu(B_i)}}\right) +
2\int_{\Delta(K,k)} \bigl(\|Z_{(1:k)}\|_2^2 + 2 \|Z_{(1:k)}\|_3^3 + \|Z_{(1:k)}\|_2^4\bigr) dZ_{(1:k)} + O \left(\frac{w \ell_{m}}{m} \right) \\
= & \sigma_0^2 + O\left(w\ell_{m} \left( \frac{1}{n\min_{i\in[\ell_{m}]}{\mu(B_i)}} + \frac{1}{m}\right)\right).
\end{split}
\end{equation*}

By \eqref{third_moment_calibrated}, we have 
\begin{equation*}
w^{3/2} \sum_{i=1}^{\ell_{m}} \EE{\left|M(B_i) - \EE{M(B_i)} \right|^3} = O\bigl(w^{3/2} \ell_{m}\bigr).
\end{equation*}

Then by the Berry-Esseen theorem (see e.g., Lemma 8.14 of \cite{dasgupta2011probability}), for any $t \in \mathbb{R}$, we have 
\begin{equation*}
% \label{cal_finite_sample_bound}
    \left| P\left(\frac{T(\tilde{N})}{\sqrt{\Var{T(\tilde{N})}}} \leq t \right) - \Phi(t) \right| \leq \kappa_{19} \frac{w^{3/2} \sum_{i=1}^{\ell_{m}} \EE{\left|M(B_i) - \EE{M(B_i)} \right|^3}}{\Var{T(\tilde{N})}^{3/2}},
\end{equation*}
where $\Phi$ is the CDF of the standard normal distribution and $\kappa_{19}$ is a constant. 
Therefore,
\begin{equation*}
\begin{split}
& \left|P\bigl(\tilde T(\tilde{N}) \leq z_{\alpha} \sigma_0/(n\sqrt{w})\bigr) - \alpha \right| = \left|P\bigl( T(\tilde{N}) \leq z_{\alpha} \sigma_0\bigr) - \Phi(z_\alpha)\right| \\
\leq & \left| P\left( \frac{T(\tilde{N})}{\sqrt{\Var{T(\tilde{N})}}} \leq z_{\alpha}\frac{\sigma_0}{\sqrt{\Var{T(\tilde{N})}}}  \right) - \Phi\left(z_{\alpha}\frac{\sigma_0}{\sqrt{\Var{T(\tilde{N})}}} \right) \right| \\
& + \left|\Phi\left(z_{\alpha}\frac{\sigma_0}{\sqrt{\Var{T(\tilde{N})}}} \right) - \Phi(z_{\alpha}) \right| \\
\leq & \kappa_{20} w\ell_{m} \left( \sqrt{w} +  \frac{1}{n\min_{i\in[\ell_{m}]}{\mu(B_i)}} + \frac{1}{m}\right),
\end{split}
\end{equation*}
where $\kappa_{20}$ is some constant. 

If $Z_{(1:k)}$ has a strictly positive density, 
then 
under our binning scheme,
we have 
$w = O(1 / m^{\tilde{k}}), \ell_{m} = O(m^{\tilde{k}})$, 
and
$1 / \min_{i\in[\ell_{m}]}{\mu(B_i)} = O(m^{\tilde{k}})$, where $\tilde{k} = \min(k, K-1)$. Then the bound becomes
\begin{equation*}
    \left|P\bigl(\tilde T(\tilde{N}) \leq z_{\alpha} \sigma_0/(n\sqrt{w})\bigr) - \alpha \right| \leq \kappa_{21} \left(\frac{1}{m^{\tilde{k}/2}} + \frac{m^{\tilde{k}}}{n} + \frac{1}{m}\right),
\end{equation*}
for some constant $\kappa_{21}$.

\subsubsection{Miscalibrated Model}
We next consider the miscalibrated model with $\mathrm{ECE}_{1:k}^2 > 0$. 
Our first goal is to derive an upper bound on the difference between the coverage probability $P\bigl(\mathrm{ECE}_{1:k}^2 \leq \tilde T(\tilde{N}) + z_{\alpha} \sigma_1/\sqrt{n}\bigr)$. 
and the target level $\alpha$. 
Then, 
in 
Section \ref{miscal-est} below, we provide a result for the case of an estimated variance.

Similar to the calibrated case, we define the rescaled estimator:
\begin{equation*}
T(\tilde{N}) = \sqrt{n}\tilde T(\tilde{N}) = \sum_{\substack{1 \le  i \le  \ell_{m},\, |\mathcal{I}(\tilde{N})_{i}| \ge  2}} \frac{1}{\sqrt{n}} M(B_i).
\end{equation*}
While this notation coincides with the differently scaled statistic for the calibrated model, no confusion will arise.
Following calculations similar  to those in 
\Cref{proof_clt_miscalibrated_model}, and using \eqref{partition_variance}, \eqref{poisson_order_0}, \eqref{poisson_order}, as well as
H\"older smoothness (Condition 3.1), we 
find
\begin{equation*}
\begin{split}
&\Var{T(\tilde{N})} 
=  \sum_{i=1}^{\ell_{m}}4\EE{U^{\top}|Z_{(1:k)} \in B_i}\Cov{U|Z_{(1:k)} \in B_i}\EE{U|Z_{(1:k)} \in B_i} \mu(B_i) \\
& + \sum_{i=1}^{\ell_{m}}  \bigl\|\EE{U | Z_{(1:k)} \in B_i}\bigr\|^4 \mu(B_i)  + O\left(\frac{\ell_{m}}{n}\right)
= \sigma_1^2 + O\left(\frac{\ell_{m}}{n}\right) + O\left(\frac{1}{m^s}\right) .
\end{split}
\end{equation*}

By \eqref{emn}, we have 
that $\EE{T(\tilde{N})}$ equals
\begin{equation*}
\sqrt{n}\sum_{i=1}^{\ell_{m}} \bigl\|\EE{U | Z_{(1:k)} \in B_i}\bigr\|^2 \mu(B_i) (1 - e^{-n\mu(B_i)}) = \sqrt{n}\mathrm{ECE}_{1:k}^2 + O\left(\frac{\sqrt{n}}{m^s} + \sqrt{n}e^{-n\min_{i\in[\ell_{m}]}\mu(B_i)}\right).
\end{equation*}

By \eqref{third_moment}, we have
\begin{equation*}
 \EE{\sum_{\substack{1 \le  i \le  \ell_{m},\, |\mathcal{I}(\tilde{N})_{i}| \ge  2}} \frac{1}{n^{\frac{3}{2}}} \bigl|M(B_i)- \EE{M(B_i)}\bigr|^{3}}  = O \left(\frac{1}{n\sqrt{\min_{i\in[\ell_{m}]}{\mu(B_i)}}} + \frac{ \ell_{m}}{n^{\frac{3}{2}}} \right).
\end{equation*}

By the Berry-Esseen theorem, for any $t \in \mathbb{R}$, we have 
\begin{equation*}
    \left| P\left(\frac{T(\tilde{N}) - \EE{T(\tilde{N})}}{\sqrt{\Var{T(\tilde{N})}}} \leq t \right) - \Phi(t) \right| \leq \kappa_{22} \frac{ \sum_{q=1}^{\ell_{m}} \EE{\left|M(B_q) - \EE{M(B_q)} \right|^3}}{n^{3/2}\Var{T(\tilde{N})}^{3/2}},
\end{equation*}
where $\kappa_{22}$ is some constant. Let
\begin{equation*}
t = \frac{\sigma_1}{\sqrt{\Var{T(\tilde{N})}}} \left(-z_{\alpha} + \frac{\sqrt{n}\mathrm{ECE}_{1:k}^2 - \EE{T(\tilde{N})}}{\sigma_1}\right)  = -z_{\alpha} + O\left(\frac{\ell_{m}}{n} +  \frac{\sqrt{n}}{m^s} + \sqrt{n}e^{-n\min\mu(B_i)}\right).
\end{equation*}
Then, we have
\begin{equation*}
\begin{split}
& \left| P\bigl(\mathrm{ECE}_{1:k}^2 \leq \tilde T(\tilde{N}) + z_{\alpha} \sigma_1/\sqrt{n}\bigr) - \alpha  \right| = \left| P\left(\frac{T(\tilde{N}) - \EE{T(\tilde{N})}}{\sqrt{\Var{T(\tilde{N})}} } \geq  t\right) - \Phi(z_{\alpha})\right| \\
 \leq & \left| P\left(\frac{T(\tilde{N}) - \EE{T(\tilde{N})}}{\sqrt{\Var{T(\tilde{N})}} } \geq  t\right) - (1 - \Phi(t)) \right| + \left| 1 - \Phi(t) - (1 - \Phi(-z_{\alpha}))\right|, 
\end{split}
\end{equation*}
which is further upper bounded by
 \begin{equation*}
\begin{split}
& \kappa_{23}  \left(\frac{1}{n\sqrt{\min{\mu(B_i)}}} + \frac{\ell_{m}}{n}  + \frac{\sqrt{n}}{m^s} + \sqrt{n}e^{-n\min\mu(B_i)} \right),
\end{split}
\end{equation*}
where $\kappa_{23}$ is some constant. 

If $Z_{(1:k)}$ has a strictly positive density, under our binning scheme, $\ell_{m} = O(m^{\tilde{k}})$
and 
$\min_{i\in[\ell_{m}]}{\mu(B_i)} = O(m^{-\tilde{k}})$. Then we conclude
\begin{equation*}
    \left| P\bigl(\mathrm{ECE}_{1:k}^2 \leq \tilde T(\tilde{N}) + z_{\alpha} \sigma_1/\sqrt{n}\bigr) - \alpha  \right| \leq \kappa_{24} \left(\frac{m^{\tilde{k}}}{n} + \frac{\sqrt{n}}{m^{s}} \right).
\end{equation*}

\subsubsection{Miscalibrated Model with Estimated Variance}
\label{miscal-est}

In practice, $\sigma_1^2$ is typically unknown
and we construct the confidence interval using the estimated variance $\hat{\sigma}_1^2$. We next study the difference between the coverage probability $P\bigl(\mathrm{ECE}_{1:k}^2 \leq \tilde T(\tilde{N}) + z_{\alpha} \hat{\sigma}_1/\sqrt{n}\bigr)$ and the target level $\alpha$.
By \eqref{sigma1_mean_diff}, we have
\[
\EE{\hat{\sigma}_1^2} = \sigma_1^2 + O\Bigl( \frac{1}{m^s} + \frac{\ell_{m}}{n} + \frac{\ell_{m}^2}{n^2}\Bigr).
\]

By \eqref{var_hat_sigma_1}, \eqref{var_hat_sigma_2} and \eqref{var_hat_sigma_3}, we have 
\[
\Var{\hat{\sigma}_1^2} = O\left(\frac{1}{n^2 \min_{i\in[\ell_{m}]} \mu(B_i)} + \max_{i\in[\ell_{m}]} \mu(B_i) + \frac{1}{n} \right).
\]
By Chebshev's inequality, for any $\delta > 0$, we have 
\(
P\left(\left|\hat{\sigma}_1^2 - \EE{\hat{\sigma}_1^2}  \right| > \frac{\sqrt{\Var{\hat{\sigma}_1^2}}}{\sqrt{\delta}}\right) \leq \delta.
\)
Recall that $\Var{T(\tilde{N})} 
= \sigma_1^2 + O(\ell_{m}/n) + O(1/m^s)$.
Then for any $\delta \in (0, 1)$, with probability at least $1-\delta$, we have
\begin{equation*}
\begin{split}
\left| \frac{\hat{\sigma}_1^2}{\Var{T(\tilde{N})}} - 1 \right| &=  \left|\frac{\hat{\sigma}_1^2 - \EE{\hat{\sigma}_1^2} + \EE{\hat{\sigma}_1^2} - {\Var{T(\tilde{N})}}}{\Var{T(\tilde{N})}} \right| \\
&\leq \kappa_{25} \left(\frac{1}{\sqrt{\delta}} \left(\frac{1}{n\sqrt{\min_{i\in[\ell_{m}]} \mu(B_i)}} +\sqrt{\max_{i\in[\ell_{m}]} \mu(B_i)} + \frac{1}{\sqrt{n}}\right) + \frac{1}{m^s} + \frac{\ell_{m}}{n} + \frac{\ell_{m}^2}{n^2}\right), 
\end{split}
\end{equation*}
for some constant $\kappa_{25}$.

Let
\begin{equation*}
t = \frac{\hat{\sigma}_1}{\sqrt{\Var{T(\tilde{N})}}} \left(-z_{\alpha} + \frac{\sqrt{n}\mathrm{ECE}_{1:k}^2 - \EE{T(\tilde{N})}}{\hat{\sigma}_1}\right).
\end{equation*}
We then obtain
\begin{equation*}
\begin{split}
& \left| P\bigl(\mathrm{ECE}_{1:k}^2 \leq \tilde T(\tilde{N}) + z_{\alpha} \hat{\sigma}_1/\sqrt{n}\bigr) - \alpha  \right| = \left| P\left(\frac{T(\tilde{N}) - \EE{T(\tilde{N})}}{\sqrt{\Var{T(\tilde{N})}} } >  t\right) - \Phi(z_{\alpha})\right| \\
 \leq & \left| P\left(\frac{T(\tilde{N}) - \EE{T(\tilde{N})}}{\sqrt{\Var{T(\tilde{N})}} } >  t\right) - (1 - \Phi(t)) \right| + \left| 1 - \Phi(t) - (1 - \Phi(-z_{\alpha}))\right|.
\end{split}
\end{equation*}
Therefore, for any $\delta \in (0,1)$, with probability at least $1-\delta$,

\begin{equation*}
\begin{split}
    & \left| P\bigl(\mathrm{ECE}_{1:k}^2 \leq \tilde T(\tilde{N}) + z_{\alpha} \hat{\sigma}_1/(\sqrt{n})\bigr) - \alpha  \right| \\
    \leq &  \kappa_{26}  \left(\frac{1}{\sqrt{\delta}} \left(\frac{1}{n\sqrt{\min_{i\in[\ell_{m}]} \mu(B_i)}} +\sqrt{\max_{i\in[\ell_{m}]} \mu(B_i)} + \frac{1}{\sqrt{n}}\right)\right. \\
    &\left. \qquad+ \frac{\ell_{m}^2}{n^2} + \frac{\ell_{m}}{n} + \frac{1}{m^s} + \frac{\sqrt{n}}{m^s} + \sqrt{n}e^{-n\min_{i\in[\ell_{m}]} \mu(B_i)} \right).
\end{split}
\end{equation*}
where $\kappa_{26}$ is some constant. 

If $Z_{(1:k)}$ has a strictly positive density, under our binning scheme, with probability at least $1 - \delta$, we have
\begin{equation*}
    \left| P\bigl(\mathrm{ECE}_{1:k}^2 \leq \tilde T(\tilde{N}) + z_{\alpha} \sigma_1/(\sqrt{n})\bigr) - \alpha  \right| 
    \leq  \kappa_{27} \left(\frac{1}{\sqrt{\delta}} \left(\frac{m^{\tilde{k}/2}}{n} +\frac{1}{m^{\tilde{k}/2}} + \frac{1}{\sqrt{n}}\right) + \frac{m^{2\tilde{k}}}{n^2} +\frac{m^{\tilde{k}}}{n} + \frac{\sqrt{n}}{m^{s}} \right),
\end{equation*}
for some constant $\kappa_{27}$.

% In \Cref{sec:clt_calibrated_model}, we have proved a central limit theorem for $T(\tilde{N}) \tilde{N}\sqrt{w}$.

\section{Additional Experiments}
\label{addtional_experiments}
\subsection{Simulated Data}
The length of the confidence intervals of different methods for the simulated data experiments in the main text is given in \Cref{simulation_interval_length_1000}.

\begin{figure}
    \centering
    \includegraphics[scale=0.25]{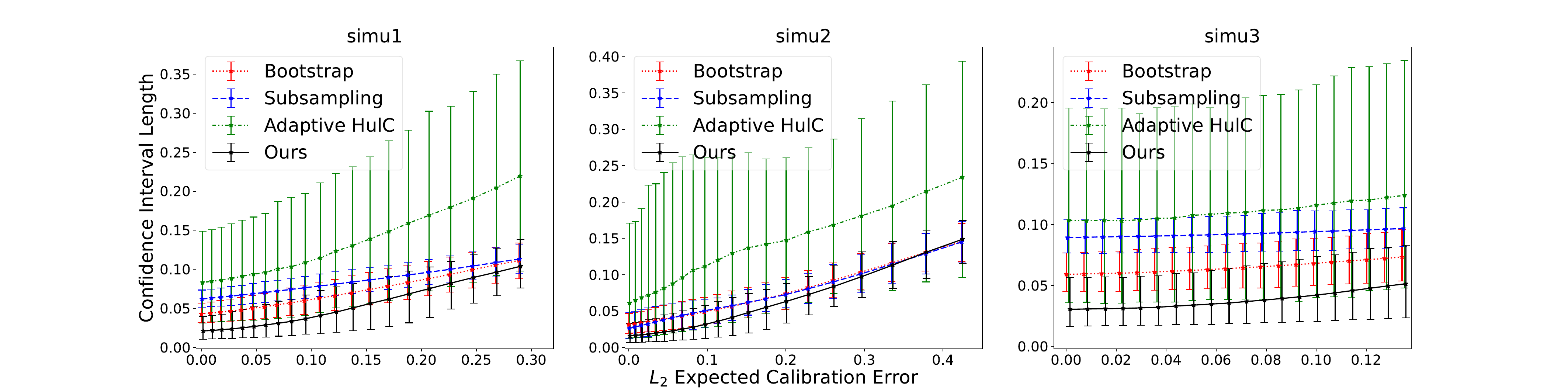}
    \caption{
    Confidence Interval Length vs. Calibration Error: Over 1000 datasets, we compute the average length of confidence intervals, with error bars representing the 5th and 95th percentiles.
    }
    \label{simulation_interval_length_1000}
\end{figure}

We also provide additional simulation experiments to validate the finite sample performance of our methods. 
We follow the same setting as in the main text, but with a different sample size $n = 1000$. For Settings 1 and 2, we consider the number of bins $mK = 50$, and for Setting 3, we set $mK = 20$. 
We report the coverage rate in \Cref{simulation_coverage_n_100}, the length of confidence intervals, the ratio of the average interval lengths of other methods relative to our method in \Cref{simulation_interval_length_n_100}.
% and power in \Cref{simulation_type2_error_n_100}. 
We also consider the hypothesis testing problem with the null hypothesis $H_0: \mathrm{ECE}_{1:k}^2 = 0$, and compare our method with the T-Cal method \citep{lee2023t} in Settings 1 and 2. T-Cal is designed for testing full calibration, so we consider full calibration in these settings. The threshold for the T-Cal method is obtained through 1000 Monte Carlo simulations, as suggested in Section 4.1 of \cite{lee2023t}. For all other methods, we reject the null hypothesis if the confidence interval does not contain zero. The power of each method is reported in \Cref{simulation_type2_error_n_100}.
Similar to the simulation example in the main text, our method still generates valid confidence intervals. Compared to Adaptive HulC \citep{kuchibhotla2024hulc} and Subsampling methods, our method generates much shorter intervals and exhibits greater power.

\begin{figure}
    \centering
    \includegraphics[scale=0.25]{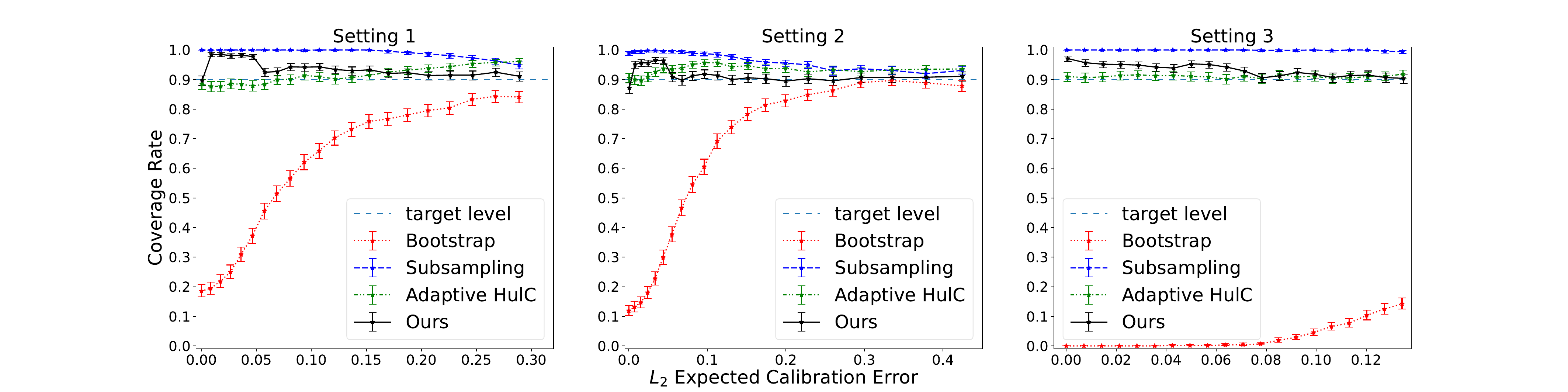}
    \caption{Coverage Rate vs. Calibration Error. For each setting and each value of $\beta$, we compute the coverage rate of the confidence intervals over 1000 datasets. The error bars for coverage rates are generated using the Clopper–Pearson method \citep{clopper1934use}. }
    \label{simulation_coverage_n_100}
\end{figure}

\begin{figure}
    \centering
    \includegraphics[scale=0.25]{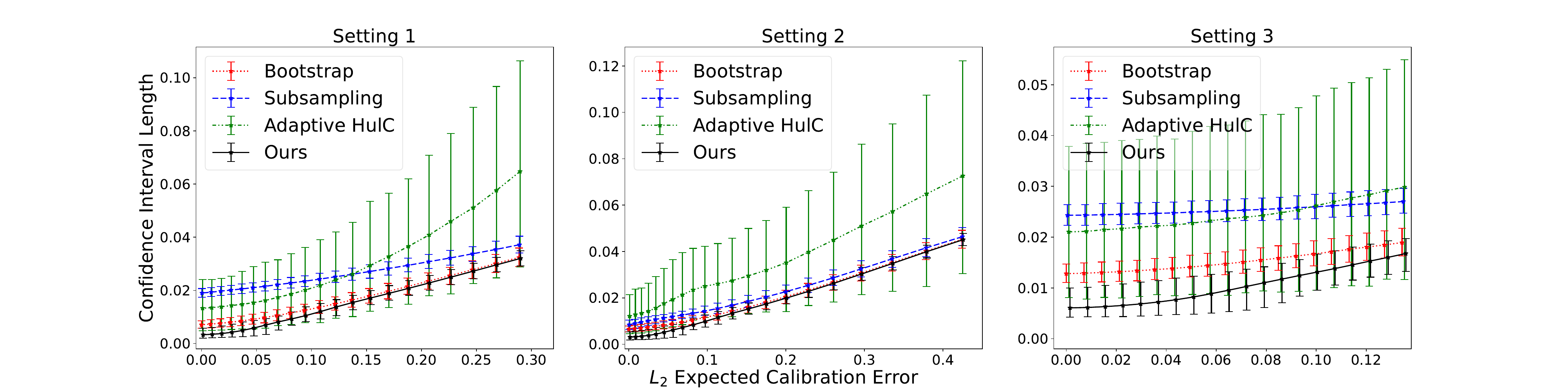}
    \includegraphics[scale=0.25]{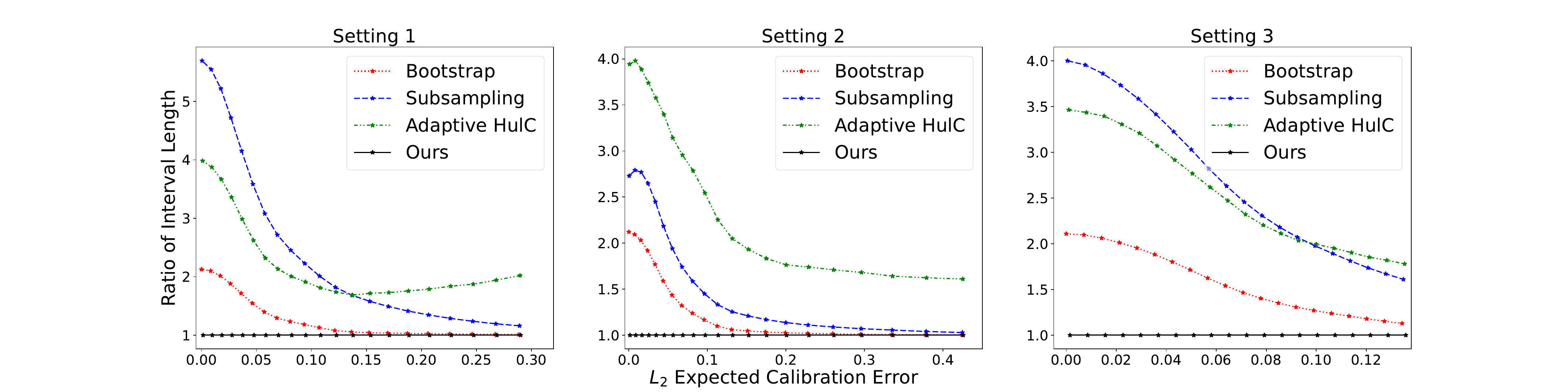}
    \caption{Confidence Interval Length vs. Calibration Error: Over 1000 datasets, we compute the average length of confidence intervals, with error bars representing the 5th and 95th percentiles. In the second row, we report the ratio of the average confidence interval lengths of various methods relative to our method.}
    \label{simulation_interval_length_n_100}
\end{figure}

\begin{figure}
    \centering
    \includegraphics[scale=0.25]{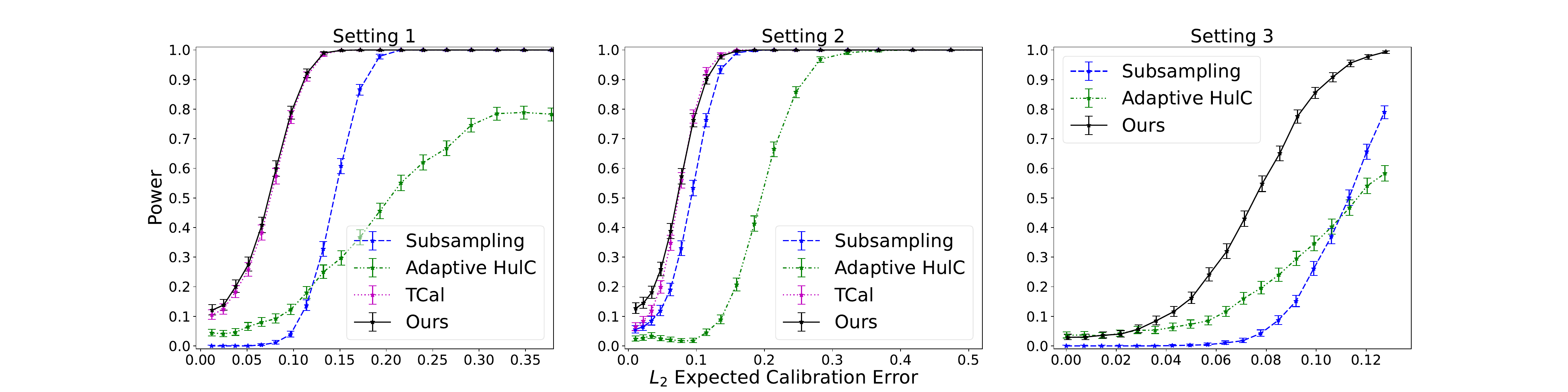}
    \caption{Power vs. Calibration Error.  For each setting and each value of $\beta$, we compute the power (percentage of null hypothesis rejections) over 1000 datasets. The error bars for power are generated using the Clopper–Pearson method.}
    \label{simulation_type2_error_n_100}
\end{figure}

{ We also conduct a sensitivity analysis of the binning parameter $m$ using our simulated data.
Figures~\ref{fig:coverage_bin_tuning} and \ref{fig:interval_length_bin_tuning} report the coverage rates and interval lengths using simulated data for different choices of the number of bins per dimension $mK$. With the exception of the case where the number of bins is 10 in Setting 2, our method consistently achieves coverage rates at or above the target level and produces confidence intervals that are shorter than those obtained from resampling-based methods, including subsampling and adaptive HulC. The result shows that our method is fairly robust to the choice of binning.}

\begin{figure}
    \centering
    \includegraphics[width=1.0\linewidth]{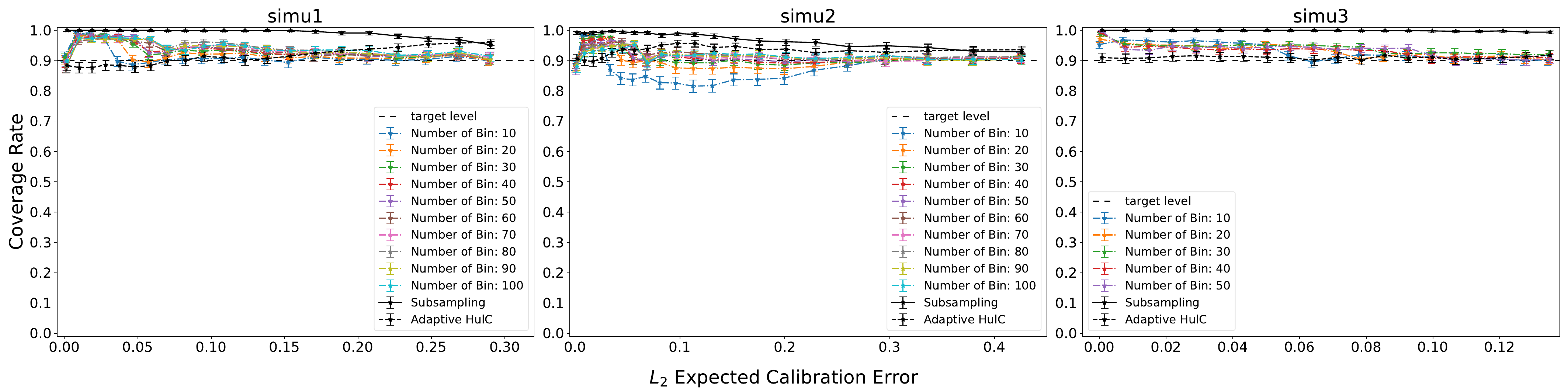}
    \caption{Coverage Rate with different binning parameters. The ``Number of Bin" $mK$ specifies the number of partitions per dimension. }
    \label{fig:coverage_bin_tuning}
\end{figure}
\begin{figure}
    \centering
    \includegraphics[width=1.0\linewidth]{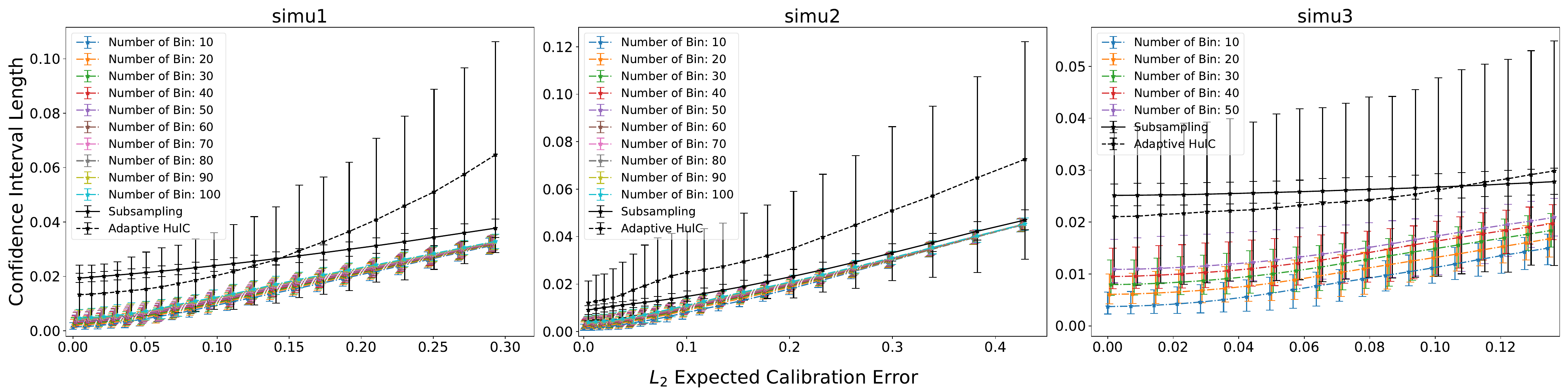}
    \caption{Interval Length with different binning parameters. The ``Number of Bin" $mK$ specifies the number of partitions per dimension. }
    \label{fig:interval_length_bin_tuning}
\end{figure}

\subsection{CIFAR}

We evaluate the top-1 $\ell_2$ ECE for models trained on the CIFAR10 dataset, all other settings are the same as those in Section 4.2.1 in the main text. Similar to the T-Cal method \citep{lee2023t}, our method can test if a model is significantly mis-calibrated by checking if the confidence interval covers zero.

\begin{table}
\centering
    \begin{tabular}{lcccc}
    \toprule
\multirow{2}{*}{} & \multicolumn{2}{c}{DenseNet40} &   \multicolumn{2}{c}{ResNet110} \\
                  & $\sqrt{T_{m,n}}$        & CI & $\sqrt{T_{m,n}}$        & CI    \\
\midrule
No Calibration      &    7.54\% & [6.69\%, 8.30\%] & 7.89\% & [7.03\%, 8.67\%]  \\                  
Temperature Scaling       &  1.05\%  & [0.00\%, 2.34\%] &  1.83\%  & (0.00\%, 2.86\%]   \\
Matrix Scaling & 1.23\% & [0.00\%, 2.36\%] & 2.04\% & [0.62\%, 3.00\%]\\
Focal Loss & 1.19\%  & [0.00\%, 2.15\%] & 1.13\%  & [0.00\%, 2.37\%]  \\
MMCE & 2.50\%  & [1.38\%, 3.44\%] & 2.41\%  & [1.20\%, 3.39\%] \\
\bottomrule          
\end{tabular}
\caption{The values of $\sqrt{T_{m,n}}$ and confidence intervals for $\mathrm{ECE}_{1:k}$ of models trained on CIFAR10. 
}
\label{tab:cifar10}
\end{table}

{Our current theoretical results assume that the evaluation set consists of i.i.d. samples independent of the data used to train the classifier. Under this assumption, the predicted probabilities and labels $(Z^{(i)}, Y^{(i)})$ are i.i.d. samples, which allows us to derive the asymptotic distribution of the $\ell_2$ ECE estimator and construct valid confidence intervals. For recent large-scale models, the training set can include almost all the seen data. When the evaluation set overlaps with the training set, this independence assumption is violated, and the confidence interval guarantees in the paper no longer directly apply. Note that a model assigning probability 1 to the correct label is perfectly calibrated with 0 $\ell_2$-ECE. When a model overfits, it can achieve 100\% accuracy on the training set, with predicted probabilities for the correct labels close to 1. In this case, using the training data for evaluation introduces a bias toward zero in the estimated calibration error.}

{
We provide an empirical investigation of our method when different proportions of the training data are used for evaluation. We consider a ResNet-110 model trained on the CIFAR-10 dataset as an illustrative example, representing a setting in which the model overfits the training data. Note that a model assigning probability 1 to the correct label is perfectly calibrated with 0 $\ell_2$-ECE. When a model overfits, it can achieve 100\% accuracy on the training set, with predicted probabilities for the correct labels close to 1. In this case, using the training data for evaluation introduces a bias toward zero in the estimated calibration error. We mix training data with validation data to compute the $\ell_2$-ECE estimator. Figure~\ref{fig:ece_vs_training_data} plots the estimator as a function of the proportion of training data used (e.g., when the proportion is 0.5, we use 5,000 training samples and 5,000 validation samples). As expected, incorporating training data into the evaluation leads to an underestimation of the ECE. In the extreme case, when only training data are used, the estimator is very close to zero. These results highlight that, in the presence of overfitting, using an independent calibration set is crucial for accurately assessing model calibration.
}

        \begin{figure}
            \centering
            \includegraphics[width=0.5\linewidth]{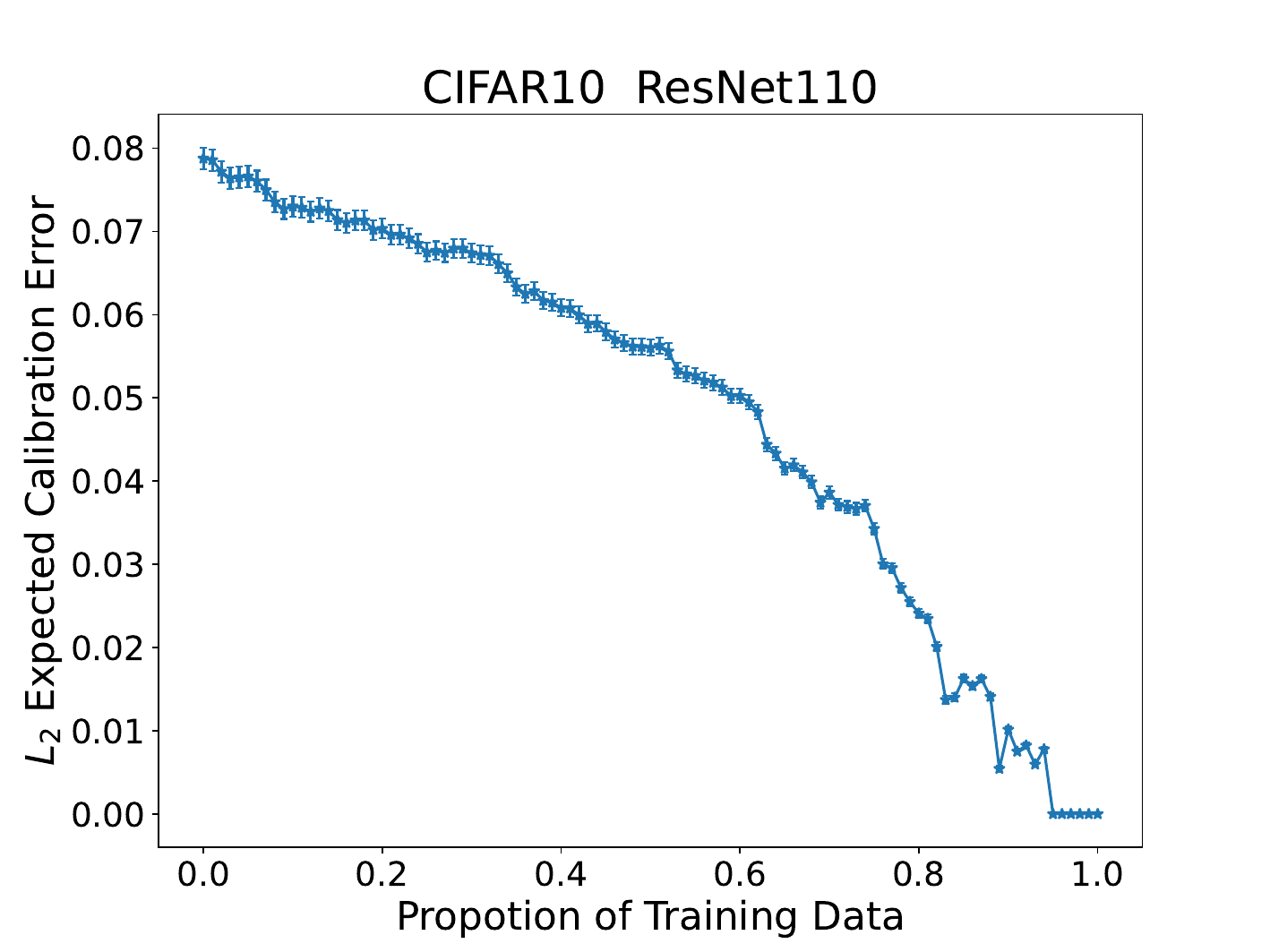}
            \caption{$\ell_2$ ECE estimator for ResNet110 model on CIFAR10 when different proportions of training data are used in evaluation}
            \label{fig:ece_vs_training_data}
        \end{figure}

\subsection{Large Language Models with Verbalized Confidence}
{
We conducted an experiment applying our method to evaluate the top-1 expected calibration error of Large Language Models (LLMs). Specifically, we follow the vanilla verbalized-confidence prompting strategy of \cite{xiong2023can}, treat the model's verbalized confidence as its top-1 predicted probability, and then apply our method to construct confidence intervals.}

{
The verbalized confidence setting in LLMs is not identical to the standard classification based calibration framework. First, the confidence is reported only for the predicted answer, rather than as a full probability vector over all classes. Second, except in multiple-choice settings, the number of candidate classes is not explicitly specified. Moreover, the verbalized confidence for the predicted answer may not be the model's largest class probability in the usual classification setting. For example, in a multiple-choice problem with four options, the model may sometimes output a verbalized confidence close to zero when the question is incomplete or highly ambiguous. Such a value cannot be interpreted as the largest probability among four choices. Therefore, only the top-1 calibration error (if we treat the verbalized confidence as top-1 predicted probability) is naturally defined in this framework. 
} 

{
We consider the Business Ethics subset of MMLU \citep{hendrycks2020measuring}, which is a multiple-choice benchmark used in \cite{xiong2023can}, and evaluate four open-source LLMs: Llama-2-7b-chat \citep{touvron2023llama}, Llama-3-8B-Instruct \citep{grattafiori2024llama}, Qwen-3-8B, and Qwen-3-14B \citep{yang2025qwen3}. The results are reported in Figure~\ref{fig:llm_mmlu}.
This dataset contains only 100 test questions, so accounting for the randomness in the data is especially important. Under our framework, the last three models are not significantly miscalibrated. In contrast, the commonly used ECE estimator \citep{guo2017calibration} computed with 20 bins yields values of \(37.26\%\), \(13.92\%\), \(9.02\%\), and \(7.20\%\), respectively, which might be interpreted as clear evidence of miscalibration. This experiment highlights the practical value of our approach in modern LLM applications. When high-quality evaluation datasets are limited, it is critical to account for statistical uncertainty in model evaluation.
}

    \begin{figure}
        \centering
        \includegraphics[width=0.5\linewidth]{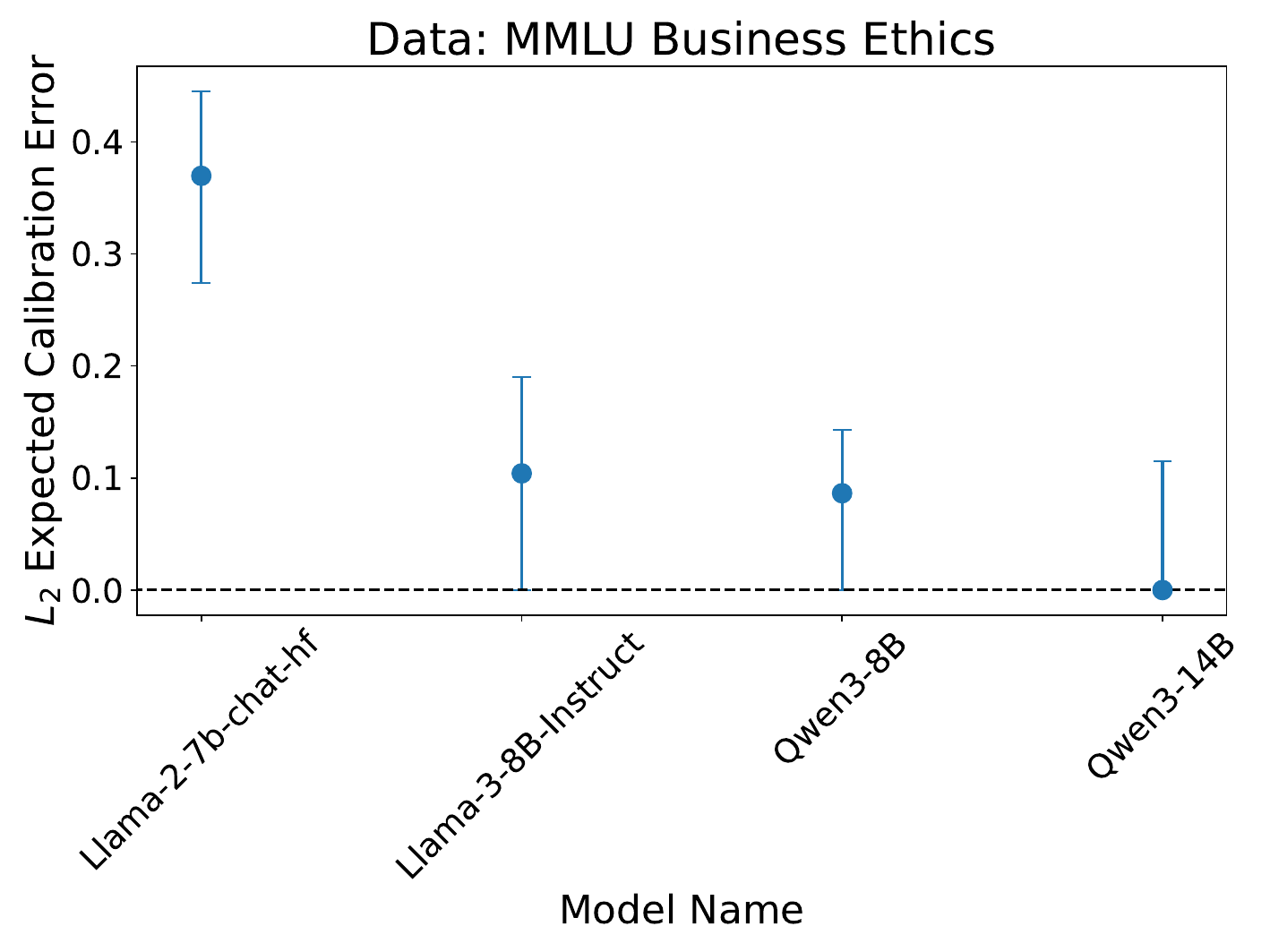}
        \caption{Confidence intervals of top-1 $\ell_2$-ECE of 4 different LLMs on Business Ethics data set of MMLUN. umber of bins $mK$ is set to be $20$.}
        \label{fig:llm_mmlu}
    \end{figure}

% {\small
% \setlength{\bibsep}{0.2pt plus 0.3ex}
% \bibliographystyle{plainnat-abbrev}
% \bibliography{ref}
% }

\end{document}